\documentclass[final,12pt]{colt2025} 

\PassOptionsToPackage{dontload}{natbib}

\usepackage{cite}
\usepackage{enumerate,multirow}
\usepackage{siunitx}
\usepackage{multirow}
\usepackage{multicol}
\usepackage{color}
\usepackage{xcolor}
\usepackage{comment}
\definecolor{MyGreen1}{RGB}{20,180,40}
\definecolor{MyBlue1}{RGB}{00,150,255}
\definecolor{MyGray1}{RGB}{200,200,200}

\usepackage{times,amssymb,amsmath,amsfonts,nicefrac,color,bbm,mathrsfs,float}
\usepackage{multirow,tikz,graphicx}

\newtheorem{assumption}{Assumption}

\newcommand{\bigO}{\mathcal{O}}

\DeclareMathAlphabet{\mathbfsl}{OT1}{ppl}{b}{it} 

\newcommand{\al}{\alpha}

\newcommand{\eps}{\epsilon}

\newcommand{\tnabla}{\tilde{\nabla}}
\newcommand{\D}{\mathcal{D}}
\newcommand{\A}{\mathcal{A}}

\def\QEDclosed{\mbox{\rule[0pt]{1.3ex}{1.3ex}}} 

\def\QED{\QEDclosed} 

\def\proof{\noindent\hspace{2em}{\itshape Proof: }}
\def\endproof{\hspace*{\fill}~\QED\par\endtrivlist\unskip}

\newcommand{\be}[1]{\begin{equation}\label{#1}}
	\newcommand{\ee}{\end{equation}}

\renewcommand{\leq}{\leqslant}
 
\renewcommand{\geq}{\geqslant}

\renewcommand{\Bbb}{\mathbb}
\newcommand{\E}{{\mathbb{E}}}

\newcommand{\R}{{\Bbb R}}

\usepackage{comment}

\newcommand{\Cref}[1]{Co\-ro\-lla\-ry\,\ref{#1}}

\title{\LARGE A Generalized Meta Federated Learning Framework\\ with Theoretical Convergence Guarantees}

\coltauthor{%
	\Name{Mohammad Vahid Jamali} \Email{mv.jamali@samsung.com}\\
	\Name{Hamid Saber} \Email{hamid.saber@samsung.com}\\
	\Name{Jung Hyun Bae} \Email{jung.b@samsung.com}\\
	\addr Samsung Semiconductor, Inc., SOC Lab, San Diego, CA, USA
}

\begin{document}

\maketitle
\thispagestyle{empty} 
\begin{abstract}%
Meta federated learning (FL) is a personalized variant of FL, where multiple agents collaborate on training an initial shared model without exchanging raw data samples. The initial model should be trained in a way that current or new agents can easily adapt it to their local datasets after one or a few fine-tuning steps, thus improving the model personalization.
Conventional meta FL approaches minimize the average loss of agents on the local models obtained after one step of fine-tuning.
In practice, agents may need to apply several fine-tuning steps to adapt the global model to their local data, especially under highly heterogeneous data distributions across agents. To this end, we present a generalized framework for the meta FL by minimizing the average loss of agents on their local model after any arbitrary number $\nu$ of fine-tuning steps.
For this generalized framework, we present a variant of the well-known
federated averaging (FedAvg) algorithm and conduct a comprehensive theoretical convergence analysis to characterize the convergence speed as well as behavior of the meta loss functions in both the exact and approximated cases. 
Our experiments on real-world datasets demonstrate superior accuracy and faster convergence for the proposed scheme compared to conventional approaches.
\end{abstract}


\section{Introduction}\label{Sec1}
Federated learning (FL) is a machine learning setting where multiple entities (devices, clients, organizations, sensors, nodes, etc.) collaborate in solving a machine learning problem (e.g., learning a single global statistical model in a standard federated learning problem), under the coordination of a service provider or central server, without exchanging the raw data of the entities \cite{mcmahan2017communication,kairouz2021advances}. A major advantage of FL is that data from other agents can
be aggregated to improve the local task of each agent.

In a standard FL setting, a set of $N$ agents are connected to a central node (server) to collaboratively solve the following standard distributed optimization problem \cite{konevcny2016federated}:
\begin{equation}\label{class_FL}
	\min_{\mathbf{w} \in \R^d} f(\mathbf{w}):= \frac{1}{N} \sum_{i=1}^N f_i(\mathbf{w}),
\end{equation}
where $\mathbf{w} \in \R^d$ is the size-$d$ vector of model parameters/weights, simply referred to as the \textit{model} hereafter. Moreover, $f_i(\mathbf{w})$ represents the expected loss over the data distribution of agent $i$, which under  a supervised learning setting
can be expressed as
\begin{equation}\label{supervised_ML}
	f_i(\mathbf{w}) := \E_{(\mathbf{x},y) \sim p_i} \left [ l_i(\mathbf{w};\mathbf{x},y) \right ],	
\end{equation}
where $l_i(\mathbf{w};\mathbf{x},y)$ measures the error of model $\mathbf{w}$ in predicting the true label $y \in \mathcal{Y}_i$ given the input $\mathbf{x} \in \mathcal{X}_i$, and $p_i$ is the distribution over $\mathcal{X}_i \times \mathcal{Y}_i$.

The conventional formulation of FL in \eqref{class_FL} only generates a common output for all agents and does not tailor the model to each individual agent. This is significant, particularly under data  \textit{heterogeneous} settings where the probability distribution $p_i$ of agents are not identical. Inspired by the model-agnostic meta learning (MAML) framework in \cite{finn2017model}, the authors in \cite{fallah2020personalized} proposed a personalized variant of FL aimed at obtaining an initial shared model that current and new agents can adapt to their local datasets by performing one or a few steps of gradient descent with their own data. 
More specifically, to achieve a personalized solution, meta FL modifies the objective function to optimize the average loss of devices on their local model obtained after one step of fine-tuning, i.e., 
\begin{equation}\label{Meta_FL}
	\min_{\mathbf{w} \in \R^d} F(\mathbf{w}) := \frac{1}{N} \sum_{i=1}^N f_i(\mathbf{w} - \alpha \nabla f_i(\mathbf{w})),
\end{equation}
where $\alpha \geq 0$ is the stepsize. 
The local updates are then with respect to (w.r.t.) the so-called \textit{meta  functions} $F_i$
 defined as  
\begin{equation}\label{Meta_FL_2}
	F_i(\mathbf{w}):= f_i(\mathbf{w} - \alpha \nabla f_i(\mathbf{w})).
\end{equation}

There is a growing body of work focused  on personalized FL, which aim at personalizing the global models for individual devices to address the data and system heterogeneity. In particular, \cite{fallah2020personalized}
presented a meta FL algorithm, named Per-FedAVg, by allowing the optimization problem to minimize the average loss of the devices on the model obtained after a single step of fine-tuning. \cite{li2020federated} proposed an algorithm, named FedProx, which uses $l_2$ regularization to combine local models. There are several other studies on statistical heterogeneity of the data points of the devices, including \cite{li2020federated, 
	karimireddy2020scaffold, zhao2018federated, li2019convergence}. More recently, \cite{t2020personalized} proposed an algorithm, called pFedMe, for personalized FL  using Moreau envelopes as the  regularized loss functions for the clients. Additional recent work on personalized FL include \cite{lee2024fedl2p,ilhan2023scalefl, zhang2023fedala}.

This paper makes further progress in meta FL
\cite{fallah2020personalized} to present and analyze a generalized meta FL framework. 
In practice, devices may need to apply several fine-tuning steps to adapt the global model to their local data. To this end, we consider a generalized framework 
for the meta FL to optimize the average loss of devices on their local model after  arbitrary number $\nu$ of fine-tuning steps.
The goal is to explore whether one can achieve  a better accuracy and 
faster convergence via this generalized meta FL framework, especially under heterogeneous settings.
A similar setup was considered in \cite{ji2022theoretical}, where the authors  explored the convergence analysis of the multi-step MAML requiring the calculation of Hessian matrices. 
In this paper, we consider more practical generalized meta FL schemes, where the heavy computations of Hessian matrices need to be avoided. In particular, we provide the generalized meta FL algorithm and convergence analysis under practical first-order and Hessian-free approximations, thus avoiding any calculation of Hessian matrices. 
We also independently study the exact case of the algorithm and provide somewhat different bounds for the convergence behavior and properties of meta loss functions. 
Moreover, we perform 
training experiments over real-world datasets 
to quantify the gains on accuracy and convergence speed compared to conventional approaches. Such experiments provide further insights, beyond the asymptotic convergence analysis, about the behavior of the algorithms over finite number of communication rounds.

\section{Generalized Meta FL under Multiple Fine-Tuning Steps}
Once  the objective function is  formulated and  the meta functions $F_i$'s are defined, the gradient descent can be applied to perform the local updates at each $i$-th device, $i=1,\cdots,N$, and over each $k$-th global epoch,  $k=1,\cdots,K$. Each device does $\tau$ local updates before sending the updated model to the server. The local updates, $\mathbf{w}_{k,t}^{i}$ for $1 \leq t \leq \tau$, are governed by the gradient of meta functions as 
\begin{equation}\label{local_update}
	\mathbf{w}_{k,t}^{i} = \mathbf{w}_{k,t-1}^{i} -  \beta \nabla F_i(\mathbf{w}_{k,t-1}^{i}),
\end{equation}
where $\mathbf{w}_{k,0}^{i} := \mathbf{w}_{k-1}$, i.e., the model sent by the server, and
$\beta$ is the local learning rate (stepsize). Note that in the case of $\nu=1$, which corresponds to the conventional meta FL formulated in \eqref{Meta_FL}, 
 the gradient $\nabla F_i$ of the meta function is given by \cite{fallah2020personalized}
\begin{equation}\label{grad_Fi}
	\nabla F_i(\mathbf{w}) = \left(\mathbf{I} - \alpha \nabla^2 f_i(\mathbf{w})\right) \nabla f_i(\mathbf{w} - \alpha \nabla f_i(\mathbf{w})),
\end{equation}
which involves second-order derivatives through the Hessian matrix $\nabla^2 f_i(\mathbf{w})$. In the following, we present the case of
arbitrary $\nu\in\{1,2,\cdots\}$.

\subsection{General $\nu$ Fine-Tuning Steps}
Let us recursively define $\mathbf{w}^i_{l}:=\mathbf{w}^i_{l-1} - \alpha \nabla f_i(\mathbf{w}^i_{l-1})$, i.e., the model after ${l}$ fine-tuning steps, ${l}=1,\cdots,\nu$, with $\mathbf{w}^i_0:=\mathbf{w}$, i.e., the global \textit{meta} model.
We can then formulate the optimization problem as 
\begin{align}
	\min_{\mathbf{w} \in \R^d} F(\mathbf{w}) &:= \frac{1}{N} \sum_{i=1}^N f_i(\mathbf{w}^i_{\nu})
	=	\frac{1}{N} \sum_{i=1}^N f_i(\mathbf{w}^i_{\nu-1} - \alpha \nabla f_i(\mathbf{w}^i_{\nu-1})).\label{Meta_FLt2}
\end{align}
As shown in Appendix \ref{Appendix_gradF}, by applying the chain rule, the gradient of the meta functions $F_i(\mathbf{w}):=f_i(\mathbf{w}^i_{\nu-1} - \alpha \nabla f_i(\mathbf{w}^i_{\nu-1}))$ can be characterized in closed-form as
\begin{align}
	\nabla F_i(\mathbf{w})
	&= \left(\mathbf{I} - \alpha \nabla^2 f_i(\mathbf{w})\right) \left(\mathbf{I} - \alpha \nabla^2 f_i(\mathbf{w}^i_1)\right) \left(\mathbf{I} - \alpha \nabla^2 f_i(\mathbf{w}^i_2)\right) \nonumber\\
	& {\hspace{1.5in}} \cdots \left(\mathbf{I} - \alpha \nabla^2 f_i(\mathbf{w}^i_{\nu-1})\right) \nabla f_i(\mathbf{w}^i_{\nu-1} - \alpha \nabla f_i(\mathbf{w}^i_{\nu-1})).\label{grad_Fit_2_last}
\end{align}
As seen, the gradient of the meta functions and thus the local update rule only involve derivatives of at most second order, i.e., the Hessian matrices of the loss functions, which can be approximated as in \cite{fallah2020personalized} to reduce the computational burden. In Section \ref{sec_1stOrder}, we extend our algorithm and analysis to such first-order approximations.

\begin{remark}\label{Remark_1}
Computing the gradient $\nabla f_i(\mathbf{w})$ at every round is often computationally costly. Hence, we take a batch of data $\D_i$ with respect to distribution $p_i$ to obtain an unbiased estimate $\tnabla f_i(\mathbf{w}, \D_i)$ given by
\begin{equation}\label{unbiased_grad}
	\tnabla f_i(\mathbf{w}, \D_i) := 	\frac{1}{|\D_i|} \sum_{(\mathbf{x},y) \in \D_i} \nabla l_i(\mathbf{w};\mathbf{x},y).
\end{equation}
Similarly, the Hessian $\nabla^2 f_i(\mathbf{w})$ can be replaced by its unbiased estimate $\tnabla^2 f_i(\mathbf{w}, \D'_i)$ evaluated using a batch $\D'_i$.
\end{remark}

\begin{remark}\label{Remark_2}
Based on the recursive expression $\mathbf{w}^i_{l}:=\mathbf{w}^i_{l-1} - \alpha \nabla f_i(\mathbf{w}^i_{l-1})$, for ${l}=1,\cdots,\nu$ (with $\mathbf{w}^i_{0}=\mathbf{w}$), the unbiased estimate of $\mathbf{w}^i_{l}$ at the $i$-th agent can be obtained using $l$ independent batches $\D_{i,0},\D_{i,1}, \cdots, \D_{i,l-1}$ of size $D_{i,0},D_{i,1}, \cdots, D_{i,l-1}$ as
\begin{align}\label{w_l_unbiased}
	\tilde{\mathbf{w}}^i_l:=\tilde{\mathbf{w}}^i_l\left(\left\{\D_{i,j}\right\}_{j=0}^{l-1}\right)=\tilde{\mathbf{w}}^i_{l-1} - \alpha \tnabla f_i(\tilde{\mathbf{w}}^i_{l-1},\D_{i,l-1}), ~~~ {l}=1,\cdots,\nu,
\end{align}
with $\tilde{\mathbf{w}}^i_0:={\mathbf{w}}$. Therefore, using \eqref{grad_Fit_2_last}, the stochastic gradient $\tnabla F_i(\mathbf{w})$ at each local update iteration is computed using $2\nu+1$ independent batches $\left\{\{\D_{i,j}\}_{j=0}^{\nu},\{\D'_{i,j'}\}_{j'=0}^{\nu-1}\right\}$ of size $\left\{\{D_{i,j}\}_{j=0}^{\nu},\{D'_{i,j'}\}_{j'=0}^{\nu-1}\right\}$ as
\begin{align}
	\tnabla F_i(\mathbf{w}) := & \left(\mathbf{I} - \alpha \tnabla^2 f_i(\mathbf{w},\D'_{i,0})\right) \left(\mathbf{I} - \alpha \tnabla^2 f_i(\tilde{\mathbf{w}}^i_1,\D'_{i,1})\right) \left(\mathbf{I} - \alpha \tnabla^2 f_i(\tilde{\mathbf{w}}^i_2,\D'_{i,2})\right) \nonumber\\
	& {\hspace{0.3in}} \cdots \left(\mathbf{I} - \alpha \tnabla^2 f_i(\tilde{\mathbf{w}}^i_{\nu-1},\D'_{i,\nu-1})\right) \tnabla f_i(\tilde{\mathbf{w}}^i_{\nu-1} - \alpha \tnabla f_i(\tilde{\mathbf{w}}^i_{\nu-1},\D_{i,\nu-1}),\D_{i,\nu}).\label{grad_F_stoch}
\end{align}
%
\end{remark}

\subsection{Generalized Meta FL Algorithm}
\begin{algorithm2e}[t]
\caption{Generalized Meta FL Algorithm}
\label{Alg_MetaFL}
\DontPrintSemicolon
\LinesNumbered
\KwIn{Initial model $\mathbf{w}_0$, fraction of active users $r$, number of local updates $\tau$, number of global updates $K$, and number of fine-tuning steps $\nu$.}
\KwOut{Final meta model $\mathbf{w}_K$}
\For{$k\leftarrow 0$ \KwTo $K-1$}{
	Server chooses a subset $\A_k$ of users, uniformly at random, with size $rN$\;
	Server sends $\mathbf{w}_k$ to all users in $\A_k$\;
	\For{{all} $i \in \A_k$}{
		$\mathbf{w}_{k+1,0}^{i} \leftarrow \mathbf{w}_k$\;
		\For{$t\leftarrow 1$ \KwTo $\tau$}{
			$\mathbf{w}_{k+1,t}^{i}=\mathsf{gMeta\_Local{\_}Update}\big(\mathbf{w}_{k+1,t-1}^{i},f_i(.),\nu, \{\D^{~\!t}_{i,j}\}_{j=0}^{\nu},\{\D'^{~\!t}_{i,j'}\}_{j'=0}^{\nu-1}\big)$\;
		}
	Agent $i$ sends $\mathbf{w}_{k+1,\tau}^{i}$ back to the server\;
	}
Server updates the global model by aggregating the received local models: $\mathbf{w}_{k+1} = \frac{1}{rN} \sum_{i \in \A_k} \mathbf{w}_{k+1,\tau}^{i}$\;
}
\end{algorithm2e}

The generalized meta FL algorithm, built upon ideas of the FedAvg algorithm, is summarized in Algorithm \ref{Alg_MetaFL}. Starting from an initial global model $\mathbf{w}_0$, the server at each $k$-th global iteration shares the current global model $\mathbf{w}_k$ with a fraction $r$ of agents. Each selected agent applies $\tau$ local updates on top of $\mathbf{w}_k$ before sharing the final updated model with the central server. The server then aggregates the received local models to obtain the new global (meta) model.

The key distinction of our algorithm is the local updates performed by agents. Such updates are governed by the gradient of the meta functions, as established in \eqref{local_update} and \eqref{grad_F_stoch}. The procedure is summarized in Algorithm \ref{Alg_gMeta_Exact}. The function $\mathsf{gMeta\_Local{\_}Update}(.)$, that defines the local generalized meta update, takes the local loss function and a set of data batches as input and returns the updated model for a given input model and number $\nu$ of fine-tuning steps.  Each agent first calculates and stores $\nu$ intermediate models obtained after $\nu$ gradient descent steps. It then recursively applies \eqref{grad_F_stoch} to estimate $\tnabla F_i(\mathbf{w})$, from which the updated model can be obtained according to \eqref{local_update}. The superscript $t$ in $\D^{~\!t}_{i,j}$ and $\D'^{~\!t}_{i,j'}$
is used to highlight that we apply new batches for each $t$-th local update iteration.
Note that Algorithm \ref{Alg_gMeta_Exact} requires evaluations of the Hessian matrix. In Section \ref{sec_1stOrder}, we present variations of the algorithm under first-order approximations.

\begin{algorithm2e}[t]
	\caption{Generalized Meta Local Update;
		 $\mathsf{gMeta\_Local{\_}Update}(.)$}
	\label{Alg_gMeta_Exact}
	\DontPrintSemicolon
	\LinesNumbered
	\KwIn{Initial local model $\mathbf{w}^{i}_{\rm in}$, local loss function $f_i(.)$, number of fine-tuning steps $\nu$, and data batches $\left\{\D_{i,j}\right\}_{j=0}^{\nu},\{\D'_{i,j'}\}_{j'=0}^{\nu-1}$ for arbitrary agent $i$.}
	\KwOut{Updated local model $\mathbf{w}^{i}_{\rm out}$}
	$\tilde{\mathbf{w}}^{i}_{0}\leftarrow\mathbf{w}^{i}_{\rm in}$\;
	\For{$l\leftarrow 1$ \KwTo $\nu$}{
		Compute  stochastic gradient $\tnabla f_i(\tilde{\mathbf{w}}^{i}_{l-1},\D_{i,l-1})$ using dataset  $\D_{i,l-1}$\;
		Calculate and store $\tilde{\mathbf{w}}_{l}^{i}\leftarrow\tilde{\mathbf{w}}_{l-1}^{i} - \alpha \tnabla f_i(\tilde{\mathbf{w}}_{l-1}^{i},\D_{i,l-1})$\;
	}
		$\mathbf{d}\leftarrow\tnabla f_i(\tilde{\mathbf{w}}_{\nu}^{i},\D_{i,\nu})$\;
	\For{$l'\leftarrow 1$ \KwTo $\nu$}{
		Estimate the Hessian $\tnabla^2 f_i(\tilde{\mathbf{w}}_{\nu-l'}^{i},\D'_{i,\nu-l'})$ using dataset  $\D'_{i,\nu-l'}$\;
		 ${\mathbf{d}}\leftarrow\left(\mathbf{I} - \alpha \tnabla^2 f_i(\tilde{\mathbf{w}}_{\nu-l'}^{i},\D'_{i,\nu-l'})\right)\mathbf{d}$\;
	}
$\mathbf{w}_{\rm out}^{i} \leftarrow \mathbf{w}_{\rm in}^{i} - \beta {\mathbf{d}}$\;
\end{algorithm2e}

\section{Convergence Analysis}\label{convergence}
In this section, we provide a thorough analysis of the convergence behavior of our generalized meta FL approach. The goal is to characterize the overall number of communication rounds between the server and agents to find an $\epsilon$-approximate first-order stationary point, as defined in the following. 
\begin{definition}
	A random vector $\mathbf{w}_\eps \in \R^d$ is called an $\epsilon$-approximate First-Order Stationary Point (FOSP) for problem \eqref{Meta_FLt2} if it satisfies $\E[ \Vert \nabla F(\mathbf{w}_\eps) \Vert^2] \leq \eps$ \cite{fallah2020personalized}.
\end{definition}

\subsection{Assumptions}
We first make the following assumptions, which are 
required for proving our main results.

\begin{assumption}\label{asm:boundedness}
	Functions $f_i$ are bounded below, i.e., $\min_{\mathbf{w} \in \R^d} f_i(\mathbf{w}) > -\infty$. 
\end{assumption}

\begin{assumption}\label{asm_grad}
	For every $i \in\{1,\dots,N\}$, $f_i$ is twice continuously differentiable and $L_i$-smooth, and also, its gradient is bounded by a non-negative constant $B_i$, i.e., 
	\begin{align}\label{grad_asm_1}
		\|\nabla f_i(\mathbf{w})\| \leq B_i,
		\quad \Vert \nabla f_i(\mathbf{w}) - \nabla f_i(\mathbf{u}) \Vert \leq L_i \Vert \mathbf{w} - \mathbf{u} \Vert  \quad \forall \mathbf{w},\mathbf{u} \in \R^d.
	\end{align}  
\end{assumption}

The  following condition is imposed on the Hessian of each $f_i$ which is a common assumption in the analysis of second-order methods.
\begin{assumption}\label{asm_Hesian_Lip}
	For every agent $i \in\{1,\dots,N\}$, the Hessian of function $f_i$ is $\rho_i$-Lipschitz continuous, i.e.,
	\begin{equation}\label{Hessian_Lip}
		\Vert \nabla^2 f_i(\mathbf{w}) - \nabla^2 f_i(\mathbf{u}) \Vert \leq \rho_i \Vert \mathbf{w} - \mathbf{u} \Vert \quad \forall \mathbf{w},\mathbf{u} \in \R^d.
	\end{equation}
\end{assumption}

For the sake of simplification of analysis, in the rest of the paper, we define $B:= \max_{i} B_i$,  $L:= \max_{i} L_i$, and $\rho:= \max_{i} \rho_i$ which can be considered as an upper bound on the norm of gradient of $f_i$, smoothness parameter of $f_i$, and Lipschitz continuity parameter of Hessian $\nabla^2 f_i$, respectively, for $i=1,\dots, N$. 

The next assumption provides upper bounds on the variance and kurtosis of gradient and Hessian estimations. 
\begin{assumption}\label{asm_bounded_var_i}
	For any $\mathbf{w}\in \R^d$, the stochastic gradient $\nabla l_i(\mathbf{x},y;\mathbf{w})$ and Hessian $\nabla^2 l_i(\mathbf{x},y; \mathbf{w})$
	have bounded variance, i.e., 
	\begin{align}
		&\E_{(\mathbf{x},y) \sim p_i} \left [\| \nabla l_i(\mathbf{x},y;\mathbf{w}) - \nabla f_i(\mathbf{w})\|^2 \right ] \leq \sigma_G^2, \label{bounded_var_i} \\
		&\E_{(\mathbf{x},y) \sim p_i} \left [\| \nabla^2 l_i(\mathbf{x},y;\mathbf{w}) - \nabla^2 f_i(\mathbf{w})\|^2 \right ] \leq \sigma_H^2. \label{bounded_hess_i}  
	\end{align}
	We additionally assume that the stochastic gradient and Hessian have a bounded kurtosis (thus a bounded fourth standardized moment), i.e.,
	\begin{align}
		&\E_{(\mathbf{x},y) \sim p_i} \left [\| \nabla l_i(\mathbf{x},y;\mathbf{w}) - \nabla f_i(\mathbf{w})\|^4 \right ] \leq \kappa_G, \label{kurt_grad} 
		\\&\E_{(\mathbf{x},y) \sim p_i} \left [\| \nabla^2 l_i(\mathbf{x},y;\mathbf{w}) - \nabla^2 f_i(\mathbf{w})\|^4 \right ] \leq \kappa_H. \label{kurt_hess}  
	\end{align}  
\end{assumption}

Finally, the following assumption characterizes the \textit{similarity} between the tasks of agents.
\begin{assumption}\label{asm_similarity}
	For any $\mathbf{w}\in \R^d$, the gradient and Hessian of local functions $f_i(\mathbf{w})$ and the average function $f(\mathbf{w})=\sum_{i=1}^Nf_i(\mathbf{w})$ satisfy the following conditions
	\begin{equation}\label{bounded_var}
		\frac{1}{N} \sum_{i=1}^N \| \nabla f_i(\mathbf{w}) - \nabla f(\mathbf{w})\|^2 \leq \gamma_G^2, \qquad \frac{1}{N} \sum_{i=1}^N \| \nabla^2 f_i(\mathbf{w}) - \nabla^2 f(\mathbf{w})\|^2 \leq \gamma_H^2. 
	\end{equation}
\end{assumption}
Assumption~\ref{asm_similarity} captures the diversity between the gradients and Hessians of agents. 
Note that under Assumption \ref{asm_grad}, the conditions in Assumption~\ref{asm_similarity} are automatically satisfied for $\gamma_G = 2B$ and $\gamma_H = 2L$.

\subsection{Several Useful Lemmas}
The following result shows that, under Assumptions~\ref{asm_grad} and \ref{asm_Hesian_Lip}, the local meta functions
 $F_i(\mathbf{w}):=f_i(\mathbf{w}^i_{\nu-1} - \alpha \nabla f_i(\mathbf{w}^i_{\nu-1}))$ and their average function $F(\mathbf{w})=\frac{1}{N}\sum_{i=1}^N F_i(\mathbf{w})$ are smooth. 
\begin{lemma}\label{lemma_F_smooth}
	Recall the definition of $F_i(\mathbf{w})$ in \eqref{Meta_FLt2}.
	If  Assumptions~\ref{asm_grad} and \ref{asm_Hesian_Lip} hold, then $F_i$ is smooth
	with parameter
	\begin{align}\label{LF}
		L_F := L(1+\alpha L)^{2\nu} + B\alpha\rho(1+\alpha L)^{\nu -1}\sum_{l=0}^{\nu-1}(1+\alpha L)^l.
	\end{align}
	As a result, the average function $F(\mathbf{w})=\frac{1}{N}\sum_{i=1}^N F_i(\mathbf{w})$ is also smooth
	with parameter $L_F$.
\end{lemma}
\begin{proof}
	See Appendix \ref{App_A}.
\end{proof}
As a sanity check, when $\nu=1$ the smoothness parameter reduces to $L_F=L(1+\alpha L)^2+B\alpha\rho$, which further reduces to $4L+B\alpha\rho$ for $\alpha \leq 1/L$. This matches the smoothness parameter in \cite{fallah2020personalized} for Per-FedAvg algorithm.

The following intermediate lemma characterizes the difference between the unbiased estimates of intermediate models with the true ones.
\begin{lemma}\label{lemma_w_tilde_difference}
	Recall that the unbiased estimate of intermediate models, at each $i$-th agent, can be obtained using $l$ independent batches $\D_{i,0},\D_{i,1}, \cdots, \D_{i,l-1}$ of size $D_{i,0},D_{i,1}, \cdots, D_{i,l-1}$ as \eqref{w_l_unbiased}, for $l=1,2,,\cdots\nu$. The distances of the estimated models $\tilde{\mathbf{w}}^i_l$ with the true models ${\mathbf{w}}^i_l$ (obtained using exact/true gradients) are upper bounded as
	\begin{align}\label{w_distance}
		\|\E\left[\tilde{\mathbf{w}}^i_l-{\mathbf{w}}^i_l\right]\|&\leq\E\left[\|\tilde{\mathbf{w}}^i_l-{\mathbf{w}}^i_l\|\right]\leq h_{i,l}:=\alpha \sigma_G\sum_{j=0}^{l-1} \frac{(1+\alpha L)^j}{\sqrt{D_{i,l-1-j}}},\\
		\E\left[\|\tilde{\mathbf{w}}^i_l-{\mathbf{w}}^i_l\|^2\right]&\leq h'_{i,l}:=\frac{\alpha^2 \sigma^2_G(2+2\alpha^2 L^2)^{l-1}}{{D_{i,0}}} + 2\alpha^2 \sigma^2_G\sum_{j=1}^{l-1} \frac{(2+2\alpha^2 L^2)^{l-j-1}}{{D_{i,j}}},\\
		\E\left[\|\tilde{\mathbf{w}}^i_l-{\mathbf{w}}^i_l\|^4\right]&\leq h''_{i,l}:=\frac{\alpha^4\left(\kappa_G+3(D_{i,0}-1)\sigma^4_G\right) (8+64\alpha^4 L^4)^{l-1}}{{D^3_{i,0}}} 
		\nonumber\\
		&{\hspace{1.55cm}}+ 64\alpha^4 \sum_{j=1}^{l-1} \frac{\left(\kappa_G+3(D_{i,j}-1)\sigma^4_G\right)(8+64\alpha^4 L^4)^{l-j-1}}{{D^3_{i,j}}}.
	\end{align}
\end{lemma}
\begin{proof}
	See Appendix \ref{App_B0}.
\end{proof}

The convergence analysis of our algorithm requires the characterization of the upper bounds on the bias and variance of gradient estimation of  meta functions $F_i$'s, which are provided in the following lemma.
\begin{lemma}\label{lemma_err_moment}
	Recall  the definition of the gradient estimate $\tnabla F_i(\mathbf{w})$ in \eqref{grad_F_stoch}, computed using
	$2\nu+1$ independent batches $\left\{\D_{i,0},\cdots,\D_{i,\nu},\D'_{i,0},\cdots,\D'_{i,\nu-1}\right\}$ of size $\left\{D_{i,0},\cdots,D_{i,\nu},D'_{i,0},\cdots,D'_{i,\nu-1}\right\}$. If Assumptions~\ref{asm_grad}-\ref{asm_bounded_var_i} hold, then for any $\alpha$, $\nu>1$, and $\mathbf{w} \in \R^d$ we have
	\begin{align}
		\left \| \E \left [ \tnabla F_i(\mathbf{w}) - \nabla F_i(\mathbf{w}) \right ] \right \| & \leq  
		\mu_{F_i}:=(1+\alpha L)^{\nu}\left(\frac{\sigma_G}{\sqrt{D_{i,\nu}}}+Lh_{i,\nu}\right)+B\alpha\rho(1+\alpha L)^{\nu-1}\sum_{j=1}^{\nu-1}h_{i,j}\nonumber\\
		&\hspace{1.7cm}+\sqrt{d_{i,\nu}\left(\frac{\sigma^2_G}{D_{i,\nu}} +
			L^2 h'_{i,\nu}\right)},\\
		\E \left [ \left \| \tnabla F_i(\mathbf{w}) - \nabla F_i(\mathbf{w})\right \|^2 \right ] & \leq \sigma_{F_i}^2:=3(1+\alpha L)^{2\nu}\left(\frac{\sigma^2_G}{D_{i,\nu}} +
		L^2 h'_{i,\nu}\right) +3B^2d_{i,\nu}\nonumber\\
		&\hspace{1.7cm}+6\sqrt{2d'_{i,\nu}\left(\frac{\kappa_G+3(D_{i,\nu}-1)\sigma^4_G}{D_{i,\nu}^3}+L^4 h''_{i,\nu}\right)},	
	\end{align}
	where $h_{i,\nu}$ and $h'_{i,\nu}$ are defined in Lemma \ref{lemma_w_tilde_difference}. Moreover, $d_{i,\nu}$ and $d'_{i,\nu}$ are the solutions to the following recursive relations
	\begin{align}\label{d_nu_rec_lemma} 
		d_{i,\nu}=&2d_{i,\nu-1}\left(1+\alpha L + \alpha \frac{\sigma_H}{\sqrt{D'_{i,\nu-1}}}\right)^{\!\!2}+ 
		2\alpha^2 (1+\alpha L)^{2\nu-2}\left(\frac{\sigma^2_H}{D'_{i,\nu-1}} +
		\rho^2 h'_{i,\nu-1}\right),\\
		d'_{i,\nu} = &  {~} 64d'_{i,\nu-1}\left((1+\alpha L)^4+\alpha^4\frac{\kappa_H+3(D'_{i,\nu-1}-1)\sigma^4_H}{D'^{~\!\!3}_{i,\nu-1}}\right)\nonumber\\
		&+64\alpha^4(1+\alpha L)^{4\nu-4}\left(\rho^4 h''_{i,\nu-1}+\frac{\kappa_H+3(D'_{i,\nu-1}-1)\sigma^4_H}{D'^{~\!\!3}_{i,\nu-1}}\right),
	\end{align}
	with the initial values $d_{i,1}=\alpha^2{\sigma_H^2}/{D'_{i,0}}$ and $d'_{i,1}=\alpha^4 ({\kappa_H}+3(D'_{i,0}-1)\sigma^4_H)/{D'^{~\!\!3}_{i,0}}$, respectively.
\end{lemma}
\begin{proof}
Refer to Appendix \ref{App_B}.
\end{proof}

Next, we use the similarity conditions for the functions $f_i$ in Assumption~\ref{asm_similarity} to study the similarity between gradients of the functions $F_i$.

\begin{lemma}\label{lemma_F_similar}
	Recall the definition of $F_i(\mathbf{w})$ in \eqref{Meta_FLt2}, and suppose that the conditions in Assumptions~\ref{asm_grad}, \ref{asm_Hesian_Lip}, and \ref{asm_similarity} are satisfied. Then, for any $\mathbf{w} \in \R^d$, we have
	\begin{align}
		\frac{1}{N} \sum_{i=1}^N \| \nabla F_i(\mathbf{w}) - \nabla F(\mathbf{w})\|^2 \leq \gamma_F^2:= 15B^2g_{\nu}+6\gamma_G^2(1+\alpha L)^{2\nu} \left (1+ \alpha^2 L^2 \right ),
	\end{align}
	where 
	\begin{align}
		g_{\nu}:=\alpha^2 \gamma_H^2 (1+\alpha L)^{2\nu-2}\left(2^{\nu-1}+\sum_{l=1}^{\nu-1}2^l\right).
	\end{align}
\end{lemma}
\begin{proof}
	See Appendix \ref{App_C}.
\end{proof}


\subsection{Convergence Result}
	Recall that at any communication round $k \geq 1$, and for any agent $i \in \{1,\cdots,N\}$, we can define a sequence of local updates  
\begin{equation}\label{local_update_2}
	\mathbf{w}_{k,t}^{i} = \mathbf{w}_{k,t-1}^{i} -  \beta \tnabla F_i(\mathbf{w}_{k,t-1}^{i}),
\end{equation}
for $1 \leq t \leq \tau$, with 
$\mathbf{w}_{k,0}^{i}:= \mathbf{w}_{k-1}$. Also, let $\mu_{F}:= \max_{i} ~\!\mu_{F_i}$ and $\sigma_{F}:= \max_{i}~\! \sigma_{F_i}$. 
We then have the following theorem for the  convergence result.
\begin{theorem}\label{theorem_main} 
	Suppose that the conditions in Assumptions~\ref{asm:boundedness}-\ref{asm_similarity} are satisfied, and recall the definitions of $L_F$, $\mu_F$, $\sigma_F$, and $\gamma_F$ from Lemmas \ref{lemma_F_smooth}-\ref{lemma_F_similar}. Consider running Algorithm \ref{Alg_MetaFL} for $K$ rounds with $\tau$ local updates in each round and with $\beta \leq 1/(10 \tau L_F)$.
	Then, the following first-order stationary condition holds
	\begin{align}\label{eq_thrm}
	&\hspace{-0.5cm}	\frac{1}{\tau K} \sum_{k=0}^{K-1} \sum_{t=0}^{\tau-1} \E \left [ \| \nabla F(\bar{\mathbf{w}}_{k+1,t})\|^2 \right ] \leq 	\frac{4 (F(\mathbf{w}_0)-F^*)}{\beta \tau K}\nonumber \\
		& \hspace{0cm} +  \bigO(1) \left ( \beta L_F \left (1+ \beta L_F \tau(\tau\!-\!1) \right ) \sigma_F^2 + \beta L_F \gamma_F^2 \left (\frac{1-r}{r(n-1)} + \beta L_F \tau(\tau\!-\!1) \right) + 
		\mu^2_{F}
		\right ),
	\end{align}
	where $F^*$ is the value of $F$ at a (hypothetical) optimal solution (minimizing $F$), and  $\bar{\mathbf{w}}_{k+1,t}$ is the average of iterates of agents in $\A_k$ at time $t$, i.e., $
	\bar{\mathbf{w}}_{k+1,t} := \frac{1}{rn}  	\sum_{i \in \A_k} \mathbf{w}_{k+1,t}^i$, with $\bar{\mathbf{w}}_{k+1,0} := \mathbf{w}_k$ and $\bar{\mathbf{w}}_{k+1,\tau} := \mathbf{w}_{k+1}$.
\end{theorem}
\begin{proof}
See Appendix \ref{App_thrm}.
\end{proof}
Note that with proper choice of parameters one can attain an $\epsilon$-approximate FOSP for our algorithm (with $\epsilon\to 0$). To achieve that, one needs to satisfy $1/\beta\tau K=\mathcal{O}(\epsilon)$, $\beta=\mathcal{O}(\epsilon)$, and $\beta^2\tau(\tau-1)\approx \beta^2\tau^2=\mathcal{O}(\epsilon)$. The resulting choice of parameters is then summarized in the following remark. To quantify the behavior of the additional term $\mu_F^2$ in \eqref{eq_thrm}, it can be verified that, under the regime of large batch sizes $D\to\infty$ and small stepsizes $\alpha\to 0$, we have $\mu_F^2=\mathcal{O}(1/D)$.
\begin{remark}
Suppose that the conditions in Theorem~\ref{theorem_main} are satisfied. If we set the number of local updates as $\tau =\mathcal{O}(\epsilon^{-1/2})$, number of communication rounds with the server as $K =\mathcal{O}(\epsilon^{-3/2})$, and stepsize as $\beta =\mathcal{O}(\epsilon) $, then we find an $\mathcal{O}(\epsilon+1/{D})$-first-order stationary point of $F$.
As a result, to achieve an $\mathcal{O}(\epsilon+1/{D})$-approximate FOSP, our algorithm requires $K =\mathcal{O}(\epsilon^{-3/2})$ rounds of communication between the server and agents. Indeed, by setting $D=\mathcal{O}(\epsilon^{-1})$ the algorithm can find an $\epsilon$-first-order stationary point of $F$ for arbitrary small $\epsilon>0$.
\end{remark}

It is worth mentioning that the bounds presented in Lemmas \ref{lemma_F_smooth}-\ref{lemma_F_similar} increase with $\nu$, leading to a larger value for the right-hand side of \eqref{eq_thrm} in Theorem \ref{theorem_main}. Note that a similar observation holds for the Per-FedAvg algorithm (i.e., the case of $\nu=1$) in \cite{fallah2020personalized} compared to the FedAvg algorithm. This should not be interpreted as a slower convergence for Per-FedAvg compared to FedAvg or for larger values of $\nu$ in our algorithm. Rather, the goal of Theorem \ref{theorem_main} was to prove the convergence of our algorithm with a same \textit{asymptotic} speed $K =\mathcal{O}(\epsilon^{-3/2})$ as Per-FedAvg algorithm, thus guaranteeing the achievability of an $\epsilon$-approximate FOSP for our algorithm under any arbitrary small $\epsilon>0$. In practice, one may be more interested in the performance over finite number of communication rounds, especially given the stringent communication-efficiency requirement of FL systems. As shown in Section \ref{experiments}, our experiments reveal achieving better accuracy and faster convergence for our algorithm compared to Per-FedAvg under finite number of global epochs. 

\section{First-Order Approximations}\label{sec_1stOrder}
As seen, our algorithm requires computing Hessian-vector products which are computationally costly in some applications. In this section, we apply the first-order approximations \cite{fallah2020personalized} to the update rule and present the convergence results under these approximations.

Note that the approximations are applied to the gradient estimations to approximately calculate $\tnabla F_i(\mathbf{w})$. The true (theoretical) gradient $\nabla F_i(\mathbf{w})$ is still defined as \eqref{grad_Fit_2_last}. Therefore, the results of Lemmas \ref{lemma_F_smooth} and \ref{lemma_F_similar} remain unchanged, and thus the characterization of the convergence theorem only requires extending Lemma \ref{lemma_err_moment}, i.e., deriving $\mu_{F_i}$ and $\sigma_{F_i}^2$, under these approximation. The convergence theorem is then characterized under these approximations by substituting the new $\mu_{F_i}$ and $\sigma_{F_i}^2$ in Theorem \ref{theorem_main}.

\subsection{First-Order (FO) Approximation}
The first-order (FO) approximation is obtained by simply ignoring the second-order
terms in the update rule. As a result, using \eqref{grad_F_stoch}, the stochastic gradient is approximated as
\begin{align}
	\tnabla F_i^{\rm FO}(\mathbf{w}) := & \tnabla f_i(\tilde{\mathbf{w}}^i_{\nu-1} - \alpha \tnabla f_i(\tilde{\mathbf{w}}^i_{\nu-1},\D_{i,\nu-1}),\D_{i,\nu}).\label{grad_tilde_FO}
\end{align}

\begin{algorithm2e}[t]
	\caption{Generalized Meta First-Order Local Update;
		$\mathsf{gMeta\_FO\_Local{\_}Update}(.)$}
	\label{Alg_gMeta_FO}
	\DontPrintSemicolon
	\LinesNumbered
	\KwIn{Initial model $\mathbf{w}^{i}_{\rm in}$, local loss function $f_i(.)$, fine-tuning steps $\nu$, and  batches $\left\{\D_{i,j}\right\}_{j=0}^{\nu}$.}
	\KwOut{Updated local model $\mathbf{w}^{i}_{\rm out}$}
	$\tilde{\mathbf{w}}\leftarrow\mathbf{w}^{i}_{\rm in}$\;
	\For{$l\leftarrow 1$ \KwTo $\nu$}{
		$\tilde{\mathbf{w}}\leftarrow\tilde{\mathbf{w}} - \alpha \tnabla f_i(\tilde{\mathbf{w}},\D_{i,l-1})$\;
	}
	$\mathbf{w}_{\rm out}^{i} \leftarrow \mathbf{w}_{\rm in}^{i} - \beta \tnabla f_i(\tilde{\mathbf{w}},\D_{i,\nu})$\;
\end{algorithm2e}

The corresponding algorithm for performing the local updates under the FO approximation is 
presented in Algorithm \ref{Alg_gMeta_FO}. By substituting the function $\mathsf{gMeta\_FO\_Local{\_}Update}(.)$ into line 7 of Algorithm \ref{Alg_MetaFL}, we obtain the meta FL algorithm under this approximation. 
Moreover, to characterize the convergence behavior, we have the following result for $\mu_{F_i}$ and $\sigma_{F_i}^2$ under the FO approximation.

\begin{lemma}\label{lemma_err_moment_FO}
	Recall  the definition of the gradient estimate $\tnabla F^{\rm FO}_i(\mathbf{w})$ in \eqref{grad_tilde_FO} for the FO approximation, computed using
	$\nu+1$ independent batches $\{\D_{i,0},\cdots,\D_{i,\nu}\}$ of size $\{D_{i,0},\cdots,D_{i,\nu}\}$. If Assumptions~\ref{asm_grad}-\ref{asm_bounded_var_i} hold, then for any $\alpha$ and $\mathbf{w} \in \R^d$ we have
	\begin{align}
		\left \| \E \left [ \tnabla F_i^{\rm FO}(\mathbf{w}) - \nabla F_i(\mathbf{w}) \right ] \right \| & \leq  
		\mu^{\rm FO}_{F_i}:=\frac{\sigma_G}{\sqrt{D_{i,\nu}}}+Lh_{i,\nu}+B\left((1+\alpha L)^{\nu} + 1\right),\\
		\E \left [ \left \| \tnabla F_i^{\rm FO}(\mathbf{w}) - \nabla F_i(\mathbf{w})\right \|^2 \right ] & \leq (\sigma^{\rm FO}_{F_i})^2:=\frac{2\sigma^2_G}{D_{i,\nu}} +2
		L^2 h'_{i,\nu} + 2 B^2\left((1+\alpha L)^{\nu} + 1\right)^2,	
	\end{align}
	where $h_{i,\nu}$ and $h'_{i,\nu}$ are defined in Lemma \ref{lemma_w_tilde_difference}.
\end{lemma}
\begin{proof}
	See Appendix \ref{App_FO}.
\end{proof}

\subsection{Hessian-Free (HF) Approximation}
The idea behind the HF approximation \cite{fallah2020personalized,fallah2020convergence} is that for any function $g$, the product of the Hessian $\nabla^2 g (\mathbf{w}) $ by any vector $\mathbf{v}$ can be approximated as 
\begin{align}\label{HF_approx}
	\nabla^2 g (\mathbf{w})\mathbf{v} \approx	\frac{\nabla g(\mathbf{w}+\delta \mathbf{v})-\nabla g(\mathbf{w}-\delta \mathbf{v})}{2\delta},
\end{align}
with an error of at most $\rho \delta \|\mathbf{v}\|^2$, where $\rho $ is the Lipschitz continuity parameter of the Hessian of $g$.

Using \eqref{grad_F_stoch} and \eqref{HF_approx}, we can approximate the stochastic gradient of the meta  functions as
\begin{align}\label{grad_tilde_Fi_HF}
	\tnabla F^{\rm HF}_i(\mathbf{w}) =  \tilde{\mathbf{d}}_{\mathbf{w},i}^{(\nu-1)} - \alpha \frac{\tnabla f_i(\mathbf{w}+\delta \tilde{\mathbf{d}}_{\mathbf{w},i}^{(\nu-1)},\D'_{i,0})-\tnabla f_i(\mathbf{w}-\delta \tilde{\mathbf{d}}_{\mathbf{w},i}^{(\nu-1)},\D'_{i,0})}{2\delta},
\end{align}
where the vector $\tilde{\mathbf{d}}_{\mathbf{w},i}^{(l)}$, for $l=1,\cdots,\nu$, is recursively defined as
\begin{align}\label{d_tilde_HF}
	\tilde{\mathbf{d}}_{\mathbf{w},i}^{(l)} =  \tilde{\mathbf{d}}_{\mathbf{w},i}^{(l-1)} - \alpha \frac{\tnabla f_i(\tilde{\mathbf{w}}^i_{\nu-l}+\delta \tilde{\mathbf{d}}_{\mathbf{w},i}^{(l-1)},\D'_{i,\nu-l})-\tnabla f_i(\tilde{\mathbf{w}}^i_{\nu-l}-\delta \tilde{\mathbf{d}}_{\mathbf{w},i}^{(l-1)},\D'_{i,\nu-l})}{2\delta},
\end{align}
with the initial value of $\tilde{\mathbf{d}}_{\mathbf{w},i}^{(0)}:=\tnabla f_i(\tilde{\mathbf{w}}^i_{\nu-1} - \alpha \tnabla f_i(\tilde{\mathbf{w}}^i_{\nu-1},\D_{i,\nu-1}),\D_{i,\nu})$.
Note that $\tnabla F^{\rm HF}_i(\mathbf{w}) =  \tilde{\mathbf{d}}_{\mathbf{w},i}^{(\nu)}$ since $\tilde{\mathbf{w}}^i_0:=\mathbf{w}$. 
Algorithm \ref{Alg_gMeta_HF} 
 summarizes the local update procedure under the HF approximation.
Substituting the function $\mathsf{gMeta\_HF\_Local{\_}Update}(.)$ in line 7 of Algorithm \ref{Alg_MetaFL} results in the meta FL algorithm under the HF approximation.
The following lemma presents the bounds on $\mu_{F_i}$ and $\sigma_{F_i}^2$, required for the convergence analysis under this approximation.

\begin{algorithm2e}[t]
	\caption{Generalized Meta Hessian-Free Local Update;
		$\mathsf{gMeta\_HF\_Local{\_}Update}(.)$}
	\label{Alg_gMeta_HF}
	\DontPrintSemicolon
	\LinesNumbered
	\KwIn{Initial local model $\mathbf{w}^{i}_{\rm in}$, local loss function $f_i(.)$, number of fine-tuning steps $\nu$, HF approximation parameter $\delta$, and data batches $\left\{\D_{i,j}\right\}_{j=0}^{\nu},\{\D'_{i,j'}\}_{j'=0}^{\nu-1}$ for  agent $i$.}
	\KwOut{Updated local model $\mathbf{w}^{i}_{\rm out}$}
	$\tilde{\mathbf{w}}^{i}_{0}\leftarrow\mathbf{w}^{i}_{\rm in}$\;
	\For{$l\leftarrow 1$ \KwTo $\nu$}{
		Compute  stochastic gradient $\tnabla f_i(\tilde{\mathbf{w}}^{i}_{l-1},\D_{i,l-1})$ using dataset  $\D_{i,l-1}$\;
		Calculate and store $\tilde{\mathbf{w}}_{l}^{i}\leftarrow\tilde{\mathbf{w}}_{l-1}^{i} - \alpha \tnabla f_i(\tilde{\mathbf{w}}_{l-1}^{i},\D_{i,l-1})$\;
	}
	${\mathbf{d}}\leftarrow\tnabla f_i(\tilde{\mathbf{w}}_{\nu}^{i},\D_{i,\nu})$\;
	\For{$l'\leftarrow 1$ \KwTo $\nu$}{
		$\mathbf{g}^+_{i}\leftarrow\tnabla f_i(\tilde{\mathbf{w}}^i_{\nu-l'}+\delta {\mathbf{d}},\D'_{i,\nu-l'})$\;
		$\mathbf{g}^-_{i}\leftarrow\tnabla f_i(\tilde{\mathbf{w}}^i_{\nu-l'}-\delta {\mathbf{d}},\D'_{i,\nu-l'})$\;
		${\mathbf{d}}\leftarrow{\mathbf{d}}-\frac{\alpha}{2\delta}(\mathbf{g}^+_{i}-\mathbf{g}^-_{i})$\;
	}
	$\mathbf{w}_{\rm out}^{i} \leftarrow \mathbf{w}_{\rm in}^{i} - \beta {\mathbf{d}}$\;
\end{algorithm2e}

	\begin{lemma}\label{lemma_err_moment_HF}
		Recall  the definition of the gradient estimate $\tnabla F^{\rm HF}_i(\mathbf{w})$ in \eqref{grad_tilde_Fi_HF} for the HF approximation, computed using
		$2\nu+1$ independent batches $\{\D_{i,0},\cdots,\D_{i,\nu},\D'_{i,0},\cdots,\D'_{i,\nu-1}\}$ of size $\{D_{i,0},\cdots,D_{i,\nu},D'_{i,0},\cdots,D'_{i,\nu-1}\}$. If Assumptions~\ref{asm_grad}-\ref{asm_bounded_var_i} hold, then for any $\alpha$ and $\mathbf{w} \in \R^d$ we have
		\begin{align}
			\left \| \E \left [ \tnabla F^{\rm HF}_i(\mathbf{w}) - \nabla F^{\rm HF}_i(\mathbf{w}) \right ] \right \| & \leq  
			\mu^{\rm HF}_{F_i}:=p_{i,\nu}+q_{i,\nu},\\
			\E \left [ \left \| \tnabla F^{\rm HF}_i(\mathbf{w}) - \nabla F^{\rm HF}_i(\mathbf{w})\right \|^2 \right ] & \leq (\sigma^{\rm HF}_{F_i})^2:=2p'_{i,\nu}+2q^2_{i,\nu},	
		\end{align}
		where $p_{i,\nu}$ and $p'_{i,\nu}$ are the solutions to the recursive expressions
		\begin{align}
			p_{i,l} &:=\left(1+ \alpha L\right) p_{i,l-1}+  \frac{\alpha}{\delta}\left(\frac{\sigma_G}{\sqrt{\D'_{i,\nu-l}}} +
			L h_{i,\nu-l} \right),\\
			p'_{i,l} &:=3p'_{i,l-1}\left(1+ \alpha^2 L^2\right) +  \frac{3\alpha^2}{2\delta^2}\left( \frac{\sigma^2_G}{\D'_{i,\nu-l}} +
			2L^2 h'_{i,\nu-l}\right),
		\end{align}
		with the initial values of $p_{i,0}:={\sigma_G}/{\sqrt{D_{i,\nu}}} +
		L h_{i,\nu}$ and $p'_{i,0}:= {\sigma^2_G}/{D_{i,\nu}} +
		L^2 h'_{i,\nu}$, respectively. Moreover, $h_{i,l}$ and $h'_{i,l}$ are defined in Lemma \ref{lemma_w_tilde_difference}, and $q_{i,\nu}$ is given by
		\begin{align}\label{q_nu}
			q_{i,\nu}=\alpha \rho \delta B^2\sum_{i=0}^{\nu-1}(1+\alpha L)^{\nu+i-1}.
		\end{align}
	\end{lemma}
	\begin{proof}
		See Appendix \ref{App_HF}.
	\end{proof}

\section{Experiments}\label{experiments}
In this section, we present some experimental results to evaluate the performance of our scheme. 
Additional experiments and discussions are provided in Appendix \ref{Appendix_experiments}. 
We conduct our experiments over two widely-used datasets: CIFAR-10 and CIFAR-100 \cite{krizhevsky2009learning}. 
Both datasets consist of $60000$ $32\!\times\!32$ colour images, with
50000 training images and 10000 test images.
CIFAR-10 has $10$ classes, with $6000$ images per class, while CIFAR-100 has $100$ classes each containing $600$ images.

To heterogeneously split the data among clients, we apply the Dirichlet distribution with parameter $\alpha_d>0$ that controls the data heterogeneity among clients. The smaller values of $\alpha_d$ correspond to a more heterogeneous data splitting. Specifically, with $\alpha_d\to \infty$  all clients have identical distributions while with $\alpha_d\to 0$ each client has samples from only one class chosen at random. 

Throughout our experiments, we consider $N=50$ agents, out of which a fraction of $r=0.2$ agents are selected randomly to participate in the FL algorithm per global epoch. To do the data splitting, each agent is first assigned a probability mass function (PMF) of the same length as the number of classes (i.e., a class distribution) according to the Dirichlet distribution. Each agent is then assigned $1000$ samples, out of the training set of size $50000$, according to its specific class PMF. 
Unless explicitly mentioned, we 
 consider $\alpha=0.01$, $\beta=0.001$, $\tau=4$, $\nu=3$, and $\alpha_d=0.01$ for our experiments.
The models of all agents consist of simple fully-connected neural networks (FCNNs) with two hidden layers of size 
$n_{L_1}=80$
and $n_{L_2}=60$.
The input layer is of size $n_{\rm in}=32\times32\times3$ (corresponding to pixels of 3 color channels) and the the output layer has the same size as the number of classes (i.e., $10$) for all models. 

All agents use the same batch size of $D=40$ across all iterations. Each agent randomly splits its dataset of size $1000$ into a training set of size $800$ and a testing set of size $200$. We consider three algorithms: FedAvg, Per-FedAvg, and our generalized meta FL algorithm. At each global epoch, selected agents apply their training set of size $800$ to update the global (meta) model. We then evaluate and record the accuracy of the global model over all $N$ agents using the testing set of agents. The resulting evaluation accuracy across epochs is then plotted for different algorithms. Note that, to be fair, we evaluate the performance of all algorithms after $\nu$ fine-tuning steps (only during the evaluation and inference). 

	\begin{figure}[t]
		\centering
		\includegraphics[trim=0.3cm 1.2cm 0 0,width=6in]{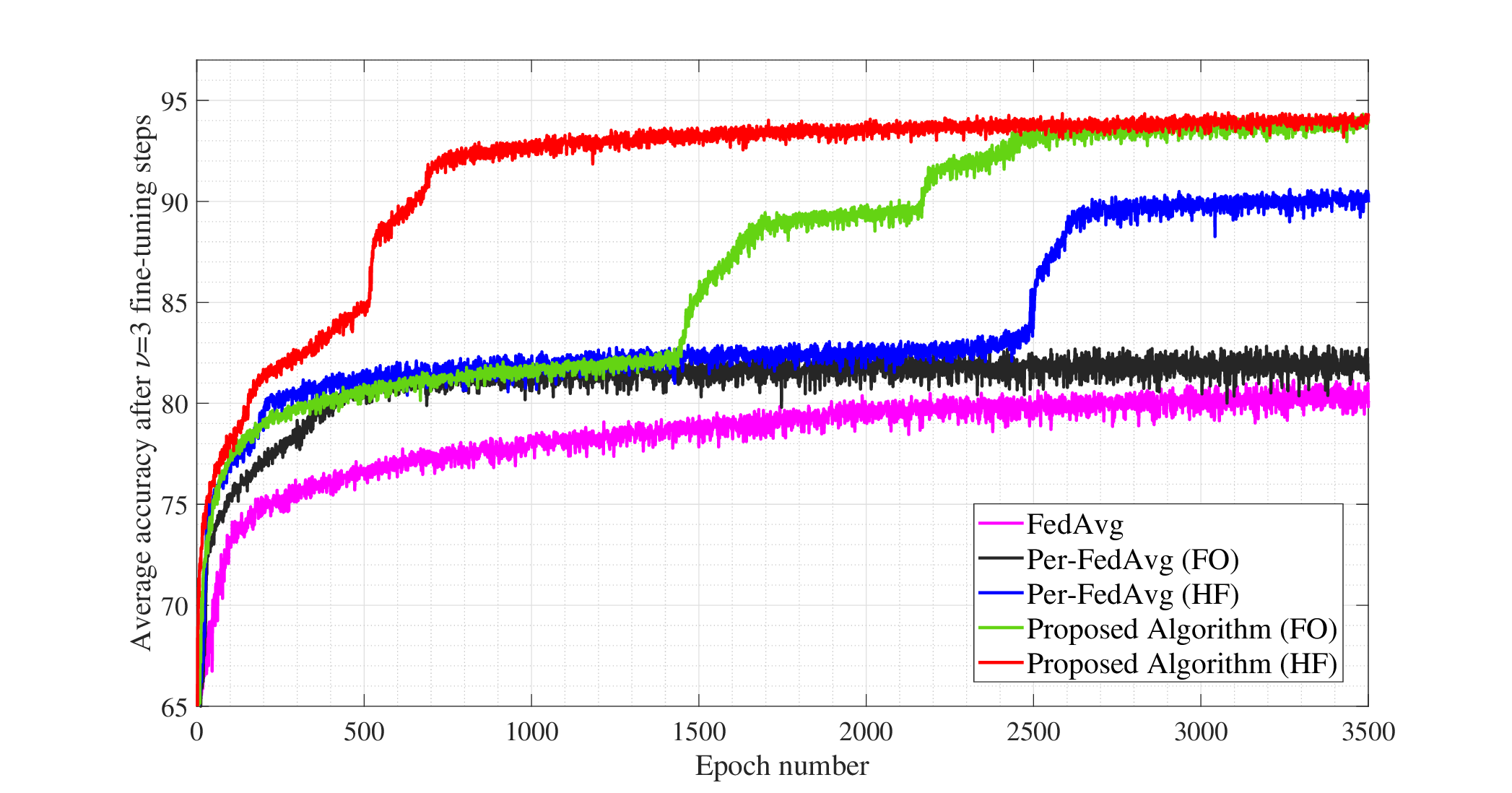}
		\caption{Accuracy $\%$ of various schemes over the CIFAR-10 dataset with $\alpha_d=0.01$.
		}
		\label{fig1}
	\end{figure}

Figure \ref{fig1} illustrates the average accuracy percentage of various algorithms over the CIFAR-10 dataset. As seen, our scheme achieves both a better accuracy and a faster convergence compared to the conventional meta FL (i.e., Per-FedAvg) and FedAvg algorithms. 
The reader is referred to Appendix \ref{Appendix_experiments} for additional experiments and discussions.	
%
%

\section{Conclusion}
In this paper, we presented a generalized framework for the meta FL by minimizing the average loss of agents on their local model after any arbitrary number $\nu$ of fine-tuning steps. In addition to analyzing the exact version of our algorithm (involving Hessian matrices), we also analyzed the convergence of our algorithm under first-order and Hessian-free approximations. Our comprehensive theoretical convergence analysis confirmed the convergence of our proposed algorithm for arbitrary $\nu$. Additionally,
our experiments on real-world datasets revealed the superior accuracy and faster convergence for the proposed scheme compared to conventional Per-FedAVg and FedAvg algorithms (corresponding to $\nu=1$ and $\nu=0$, respectively).

	\bibliographystyle{IEEEtran}
	\bibliography{Refs}

	%
	\newpage
	\appendix
	
	\section{
		Closed Form of $\nabla F_i(\mathbf{w})$  for Arbitrary $\nu\in\{1,2,\cdots\}$
	}\label{Appendix_gradF}
Throughout our proofs in this appendix and the following appendices, to simplify the notation, we drop the superscript $i$ from (intermediate) models, as it suffices to present the proofs for a given agent $i$.

   Recall the recursive relation $\mathbf{w}_{l}:=\mathbf{w}_{l-1} - \alpha \nabla f_i(\mathbf{w}_{l-1})$ for the model after ${l}$ fine-tuning steps, ${l}=1,\cdots,\nu$, with $\mathbf{w}_0:=\mathbf{w}$, and the definition of the meta functions $F_i(\mathbf{w}):=f_i(\mathbf{w}_{\nu-1} - \alpha \nabla f_i(\mathbf{w}_{\nu-1}))$ as in \eqref{Meta_FLt2}.
We first recall the following two well-known chain rule identities from matrix calculus:
\begin{itemize}
	\item 
	Let $\phi:\R^n\to \R$ be a function from vectors to scalars, and $\mathbf{A}:\R^m\to\R^n$ be a vector field. Let ${x_1,\ldots,x_m}$ denote the variables upon which the gradients are calculated. Then, we have the following chain rule.
	\begin{align}\label{rule1}
		\nabla (\phi ~\! {\rm o} ~\! \mathbf{A}) = (\nabla \mathbf{A}). (\nabla \phi ~\! {\rm o} ~\! \mathbf{A}),
	\end{align}
	where ``${\rm o}$'' denotes the function composition,  $\nabla \mathbf{A}=\mathbf{J}_{\mathbf{A}}^{\rm \mathsf{T}}$, and $\mathbf{J}_{\mathbf{A}}:=(\partial A_i/\partial x_j)_{ij}$ is the Jacobian matrix of the vector field $ \mathbf{A} =(A_{1},\ldots ,A_{n})$.
	\item  If $\mathbf{g}:\R^m\to \R^n$ is differentiable at $\mathbf{a}\in\R^m$, and  $\mathbf{h}:\R^n\to \R^l$ is differentiable at $g(\mathbf{a})\in\R^n$, then  
	\begin{align}\label{rule2}
	\mathbf{J}_{\mathbf{h}~\! {\rm o} ~\! \mathbf{g}}(\mathbf{a})=\mathbf{J}_{\mathbf{h}}(\mathbf{g}(\mathbf{a})). \mathbf{J}_{\mathbf{g}}(\mathbf{a}),
	\end{align}
	where $\mathbf{J}_{\mathbf{g}}(\mathbf{a})$ denotes the Jacobian of $\mathbf{g}$ at $\mathbf{a}$. 
\end{itemize}

Now to prove the result, note that $F_i(\mathbf{w}):=f_i(\mathbf{w}_{\nu-1} - \alpha \nabla f_i(\mathbf{w}_{\nu-1}))$ can be written as $F_i(\mathbf{w})=f_i(\mathbf{h}~\! {\rm o} ~\! \mathbf{g}(\mathbf{w}))$ with the choice of $\mathbf{g}(\mathbf{w})=\mathbf{w}_{\nu-1}$ and $\mathbf{h}(\mathbf{x})=\mathbf{x} - \alpha \nabla f_i(\mathbf{x})$. By \eqref{rule1} we have
\begin{align}\label{gF_111}
\nabla F_i(\mathbf{w}) = \nabla \mathbf{h}~\! {\rm o} ~\! \mathbf{g}(\mathbf{w}). \nabla f_i(\mathbf{w}_{\nu-1} - \alpha \nabla f_i(\mathbf{w}_{\nu-1})),
\end{align}
where, by \eqref{rule2}, we have
\begin{align}\label{gF_222}
\nabla \mathbf{h}~\! {\rm o} ~\! \mathbf{g}(\mathbf{w}) &=	\mathbf{J}^{\mathsf{T}}_{\mathbf{h}~\! {\rm o} ~\! \mathbf{g}}(\mathbf{w})\nonumber\\
&=\mathbf{J}^{\mathsf{T}}_{\mathbf{w}_{\nu-1}}\mathbf{J}^{\mathsf{T}}_{\mathbf{h}}(\mathbf{w}_{\nu-1})\nonumber\\
&=\mathbf{J}^{\mathsf{T}}_{\mathbf{w}_{\nu-1}}\left(\mathbf{I} - \alpha \nabla^2 f_i(\mathbf{w}_{\nu-1})\right),
\end{align}
where the last inequality is obtained by applying the definition of the Jacobian matrix to show that $\mathbf{J}^{\mathsf{T}}_{\mathbf{h}}(\mathbf{x})=\mathbf{I} - \alpha \nabla^2 f_i(\mathbf{x})$.

Next, we apply induction to find the closed form $\mathbf{S}_{l}$ of $\mathbf{J}^{\mathsf{T}}_{\mathbf{w}_{l}}$, for $l\in\{1,2,\ldots\}$. Given that $\mathbf{w}_{1}:=\mathbf{w} - \alpha \nabla f_i(\mathbf{w})$, for the base case of $\nu=1$ we have
\begin{align}
\mathbf{J}^{\mathsf{T}}_{\mathbf{w}_{1}} &=\mathbf{J}^{\mathsf{T}}_{\mathbf{h}}(\mathbf{w})\nonumber\\
&=\mathbf{I} - \alpha \nabla^2 f_i(\mathbf{w}) := \mathbf{S}_{1}.
\end{align}
Note that it is also possible to consider the case of $\nu=0$ as the base case, which results in $\mathbf{S}_{0}=\mathbf{J}^{\mathsf{T}}_{\mathbf{w}}=\mathbf{I}$. Now, assuming $\mathbf{J}^{\mathsf{T}}_{\mathbf{w}_{l-1}}=\mathbf{S}_{l-1}$ as the induction hypothesis, we have 
\begin{align}\label{S_rec}
	\mathbf{J}^{\mathsf{T}}_{\mathbf{w}_{l}}&\stackrel{(a)}{=}\mathbf{J}^{\mathsf{T}}_{\mathbf{w}_{l-1}}\mathbf{J}^{\mathsf{T}}_{\mathbf{h}}(\mathbf{w}_{l-1})\nonumber\\
	&\stackrel{(b)}{=}\mathbf{S}_{l-1}\left(\mathbf{I} - \alpha \nabla^2 f_i(\mathbf{w}_{l-1})\right):=\mathbf{S}_{l},
\end{align}
where step $(a)$ is by the definition of $\mathbf{w}_{l}:=\mathbf{w}_{l-1} - \alpha \nabla f_i(\mathbf{w}_{l-1})$ and then applying a similar procedure to \eqref{gF_222}, and step $(b)$ is by the induction hypothesis and $\mathbf{J}^{\mathsf{T}}_{\mathbf{h}}(\mathbf{x})=\mathbf{I} - \alpha \nabla^2 f_i(\mathbf{x})$. By solving the recursive equation in \eqref{S_rec}, we find
\begin{align}\label{S_l}
	\mathbf{J}^{\mathsf{T}}_{\mathbf{w}_{l}}:&=\mathbf{S}_{l}\nonumber\\
	&=\left(\mathbf{I} - \alpha \nabla^2 f_i(\mathbf{w})\right)\left(\mathbf{I} - \alpha \nabla^2 f_i(\mathbf{w}_{1})\right)\cdots\left(\mathbf{I} - \alpha \nabla^2 f_i(\mathbf{w}_{l-1})\right).
\end{align}
Finally, by replacing $\mathbf{J}^{\mathsf{T}}_{\mathbf{w}_{\nu-1}}$ from \eqref{S_l} in \eqref{gF_222} and then in \eqref{gF_111}, we obtain the desired result in \eqref{grad_Fit_2_last}. 	\endproof

	\section{Proof of Lemma \ref{lemma_F_smooth}}\label{App_A}
	Recall that 
	\begin{align}
		\nabla F_i(\mathbf{w}) = & \left(\mathbf{I} - \alpha \nabla^2 f_i(\mathbf{w})\right) \left(\mathbf{I} - \alpha \nabla^2 f_i(\mathbf{w}_1)\right) \left(\mathbf{I} - \alpha \nabla^2 f_i(\mathbf{w}_2)\right) \nonumber\\
		& {\hspace{1.5in}} \cdots \left(\mathbf{I} - \alpha \nabla^2 f_i(\mathbf{w}_{\nu-1})\right) \nabla f_i(\mathbf{w}_{\nu-1} - \alpha \nabla f_i(\mathbf{w}_{\nu-1})). 
	\end{align}
	Let us define the matrix $\mathbf{B}^{(\nu)}_{\mathbf{w}}$, for $\nu\geq 1$, as
	\begin{align}\label{B_w}
		\mathbf{B}^{(\nu)}_{\mathbf{w}} := & \left(\mathbf{I} - \alpha \nabla^2 f_i(\mathbf{w})\right) \left(\mathbf{I} - \alpha \nabla^2 f_i(\mathbf{w}_1)\right) \left(\mathbf{I} - \alpha \nabla^2 f_i(\mathbf{w}_2)\right) \cdots \left(\mathbf{I} - \alpha \nabla^2 f_i(\mathbf{w}_{\nu-1})\right). 
	\end{align}
	Obviously, $\mathbf{B}^{(\nu)}_{\mathbf{w}}$ can be recursively expressed as
	\begin{align}\label{B_w_rec}
		\mathbf{B}^{(\nu)}_{\mathbf{w}} &=  \mathbf{B}^{(\nu-1)}_{\mathbf{w}} \left(\mathbf{I} - \alpha \nabla^2 f_i(\mathbf{w}_{\nu-1})\right),\nonumber\\
		\mathbf{B}^{(1)}_{\mathbf{w}} &=
		\mathbf{I} - \alpha \nabla^2 f_i(\mathbf{w}). 
	\end{align}
	We can then rewrite  $\nabla F_i(\mathbf{w})$ for our algorithm with $\nu$ fine-tuning steps as
	\begin{align}\label{grad_Fi_rec}
		\nabla F_i(\mathbf{w}) =  \mathbf{B}^{(\nu)}_{\mathbf{w}} \nabla f_i(\mathbf{w}_{\nu-1} - \alpha \nabla f_i(\mathbf{w}_{\nu-1})). 
	\end{align}
	The objective then is to upper bound $\|\nabla F_i(\mathbf{w})-\nabla F_i(\mathbf{u})\|$ as a factor of $\|\mathbf{w}-\mathbf{u}\|$. Using \eqref{grad_Fi_rec} we have
	\begin{align}
		\|\nabla F_i(\mathbf{w})-\nabla F_i(\mathbf{u})\| & = \|\mathbf{B}^{(\nu)}_{\mathbf{w}} \nabla f_i(\mathbf{w}_{\nu-1} - \alpha \nabla f_i(\mathbf{w}_{\nu-1})) - \mathbf{B}^{(\nu)}_{\mathbf{u}} \nabla f_i(\mathbf{u}_{\nu-1} - \alpha \nabla f_i(\mathbf{u}_{\nu-1})) \| \nonumber\\
		& = \|\mathbf{B}^{(\nu)}_{\mathbf{w}} \nabla f_i(\mathbf{w}_{\nu-1} - \alpha \nabla f_i(\mathbf{w}_{\nu-1})) - \mathbf{B}^{(\nu)}_{\mathbf{w}} \nabla f_i(\mathbf{u}_{\nu-1} - \alpha \nabla f_i(\mathbf{u}_{\nu-1})) \nonumber\\
		& ~~~~~~\!\! + \mathbf{B}^{(\nu)}_{\mathbf{w}} \nabla f_i(\mathbf{u}_{\nu-1} - \alpha \nabla f_i(\mathbf{u}_{\nu-1})) -  \mathbf{B}^{(\nu)}_{\mathbf{u}} \nabla f_i(\mathbf{u}_{\nu-1} - \alpha \nabla f_i(\mathbf{u}_{\nu-1})) \|\nonumber\\
		& \leq \|\mathbf{B}^{(\nu)}_{\mathbf{w}} \left(\nabla f_i(\mathbf{w}_{\nu-1} - \alpha \nabla f_i(\mathbf{w}_{\nu-1})) - \nabla f_i(\mathbf{u}_{\nu-1} - \alpha \nabla f_i(\mathbf{u}_{\nu-1}))\right) \| \nonumber\\
		& ~~~~ + \|\left(\mathbf{B}^{(\nu)}_{\mathbf{w}} - \mathbf{B}^{(\nu)}_{\mathbf{u}}\right) \nabla f_i(\mathbf{u}_{\nu-1} - \alpha \nabla f_i(\mathbf{u}_{\nu-1})) \|.\label{eq51}
	\end{align}
	The first norm in \eqref{eq51} can be upper bounded as
	\begin{align}\label{eq51_norm1}
		& \|\mathbf{B}^{(\nu)}_{\mathbf{w}} \left(\nabla f_i(\mathbf{w}_{\nu-1} - \alpha \nabla f_i(\mathbf{w}_{\nu-1})) - \nabla f_i(\mathbf{u}_{\nu-1} - \alpha \nabla f_i(\mathbf{u}_{\nu-1}))\right) \| \nonumber\\
		& \leq \|\mathbf{B}^{(\nu)}_{\mathbf{w}}\| \|\nabla f_i(\mathbf{w}_{\nu-1} - \alpha \nabla f_i(\mathbf{w}_{\nu-1})) - \nabla f_i(\mathbf{u}_{\nu-1} - \alpha \nabla f_i(\mathbf{u}_{\nu-1})) \| \nonumber\\
		& = \|\prod_{l=0}^{\nu-1}\left(\mathbf{I} - \alpha \nabla^2 f_i(\mathbf{w}_{l})\right)\| \|\nabla f_i(\mathbf{w}_{\nu-1} - \alpha \nabla f_i(\mathbf{w}_{\nu-1})) - \nabla f_i(\mathbf{u}_{\nu-1} - \alpha \nabla f_i(\mathbf{u}_{\nu-1})) \|.
	\end{align}
	The first norm in \eqref{eq51_norm1} can be upper bounded as
	\begin{align}\label{eq51_norm1_1}
		\big\|\prod_{l=0}^{\nu-1}\left(\mathbf{I} - \alpha \nabla^2 f_i(\mathbf{w}_{l})\right)\big\| &\leq \prod_{l=0}^{\nu-1}\|\left(\mathbf{I} - \alpha \nabla^2 f_i(\mathbf{w}_{l})\right)\| \nonumber\\ &\leq
		\prod_{l=0}^{\nu-1}\left(\|\mathbf{I}\| + \alpha \|\nabla^2 f_i(\mathbf{w}_{l})\|\right) \nonumber\\ &\leq
		\prod_{l=0}^{\nu-1}\left(1+\alpha L\right)\nonumber\\
		& = (1+\alpha L)^{\nu}.
	\end{align}
	Additionally, for the second norm in \eqref{eq51_norm1} we have 
	\begin{align}
		&\| \left(\nabla f_i(\mathbf{w}_{\nu-1} - \alpha \nabla f_i(\mathbf{w}_{\nu-1})) - \nabla f_i(\mathbf{u}_{\nu-1} - \alpha \nabla f_i(\mathbf{u}_{\nu-1}))\right)\|\nonumber\\
		&\leq L\|\mathbf{w}_{\nu-1} - \alpha \nabla f_i(\mathbf{w}_{\nu-1}) - \mathbf{u}_{\nu-1} + \alpha \nabla f_i(\mathbf{u}_{\nu-1})\|\nonumber\\
		&\leq L\|\mathbf{w}_{\nu-1} - \mathbf{u}_{\nu-1}\| + \alpha L \|\nabla f_i(\mathbf{w}_{\nu-1}) -  \nabla f_i(\mathbf{u}_{\nu-1})\|
		\nonumber\\
		&\leq L(1+\alpha L) \|\mathbf{w}_{\nu-1} - \mathbf{u}_{\nu-1}\|\\
		&= L(1+\alpha L) \|\mathbf{w}_{\nu-2} - \alpha \nabla f_i(\mathbf{w}_{\nu-2})  - \mathbf{u}_{\nu-2} + \alpha \nabla f_i(\mathbf{u}_{\nu-2}) \|\|
		\nonumber\\
		&\leq L(1+\alpha L)^2 \|\mathbf{w}_{\nu-2} - \mathbf{u}_{\nu-2}\|\nonumber\\
		&~~~~~~~~~~~~\vdots\nonumber\\
		&\leq L(1+\alpha L)^{\nu} \|\mathbf{w}_0 - \mathbf{u}_0\| = L(1+\alpha L)^{\nu} \|\mathbf{w} - \mathbf{u}\|.\label{eq51_norm1_2}
	\end{align}
	Replacing \eqref{eq51_norm1_1} and \eqref{eq51_norm1_2} in \eqref{eq51_norm1} then gives
	\begin{align}\label{eq51_norm1_tot}
		\|\mathbf{B}^{(\nu)}_{\mathbf{w}} \left(\nabla f_i(\mathbf{w}_{\nu-1} - \alpha \nabla f_i(\mathbf{w}_{\nu-1})) - \nabla f_i(\mathbf{u}_{\nu-1} - \alpha \nabla f_i(\mathbf{u}_{\nu-1}))\right) \| \leq L(1+\alpha L)^{2\nu} \|\mathbf{w} - \mathbf{u}\|. 
	\end{align}
	
	On the other hand, the second norm in \eqref{eq51} can be bounded as
	\begin{align} \label{eq51_norm2}
		&\|\left(\mathbf{B}^{(\nu)}_{\mathbf{w}} - \mathbf{B}^{(\nu)}_{\mathbf{u}}\right) \nabla f_i(\mathbf{u}_{\nu-1} - \alpha \nabla f_i(\mathbf{u}_{\nu-1})) \| \leq B \|\mathbf{B}^{(\nu)}_{\mathbf{w}} - \mathbf{B}^{(\nu)}_{\mathbf{u}}\|.
	\end{align}
	We now apply the recursive relation in \eqref{B_w_rec} to prove by induction that $\|\mathbf{B}^{(\nu)}_{\mathbf{w}} - \mathbf{B}^{(\nu)}_{\mathbf{u}}\|$ can be upper bounded by a factor of $\|\mathbf{w}-\mathbf{u}\|$. Note that the base case for $\nu=1$ clearly holds as we have
	\begin{align}\label{eq58}
		\|\mathbf{B}^{(1)}_{\mathbf{w}} - \mathbf{B}^{(1)}_{\mathbf{u}}\| &= \|\left(\mathbf{I} - \alpha \nabla^2 f_i(\mathbf{w})\right) - \left(\mathbf{I} - \alpha \nabla^2 f_i(\mathbf{u})\right)\|\nonumber\\
		& = \alpha \|\nabla^2 f_i(\mathbf{w})-\nabla^2 f_i(\mathbf{u})\|\nonumber\\
		& \leq \alpha \rho \|\mathbf{w}-\mathbf{u}\|.
	\end{align}
	Now assuming that step $\nu-1$ holds with some factor $c_{\nu -1}$, i.e., 
	\begin{align}\label{lemma4_2_rec_hypot}
		\|\mathbf{B}^{(\nu -1)}_{\mathbf{w}} - \mathbf{B}^{(\nu -1)}_{\mathbf{u}}\| \leq c_{\nu -1} \|\mathbf{w}-\mathbf{u}\|,
	\end{align}
	we want to show that step $\nu$ holds with some factor $c_{\nu}$. Using \eqref{B_w_rec} we have
	\begin{align}\label{lemma4_2_rec_ind}
		\|\mathbf{B}^{(\nu)}_{\mathbf{w}} - \mathbf{B}^{(\nu)}_{\mathbf{u}}\| &= \|\mathbf{B}^{(\nu-1)}_{\mathbf{w}} \left(\mathbf{I} - \alpha \nabla^2 f_i(\mathbf{w}_{\nu-1})\right)- \mathbf{B}^{(\nu-1)}_{\mathbf{u}} \left(\mathbf{I} - \alpha \nabla^2 f_i(\mathbf{u}_{\nu-1})\right)\| \nonumber\\
		& = \|\mathbf{B}^{(\nu-1)}_{\mathbf{w}} \left(\mathbf{I} - \alpha \nabla^2 f_i(\mathbf{w}_{\nu-1})\right)-
		\mathbf{B}^{(\nu-1)}_{\mathbf{u}} \left(\mathbf{I} - \alpha \nabla^2 f_i(\mathbf{w}_{\nu-1})\right) \nonumber\\
		& ~~~~ + \mathbf{B}^{(\nu-1)}_{\mathbf{u}} \left(\mathbf{I} - \alpha \nabla^2 f_i(\mathbf{w}_{\nu-1})\right) -
		\mathbf{B}^{(\nu-1)}_{\mathbf{u}} \left(\mathbf{I} - \alpha \nabla^2 f_i(\mathbf{u}_{\nu-1})\right)\| \nonumber\\
		& \hspace{-0.5in}\leq \|\left(\mathbf{B}^{(\nu-1)}_{\mathbf{w}} - \mathbf{B}^{(\nu-1)}_{\mathbf{u}}\right)\left(\mathbf{I} - \alpha \nabla^2 f_i(\mathbf{w}_{\nu-1})\right)\|+ \alpha\|\mathbf{B}^{(\nu-1)}_{\mathbf{u}} \left(\nabla^2 f_i(\mathbf{u}_{\nu-1})-\nabla^2 f_i(\mathbf{w}_{\nu-1})\right)\| \nonumber\\
		& \hspace{-0.5in}\leq (1+\alpha L)\|\mathbf{B}^{(\nu-1)}_{\mathbf{w}} - \mathbf{B}^{(\nu-1)}_{\mathbf{u}}\|+ \alpha \rho\|\mathbf{B}^{(\nu-1)}_{\mathbf{u}}\| \|\mathbf{w}_{\nu-1}-\mathbf{u}_{\nu-1}\| \nonumber\\
		& \hspace{-0.5in}\leq (1+\alpha L)c_{\nu -1} \|\mathbf{w}-\mathbf{u}\|+ \alpha \rho(1+\alpha L)^{2\nu -2}\|\mathbf{w}-\mathbf{u}\| \nonumber\\
		& \hspace{-0.5in} = (1+\alpha L)\left(c_{\nu-1}+\alpha\rho(1+\alpha L)^{2\nu-3}\right)\|\mathbf{w}-\mathbf{u}\|,
	\end{align}
	where the last inequality follows \eqref{lemma4_2_rec_hypot}, $\|\mathbf{B}^{(\nu-1)}_{\mathbf{u}}\|\leq (1+\alpha L)^{\nu -1}$ (as shown in \eqref{eq51_norm1_1}), and $\|\mathbf{w}_{\nu-1}-\mathbf{u}_{\nu-1}\|\leq (1+\alpha L)^{\nu -1}\|\mathbf{w}-\mathbf{u}\|$ (as seen in \eqref{eq51_norm1_2}). Eq. \eqref{lemma4_2_rec_ind} thus completes the induction prove by establishing that step $\nu$ holds with factor
	\begin{align}\label{c_rec}
		c_{\nu}=(1+\alpha L)\left(c_{\nu-1}+\alpha\rho(1+\alpha L)^{2\nu-3}\right).
	\end{align}
	We can also find the closed form for $c_{\nu}$ using the recursive relation in \eqref{c_rec} and the initial value $c_1 = \alpha\rho$ (from \eqref{eq58}) as
	\begin{align}\label{c_nu}
		c_{\nu}=\alpha\rho(1+\alpha L)^{\nu -1}\sum_{l=0}^{\nu-1}(1+\alpha L)^l.
	\end{align}
	
	Putting all together, we have
	\begin{align}\label{smooth}
		\|\nabla F_i(\mathbf{w})-\nabla F_i(\mathbf{u})\| \leq 
		\left(L(1+\alpha L)^{2\nu} + B\alpha\rho(1+\alpha L)^{\nu -1}\sum_{l=0}^{\nu-1}(1+\alpha L)^l\right)  \|\mathbf{w} - \mathbf{u}\|,
	\end{align}
	which shows that the meta loss functions $F_i$'s are $L_F$-smooth, where
	\begin{align}\label{L_F}
		L_F := L(1+\alpha L)^{2\nu} + B\alpha\rho(1+\alpha L)^{\nu -1}\sum_{l=0}^{\nu-1}(1+\alpha L)^l.
	\end{align}
	
	Finally, note that for the average function $F(\mathbf{w})$ we have
	\begin{align}\label{smoothness_F}
		\|\nabla F(\mathbf{w}) - \nabla F(\mathbf{u})\| &= \frac{1}{N} \|\sum_{i=1}^{N}\left(\nabla F_i(\mathbf{w}) - \nabla F_i(\mathbf{u})\right)\|\nonumber\\
		& \leq \frac{1}{N} \sum_{i=1}^{N}\|\nabla F_i(\mathbf{w}) - \nabla F_i(\mathbf{u})\| \leq L_F\|\mathbf{w} - \mathbf{u}\|,
	\end{align}
	i.e., the average function $F(\mathbf{w})$ is also $L_F$-smooth. 
	\endproof
	
	\section{An Intermediate Result}\label{Appendix_intermediate_Lemma}
	The following lemma is useful in the proof of Lemmas \ref{lemma_w_tilde_difference} and \ref{lemma_err_moment}.
	
	\begin{lemma}\label{lemma_intermediate}
		The fourth central moments of the gradient estimate $\tnabla f_i({\mathbf{w}},\D)$ and Hessian estimate $\tnabla^2 f_i({\mathbf{w}},\D)$, calculated using a batch $\D$ of size $D$ at an arbitrary model $\mathbf{w}$, are upper bounded as
		\begin{align}
			&\E \left [ \left \| \tnabla f_i({\mathbf{w}},\D) - \nabla f_i({\mathbf{w}})\right \|^4 \right ]\leq 
			\frac{\kappa_G+3(D-1)\sigma^4_G}{D^3},\label{4th_moment}\\
			&\E \left [ \left \| \tnabla^2 f_i({\mathbf{w}},\D) - \nabla^2 f_i({\mathbf{w}})\right \|^4 \right ]\leq 
			\frac{\kappa_H+3(D-1)\sigma^4_H}{D^3}.\label{4th_moment_kappaH}
		\end{align}
	\end{lemma}
	\proof 
	To prove the first result, using \eqref{unbiased_grad} we have
	\begin{align}\label{lemB1_1}
		\left \| \tnabla f_i({\mathbf{w}},\D) - \nabla f_i({\mathbf{w}})\right \|^4= 	\frac{1}{D^4} \left \|  \sum_{(\mathbf{x},y) \in \D} \left[\nabla l_i(\mathbf{w};\mathbf{x},y)- \nabla f_i({\mathbf{w}})\right]\right \|^4.
	\end{align}
	For the simplicity of notation, let us enumerate the data samples in $\D$ by an index set $\mathcal{I}=\{1,\cdots,D\}$, that is $\D=\{(\mathbf{x}_k,y_k), k\in \mathcal{I}\}$. We additionally define the random vectors $\mathbf{g}_{i,k}:=\nabla l_i(\mathbf{w};\mathbf{x}_k,y_k)- \nabla f_i({\mathbf{w}})$ and $\mathbf{s_i}:=\sum_{(\mathbf{x},y) \in \D} \left[\nabla l_i(\mathbf{w};\mathbf{x},y)- \nabla f_i({\mathbf{w}})\right]=\sum_{k=1}^D \mathbf{g}_{i,k}$. Note that the samples of the batch are independently selected, i.e., the data pairs $(\mathbf{x}_k,y_k)$'s are independent. Accordingly, $\mathbf{g}_{i,k}$'s can be regarded as independent random vectors for different values of $k$. Additionally, \eqref{supervised_ML} implies that $\mathbf{g}_{i,k}$'s are zero mean.
	Using the definition of the vector norm, i.e., $\left \| \mathbf{s_i} \right \|^2 = \mathbf{s_i}^{\mathsf{T}}\mathbf{s_i}$, where $(.)^{\mathsf{T}}$ denotes the transpose operation, we have
	\begin{align}\label{lemB1_2}
		\E \left [ \left \| \tnabla f_i({\mathbf{w}},\D) - \nabla f_i({\mathbf{w}})\right \|^4 \right ] & = \frac{1}{D^4} \E \left[ \left \| \mathbf{s_i} \right \|^4 \right]\nonumber\\
		& 
		= \frac{1}{D^4} \E \left[ \mathbf{s_i}^{\mathsf{T}}\mathbf{s_i}\mathbf{s_i}^{\mathsf{T}}\mathbf{s_i}  \right]\nonumber\\
		& = \frac{1}{D^4} \E \left[ \left(\sum_{k_1=1}^D \mathbf{g}^{\mathsf{T}}_{i,k_1}\right)\left(\sum_{k_2=1}^D \mathbf{g}_{i,k_2}\right)\left(\sum_{k_3=1}^D \mathbf{g}^{\mathsf{T}}_{i,k_3}\right)\left(\sum_{k_4=1}^D \mathbf{g}_{i,k_4}\right) \right] \nonumber\\
		& = \frac{1}{D^4} \E \left[ \sum_{k_1,k_2,k_3,k_4} \mathbf{g}^{\mathsf{T}}_{i,k_1}\mathbf{g}_{i,k_2}\mathbf{g}^{\mathsf{T}}_{i,k_3}\mathbf{g}_{i,k_4}\right]
\nonumber\\
& \leq \frac{\kappa_G+3(D-1)\sigma^4_G}{D^3},
	\end{align}
where the last inequity is obtained by dividing the $D^4$ terms inside the summation into the following five categories.
\begin{itemize}
	\item \textbf{Case 1:} All four indices are identical, i.e., $k_1=k_2=k_3=k_4=k$ for some $k\in \mathcal{I}$. Using Assumption \ref{asm_bounded_var_i}, the expected value of such terms can be bounded as
	\begin{align}\label{E_1_identical}
		\E\left[\mathbf{g}^{\mathsf{T}}_{i,k}\mathbf{g}_{i,k}\mathbf{g}^{\mathsf{T}}_{i,k}\mathbf{g}_{i,k}\right]=\E\left[\left\|\mathbf{g}_{i,k}\right\|^4\right]\leq \kappa_G.
	\end{align}
Note that we have $D$ such terms, which contribute to $\kappa_G/D^3$ in the upper bound of \eqref{lemB1_2}.
\item \textbf{Case 2:} Three out of four indices are identical. These are terms of the form $\mathbf{g}^{\mathsf{T}}_{i,k'}\mathbf{g}_{i,k}\mathbf{g}^{\mathsf{T}}_{i,k}\mathbf{g}_{i,k}$, $\mathbf{g}^{\mathsf{T}}_{i,k}\mathbf{g}_{i,k'}\mathbf{g}^{\mathsf{T}}_{i,k}\mathbf{g}_{i,k}$, $\mathbf{g}^{\mathsf{T}}_{i,k}\mathbf{g}_{i,k}\mathbf{g}^{\mathsf{T}}_{i,k'}\mathbf{g}_{i,k}$, and $\mathbf{g}^{\mathsf{T}}_{i,k}\mathbf{g}_{i,k}\mathbf{g}^{\mathsf{T}}_{i,k}\mathbf{g}_{i,k'}$, for $k\neq k'$. By applying the tower rule, together with the fact that $\mathbf{g}_{i,k}$'s are zero mean and independent, it is easy to verify that the expected value of all these $4D(D-1)$ terms is zero.
\item \textbf{Case 3:} There are two pairs of identical indices. Depending on the distance of identical terms, we have the following three cases:
\begin{itemize}
	\item \textbf{Case 3.1:} When the identical indices are adjacent, i.e., for terms of the form $\mathbf{g}^{\mathsf{T}}_{i,k}\mathbf{g}_{i,k}\mathbf{g}^{\mathsf{T}}_{i,k'}\mathbf{g}_{i,k'}$, with $k\neq k'$, the expected value can be upper bounded (using Assumption \ref{asm_bounded_var_i} and independence of $\mathbf{g}_{i,k}$ from $\mathbf{g}_{i,k'}$) as
	\begin{align}
		\E\left[\mathbf{g}^{\mathsf{T}}_{i,k}\mathbf{g}_{i,k}\mathbf{g}^{\mathsf{T}}_{i,k'}\mathbf{g}_{i,k'}\right]&=\E\left[\left\|\mathbf{g}_{i,k}\right\|^2\left\|\mathbf{g}_{i,k'}\right\|^2\right]\nonumber\\
		&=\E\left[\left\|\mathbf{g}_{i,k}\right\|^2\right]\E\left[\left\|\mathbf{g}_{i,k'}\right\|^2\right]\nonumber\\
		&\leq \sigma^4_G.
	\end{align}
	\item \textbf{Case 3.2:} When the two middle indices are identical, i.e., for terms of the form $\mathbf{g}^{\mathsf{T}}_{i,k}\mathbf{g}_{i,k'}\mathbf{g}^{\mathsf{T}}_{i,k'}\mathbf{g}_{i,k}$, with $k\neq k'$, the expected value can be upper bounded as
\begin{align}\label{case3_2}
	\E\left[\mathbf{g}^{\mathsf{T}}_{i,k}\mathbf{g}_{i,k'}\mathbf{g}^{\mathsf{T}}_{i,k'}\mathbf{g}_{i,k}\right]&\stackrel{(a)}{=}\E\left[{\rm \mathsf{tr}}\left(\mathbf{g}^{\mathsf{T}}_{i,k}\mathbf{g}_{i,k'}\mathbf{g}^{\mathsf{T}}_{i,k'}\mathbf{g}_{i,k}\right)\right]\nonumber\\
	&\stackrel{(b)}{=}\E\left[{\rm \mathsf{tr}}\left(\mathbf{g}_{i,k}\mathbf{g}^{\mathsf{T}}_{i,k}\mathbf{g}_{i,k'}\mathbf{g}^{\mathsf{T}}_{i,k'}\right)\right]\nonumber\\
	&\stackrel{(c)}{=}{\rm \mathsf{tr}}\left(\E\left[\mathbf{g}_{i,k}\mathbf{g}^{\mathsf{T}}_{i,k}\mathbf{g}_{i,k'}\mathbf{g}^{\mathsf{T}}_{i,k'}\right]\right)\nonumber\\
	&={\rm \mathsf{tr}}\left(\left(\E\left[\mathbf{g}_{i,k}\mathbf{g}^{\mathsf{T}}_{i,k}\right]\right)^2\right)\nonumber\\
	&\leq \left({\rm \mathsf{tr}}\left(\E\left[\mathbf{g}_{i,k}\mathbf{g}^{\mathsf{T}}_{i,k}\right]\right)\right)^2\nonumber\\
	&= \left(\E\left[{\rm \mathsf{tr}}\left(\mathbf{g}_{i,k}\mathbf{g}^{\mathsf{T}}_{i,k}\right)\right]\right)^2\nonumber\\
	&= \left(\E\left[\left\|\mathbf{g}_{i,k}\right\|^2\right]\right)^2\nonumber\\
	&\leq \sigma^4_G,
\end{align}
where step $(a)$ is by the fact that the term inside the expectation is a scalar and the trace ${\rm \mathsf{tr}}(c)$ of a scalar $c$ is equal to $c$. Step $(b)$ is by ${\rm \mathsf{tr}}(\mathbf{AB})={\rm \mathsf{tr}}(\mathbf{BA})$ for matrices $\mathbf{A}=\mathbf{g}^{\mathsf{T}}_{i,k}\mathbf{g}_{i,k'}\mathbf{g}^{\mathsf{T}}_{i,k'}$ and $\mathbf{B}=\mathbf{g}_{i,k}$. Moreover, step $(c)$ follows $\E\left[{\rm \mathsf{tr}}(\mathbf{A})\right]={\rm \mathsf{tr}}\left(\E\left[\mathbf{A}\right]\right)$ for arbitrary square matrix $\mathbf{A}$.
\item \textbf{Case 3.3:} When the identical indices are alternating, i.e., for terms of the form $\mathbf{g}^{\mathsf{T}}_{i,k}\mathbf{g}_{i,k'}\mathbf{g}^{\mathsf{T}}_{i,k}\mathbf{g}_{i,k'}$, with $k\neq k'$, the expected value can be upper bounded by $\sigma^4_G$. This is because $\mathbf{g}^{\mathsf{T}}_{i,k}\mathbf{g}_{i,k'}=\mathbf{g}^{\mathsf{T}}_{i,k'}\mathbf{g}_{i,k}$, and thus the derivation in \eqref{case3_2} readily applies.
\end{itemize}
Note that we have $D(D-1)$ terms of each of the above three forms, which contribute to the term $3(D-1)\sigma^4_G/D^3$ in the upper bound of \eqref{lemB1_2}. 
\item \textbf{Case 4:} Two out of four indices are identical and  the other two indices are distinct. This includes $6D(D-1)(D-2)$ terms of the form $\mathbf{g}^{\mathsf{T}}_{i,k}\mathbf{g}_{i,k}\mathbf{g}^{\mathsf{T}}_{i,k'}\mathbf{g}_{i,k''}$, $\mathbf{g}^{\mathsf{T}}_{i,k}\mathbf{g}_{i,k'}\mathbf{g}^{\mathsf{T}}_{i,k}\mathbf{g}_{i,k''}$, $\mathbf{g}^{\mathsf{T}}_{i,k}\mathbf{g}_{i,k'}\mathbf{g}^{\mathsf{T}}_{i,k''}\mathbf{g}_{i,k}$, $\mathbf{g}^{\mathsf{T}}_{i,k'}\mathbf{g}_{i,k}\mathbf{g}^{\mathsf{T}}_{i,k}\mathbf{g}_{i,k''}$, $\mathbf{g}^{\mathsf{T}}_{i,k'}\mathbf{g}_{i,k}\mathbf{g}^{\mathsf{T}}_{i,k''}\mathbf{g}_{i,k}$, and $\mathbf{g}^{\mathsf{T}}_{i,k'}\mathbf{g}_{i,k''}\mathbf{g}^{\mathsf{T}}_{i,k}\mathbf{g}_{i,k}$, with distinct $k,k',k''$. Similar to Case 2, the application of tower rule shows that all such terms have a zero mean.
\item \textbf{Case 5:} All four indices are distinct. This includes $D(D-1)(D-2)(D-3)$ terms of the form $\mathbf{g}^{\mathsf{T}}_{i,k_1}\mathbf{g}_{i,k_2}\mathbf{g}^{\mathsf{T}}_{i,k_3}\mathbf{g}_{i,k_4}$, with distinct $k_1,k_2,k_3,k_4$. The expectation of all such terms is zero due to the fact that $\mathbf{g}_{i,k}$'s are zero mean and independent.
\end{itemize}
Note that the sum of the number of terms in the above five cases is equal to $D^4$, i.e., the above partition provides a disjoint partition covering all cases.
\begin{align}
	&D(D-1)(D-2)(D-3)+6D(D-1)(D-2)+4D(D-1)+3D(D-1)+D\nonumber\\
	&=D\left[1+(D-1)(D^2+D+1)\right]\nonumber\\
	&=D^4.
\end{align}

The proof of the second result \eqref{4th_moment_kappaH} follows similarly by considering the above cases and using various definitions of the matrix norm (e.g., the Frobenius norm). \endproof

	\section{Proof of Lemma \ref{lemma_w_tilde_difference}}\label{App_B0}
	We use induction to characterize the bounds $h_l$, $h'_l$, and $h''_l$ on $\E\left[\|\tilde{\mathbf{w}}_l-{\mathbf{w}}_l\|\right]$, $\E\left[\|\tilde{\mathbf{w}}_l-{\mathbf{w}}_l\|^2\right]$, and $\E\left[\|\tilde{\mathbf{w}}_l-{\mathbf{w}}_l\|^4\right]$, respectively, for $l=1,2,\cdots, \nu$. 
	To get the first result, note that for the base case of $l=1$ we have
	\begin{align}\label{h_1}
		\E\left[\|\tilde{\mathbf{w}}_1-{\mathbf{w}}_1\|\right]&=\E\left[{\big\|}\left(\tilde{\mathbf{w}}_{0} -{\mathbf{w}}_{0} \right) + \alpha \left(\nabla f_i({\mathbf{w}}_{0})- \tnabla f_i(\tilde{\mathbf{w}}_{0},\D_{i,0})\right){\big\|}\right]\nonumber\\
		&\stackrel{(a)}{=}  \alpha \E\left[{\big\|} \tnabla f_i({\mathbf{w}},\D_{i,0})-\nabla f_i({\mathbf{w}}){\big\|}\right]\nonumber\\
		&\stackrel{(b)}{\leq}   \frac{\alpha\sigma_G}{\sqrt{D_{i,0}}}:=h_1,
	\end{align}
	where step $(a)$ is by $\tilde{\mathbf{w}}_{0}={\mathbf{w}}_{0}={\mathbf{w}}$, and step $(b)$ is by \cite[Eq. (34)]{fallah2020personalized}.
	
	Next, as the induction hypothesis, we assume that a bound $h_{l-1}$ holds for $\E\left[\|\tilde{\mathbf{w}}_{l-1}-{\mathbf{w}}_{l-1}\|\right]$, for some $l>1$, and we recursively derive the bound $h_{l}$ on $\E\left[\|\tilde{\mathbf{w}}_l-{\mathbf{w}}_l\|\right]$ as a function of $h_{l-1}$.
	\begin{align}\label{h_l_rec}
		\E\left[\|\tilde{\mathbf{w}}_l-{\mathbf{w}}_l\|\right]&=\E\left[{\big\|}\left(\tilde{\mathbf{w}}_{l-1} -{\mathbf{w}}_{l-1} \right) + \alpha \left(\nabla f_i({\mathbf{w}}_{l-1})- \tnabla f_i(\tilde{\mathbf{w}}_{l-1},\D_{i,l-1})\right){\big\|}\right]\nonumber\\
		&\stackrel{(a)}{\leq} h_{l-1} + \alpha \E\left[{\big\|} \tnabla f_i(\tilde{\mathbf{w}}_{l-1},\D_{i,l-1})-\nabla f_i({\mathbf{w}}_{l-1}){\big\|}\right]\nonumber\\
		&\stackrel{(b)}{\leq}   (1+\alpha L)h_{l-1}+\frac{\alpha\sigma_G}{\sqrt{D_{i,l-1}}}:=h_l,
	\end{align}
	where step $(a)$ is by the convexity of norm together with the induction hypothesis, and step $(b)$ is by \begin{align}\label{h_l_rec_stepB}
		&\E\left[{\big\|} \tnabla f_i(\tilde{\mathbf{w}}_{l-1},\D_{i,l-1})-\nabla f_i({\mathbf{w}}_{l-1}){\big\|}\right]\nonumber\\
		&=\E\left[{\big\|} \left(\tnabla f_i(\tilde{\mathbf{w}}_{l-1},\D_{i,l-1})-\nabla f_i(\tilde{\mathbf{w}}_{l-1})\right)+\left(\nabla f_i(\tilde{\mathbf{w}}_{l-1})-\nabla f_i({\mathbf{w}}_{l-1})\right){\big\|}\right]\nonumber\\
		&\leq\E\left[\E\left[{\big\|} \tnabla f_i(\tilde{\mathbf{w}}_{l-1},\D_{i,l-1})-\nabla f_i(\tilde{\mathbf{w}}_{l-1}){\big\|}~\!{\big|}\tilde{\mathbf{w}}_{l-1}\right]\right]+\E\left[{\big\|}\nabla f_i(\tilde{\mathbf{w}}_{l-1})-\nabla f_i({\mathbf{w}}_{l-1}){\big\|}\right]\nonumber\\
		&\leq \frac{\sigma_G}{\sqrt{D_{i,l-1}}}+ L\E\left[{\|}\tilde{\mathbf{w}}_{l-1}-{\mathbf{w}}_{l-1}{\|}\right],
	\end{align}
	where the last inequality is by \cite[Eq. (34)]{fallah2020personalized} together with Assumption \ref{asm_grad}.
	
	Now, by solving the recursive relation in \eqref{h_l_rec} with the initial value in \eqref{h_1}, we have
	\begin{align}\label{h_l}
		h_l=\alpha \sigma_G\sum_{j=0}^{l-1} \frac{(1+\alpha L)^j}{\sqrt{D_{i,l-1-j}}}.
	\end{align}
	Note that, by the Jensen's inequality together with the convexity of norm, we have
	\begin{align}
		\|\E\left[\tilde{\mathbf{w}}_l-{\mathbf{w}}_l\right]\|\leq \E\left[\|\tilde{\mathbf{w}}_l-{\mathbf{w}}_l\|\right] \leq h_l.
	\end{align}

	To get the second result, for the base case of $l=1$ we have
	\begin{align}\label{h'1}
		\E\left[\|\tilde{\mathbf{w}}_1-{\mathbf{w}}_1\|^2\right]&=\alpha^2 \E\left[{\big\|} \tnabla f_i({\mathbf{w}},\D_{i,0})-\nabla f_i({\mathbf{w}}){\big\|}^2\right]\nonumber\\
		&{\leq}   \frac{\alpha^2\sigma^2_G}{{D_{i,0}}}:=h'_1,
	\end{align}
	where last inequality is by \cite[Eq. (34)]{fallah2020personalized}.
	
	Next, assuming a bound $h'_{l-1}$ on $\E\left[\|\tilde{\mathbf{w}}_{l-1}-{\mathbf{w}}_{l-1}\|^2\right]$, for some $l>1$, we recursively derive the bound $h'_{l}$ on $\E\left[\|\tilde{\mathbf{w}}_l-{\mathbf{w}}_l\|^2\right]$.
	\begin{align}\label{h'_l_rec}
		\E\left[\|\tilde{\mathbf{w}}_l-{\mathbf{w}}_l\|^2\right]&=\E\left[{\big\|}\left(\tilde{\mathbf{w}}_{l-1} -{\mathbf{w}}_{l-1} \right) + \alpha \left(\nabla f_i({\mathbf{w}}_{l-1})- \tnabla f_i(\tilde{\mathbf{w}}_{l-1},\D_{i,l-1})\right){\big\|^2}\right]\nonumber\\
		&\stackrel{(a)}{\leq}2\E\left[{\|}\tilde{\mathbf{w}}_{l-1} -{\mathbf{w}}_{l-1}{\|^2}\right]+2\alpha^2\E\left[{\big\|}\nabla f_i({\mathbf{w}}_{l-1})- \tnabla f_i(\tilde{\mathbf{w}}_{l-1},\D_{i,l-1}){\big\|^2}\right]\nonumber\\
		&\stackrel{(b)}{\leq} 2h'_{l-1} + 2\alpha^2 \E\left[{\big\|} \tnabla f_i(\tilde{\mathbf{w}}_{l-1},\D_{i,l-1})-\nabla f_i({\mathbf{w}}_{l-1}){\big\|}^2\right]\nonumber\\
		&\stackrel{(c)}{\leq}   2(1+\alpha^2 L^2)h'_{l-1}+\frac{2\alpha^2\sigma^2_G}{{D_{i,l-1}}}:=h'_l,
	\end{align}
	where step $(a)$ is by the norm triangle inequality and then inequality $(a+b)^2\leq 2a^2 + 2b^2$, step $(b)$ is by the induction hypothesis, and step $(c)$ is by
	\begin{align}\label{h'_l_rec_stepB}
		&\E\left[{\big\|} \tnabla f_i(\tilde{\mathbf{w}}_{l-1},\D_{i,l-1})-\nabla f_i({\mathbf{w}}_{l-1}){\big\|}^2\right] \nonumber\\
		& = \E \left [\E \left [\left \| \left(\tnabla f_i(\tilde{\mathbf{w}}_{l-1},\D_{i,l-1}) - \nabla f_i(\tilde{\mathbf{w}}_{l-1})\right) + \left(\nabla f_i(\tilde{\mathbf{w}}_{l-1})-\nabla f_i(\mathbf{w}_{l-1})\right) \right \|^2 {\big |} ~\! \tilde{\mathbf{w}}_{l-1} \right ]  \right ]
		\nonumber\\
		& \stackrel{(a)}{=} \E \left [ \E \left [ \left \| \tnabla f_i(\tilde{\mathbf{w}}_{l-1},\D_{i,l-1}) - \nabla f_i(\tilde{\mathbf{w}}_{l-1})\right \|^2 {\big |} ~\! \tilde{\mathbf{w}}_{l-1} \right ] \right ] +
		\E \left [\left \| \nabla f_i(\tilde{\mathbf{w}}_{l-1})-\nabla f_i(\mathbf{w}_{l-1})\right \|^2  \right ] 
		\nonumber\\
		& \stackrel{(b)}{\leq} \frac{\sigma^2_G}{D_{i,l-1}} +
		L^2 \E \left [\left \| \tilde{\mathbf{w}}_{l-1}-\mathbf{w}_{l-1}\right \|^2  \right ],
	\end{align}
	where step $(a)$ follows the fact that conditioned on $\tilde{\mathbf{w}}_{l-1}$ the term $\tnabla f_i(\tilde{\mathbf{w}}_{l-1},\D_{i,l-1}) - \nabla f_i(\tilde{\mathbf{w}}_{l-1})$ is zero mean, and the term $\nabla f_i(\tilde{\mathbf{w}}_{l-1})-\nabla f_i(\mathbf{w}_{l-1})$ is deterministic.\footnote{It is easy to verify that if $\mathbf{X}$ is zero mean and $\mathbf{Y}$ is deterministic, then $\E\left[\left\|\mathbf{X}+\mathbf{Y}\right\|^2\right]=\E\left[\left\|\mathbf{X}\right\|^2\right]+\left\|\mathbf{Y}\right\|^2$.} Also, step $(b)$ is obtained by \cite[Eq. (34)]{fallah2020personalized} together with Assumption \ref{asm_grad}.

	Note that the recursive expression for $h'_l$ in  \eqref{h'_l_rec} only holds for $l>1$. That is, one cannot derive $h'_1$ from this expression by setting $l=1$ and $h'_0:=0$. By solving the recursive relation in \eqref{h'_l_rec} with the initial value in \eqref{h'1}, we have
	\begin{align}\label{h'_l}
		h'_l=\frac{\alpha^2 \sigma^2_G(2+2\alpha^2 L^2)^{l-1}}{{D_{i,0}}} + 2\alpha^2 \sigma^2_G\sum_{j=1}^{l-1} \frac{(2+2\alpha^2 L^2)^{l-j-1}}{{D_{i,j}}}.
	\end{align}
	
	To derive the third result of the lemma, we apply induction together with Lemma \ref{lemma_intermediate} to upper bound $\E\left[\|\tilde{\mathbf{w}}_l-{\mathbf{w}}_l\|^4\right]$ by $h''_l$, for $l=1,\cdots,\nu$. For the base case of $l=1$ we have
	\begin{align}\label{h''1}
		\E\left[\|\tilde{\mathbf{w}}_1-{\mathbf{w}}_1\|^4\right]&=\alpha^4 \E\left[{\big\|} \tnabla f_i({\mathbf{w}},\D_{i,0})-\nabla f_i({\mathbf{w}}){\big\|}^4\right]\nonumber\\
		&{\leq}   \frac{\kappa_G+3(D_{i,0}-1)\sigma^4_G}{D_{i,0}^3}\alpha^4:=h''_1,
	\end{align}
	where the last inequality is by Lemma \ref{lemma_intermediate}.
	
	Next, assuming a bound $h''_{l-1}$ on $\E\left[\|\tilde{\mathbf{w}}_{l-1}-{\mathbf{w}}_{l-1}\|^4\right]$, for some $l>1$, we recursively derive the bound $h''_{l}$ on $\E\left[\|\tilde{\mathbf{w}}_l-{\mathbf{w}}_l\|^4\right]$.
	\begin{align}\label{h''_l_rec}
		\E\left[\|\tilde{\mathbf{w}}_l-{\mathbf{w}}_l\|^4\right]&=\E\left[{\big\|}\left(\tilde{\mathbf{w}}_{l-1} -{\mathbf{w}}_{l-1} \right) + \alpha \left(\nabla f_i({\mathbf{w}}_{l-1})- \tnabla f_i(\tilde{\mathbf{w}}_{l-1},\D_{i,l-1})\right){\big\|^4}\right]\nonumber\\
		&\stackrel{(a)}{\leq}\E\left[8\left\|\tilde{\mathbf{w}}_{l-1} -{\mathbf{w}}_{l-1}\right\|^4+8\alpha^4{\big\|}\nabla f_i({\mathbf{w}}_{l-1})- \tnabla f_i(\tilde{\mathbf{w}}_{l-1},\D_{i,l-1}){\big\|^4}\right]
		\nonumber\\
		&\stackrel{(b)}{\leq} 8h''_{l-1} + 8\alpha^4 \E\left[{\big\|} \tnabla f_i(\tilde{\mathbf{w}}_{l-1},\D_{i,l-1})-\nabla f_i({\mathbf{w}}_{l-1}){\big\|}^4\right]\nonumber\\
		&\stackrel{(c)}{\leq}   8(1+8\alpha^4 L^4)h''_{l-1}+64\alpha^4\frac{\kappa_G+3(D_{i,l-1}-1)\sigma^4_G}{D_{i,l-1}^3}:=h''_l,
	\end{align}
	where step $(a)$ is by the norm triangle inequality and then applying the inequality $(a+b)^2\leq 2a^2 + 2b^2$ twice to get $\|\mathbf{x}+\mathbf{y}\|^4\leq 8\|\mathbf{x}\|^4 + 8\|\mathbf{y}\|^4$, step $(b)$ is by the induction hypothesis, and step $(c)$ is by
	\begin{align}\label{h''_l_rec_stepB}
		&\E\left[{\big\|} \tnabla f_i(\tilde{\mathbf{w}}_{l-1},\D_{i,l-1})-\nabla f_i({\mathbf{w}}_{l-1}){\big\|}^4\right]\nonumber\\
		& = \E \left [\left \| \left(\tnabla f_i(\tilde{\mathbf{w}}_{l-1},\D_{i,l-1}) - \nabla f_i(\tilde{\mathbf{w}}_{l-1})\right) + \left(\nabla f_i(\tilde{\mathbf{w}}_{l-1})-\nabla f_i(\mathbf{w}_{l-1})\right) \right \|^4\right ] \nonumber\\
		& \leq 8\E \left [ \E \left [ \left \| \tnabla f_i(\tilde{\mathbf{w}}_{l-1},\D_{i,l-1}) - \nabla f_i(\tilde{\mathbf{w}}_{l-1})\right \|^4 {\big |} ~\! \tilde{\mathbf{w}}_{l-1} \right ] \right ] + 8
		\E \left [\left \| \nabla f_i(\tilde{\mathbf{w}}_{l-1})-\nabla f_i(\mathbf{w}_{l-1})\right \|^4  \right ] 
		\nonumber\\
		& \leq 8\frac{\kappa_G+3(D_{i,l-1}-1)\sigma^4_G}{D_{i,l-1}^3} +
		8L^4 \E \left [\left \| \tilde{\mathbf{w}}_{l-1}-\mathbf{w}_{l-1}\right \|^4  \right ],
	\end{align}
	where the last inequality is by Lemma \ref{lemma_intermediate} and Assumption \ref{asm_grad}.

	Finally, by solving the recursive relation in \eqref{h''_l_rec} with the initial value in \eqref{h''1}, we have
	\begin{align}\label{h''_l}
		h''_l=\frac{\alpha^4\left(\kappa_G+3(D_{i,0}-1)\sigma^4_G\right) (8+64\alpha^4 L^4)^{l-1}}{{D^3_{i,0}}} 
		+ 64\alpha^4 \sum_{j=1}^{l-1} \frac{\left(\kappa_G+3(D_{i,j}-1)\sigma^4_G\right)(8+64\alpha^4 L^4)^{l-j-1}}{{D^3_{i,j}}}.
	\end{align}
	
	\endproof

	\section{Proof of Lemma \ref{lemma_err_moment}}\label{App_B}

	Let us define the matrix  $\mathbf{\tilde{B}}_{\tilde{\mathbf{w}}}^{(\nu)}$ as
	\begin{align}\label{B_tilde}
		\mathbf{\tilde{B}}_{\tilde{\mathbf{w}}}^{(\nu)}:=& \left(\mathbf{I} - \alpha \tnabla^2 f_i({\mathbf{w}},\D'_{i,0})\right) \left(\mathbf{I} - \alpha \tnabla^2 f_i(\tilde{\mathbf{w}}_1,\D'_{i,1})\right) \cdots \left(\mathbf{I} - \alpha \tnabla^2 f_i(\tilde{\mathbf{w}}_{\nu-1},\D'_{i,\nu-1})\right),
	\end{align}
	which can be recursively expressed as
	\begin{align}\label{B_tilde_rec}
		\mathbf{\tilde{B}}^{(\nu)}_{\tilde{\mathbf{w}}} &=  \mathbf{\tilde{B}}^{(\nu-1)}_{\tilde{\mathbf{w}}} \left(\mathbf{I} - \alpha \tnabla^2 f_i(\tilde{\mathbf{w}}_{\nu-1},\D'_{i,\nu-1})\right),\nonumber\\
		\mathbf{\tilde{B}}^{(1)}_{\tilde{\mathbf{w}}} &=
		\mathbf{I} - \alpha \tnabla^2 f_i(\mathbf{w},\D'_{i,0}). 
	\end{align}
	We can then rewrite the expression for the stochastic gradient $\tnabla F_i(\mathbf{w})$ in \eqref{grad_F_stoch} 
	as
	\begin{align}\label{F_tilde_rewrite1}
		\tnabla F_i(\mathbf{w}) &= \mathbf{\tilde{B}}^{(\nu)}_{\tilde{\mathbf{w}}} \tnabla f_i(\tilde{\mathbf{w}}_{\nu-1} - \alpha \tnabla f_i(\tilde{\mathbf{w}}_{\nu-1},\D_{i,\nu-1}),\D_{i,\nu}),
	\end{align}
	which can be further rewritten as
	\begin{align}\label{F_tilde_rewrite2}
		\tnabla F_i(\mathbf{w}) &=\left(\mathbf{{B}}^{(\nu)}_{\mathbf{w}}+\mathbf{E}^{(\nu)}_1\right) \left(\nabla f_i(\mathbf{w}_{\nu-1} - \alpha \nabla f_i(\mathbf{w}_{\nu-1}))+\mathbf{e}_2\right),
	\end{align}
	where $\mathbf{{B}}^{(\nu)}_{\mathbf{w}}$ is defined in \eqref{B_w}, and $\mathbf{E}^{(\nu)}_1$ and $\mathbf{e}_2$ are defined as
	\begin{align}
		\mathbf{E}^{(\nu)}_1&=\mathbf{\tilde{B}}^{(\nu)}_{\tilde{\mathbf{w}}}-\mathbf{{B}}^{(\nu)}_{\mathbf{w}}\label{E1}\\
		\mathbf{e}_2&=\tnabla f_i(\tilde{\mathbf{w}}_{\nu-1} - \alpha \tnabla f_i(\tilde{\mathbf{w}}_{\nu-1},\D_{i,\nu-1}),\D_{i,\nu}) - \nabla f_i(\mathbf{w}_{\nu-1} - \alpha \nabla f_i(\mathbf{w}_{\nu-1}))\nonumber\\
		&=\tnabla f_i(\tilde{\mathbf{w}}_{\nu},\D_{i,\nu}) - \nabla f_i(\mathbf{w}_{\nu}).
		\label{e2}
	\end{align}
	
	We first bound the moments of $\mathbf{E}^{(\nu)}_1$ and $\mathbf{e}_2$. Throughout the analysis in this appendix we assume that $\nu>1$. For $\left \| \E \left [ \mathbf{e}_2 \right ] \right \|$ we have
	\begin{align}\label{moment1_e2}
		\left \| \E \left [ \mathbf{e}_2 \right ] \right \| & \leq  \E \left [ \left \|\tnabla f_i(\tilde{\mathbf{w}}_{\nu},\D_{i,\nu}) - \nabla f_i(\mathbf{w}_{\nu}) \right \| \right ] \nonumber\\
		& \stackrel{(a)}{\leq} \frac{\sigma_G}{\sqrt{D_{i,\nu}}}+ L\E\left[{\|}\tilde{\mathbf{w}}_{\nu}-{\mathbf{w}}_{\nu}{\|}\right] \nonumber\\
		& \stackrel{(b)}{\leq} \frac{\sigma_G}{\sqrt{D_{i,\nu}}}+Lh_{\nu},
	\end{align}
	where step $(a)$ is by \eqref{h_l_rec_stepB}, and step $(b)$ is by Lemma \ref{lemma_w_tilde_difference}.
	Moreover, for the second moment of $\mathbf{e}_2$ we have 
	\begin{align}\label{moment2_e2}
		\E \left [\left \| \mathbf{e}_2 \right \|^2\right ]  &  \stackrel{(a)}{\leq} \frac{\sigma^2_G}{D_{i,\nu}} +
		L^2 \E \left [\left \| \tilde{\mathbf{w}}_{\nu}-\mathbf{w}_{\nu}\right \|^2  \right ]\nonumber\\
		&  \stackrel{(b)}{\leq} \frac{\sigma^2_G}{D_{i,\nu}} +
		L^2 h'_{\nu},
	\end{align}
	where step $(a)$ is by \eqref{h'_l_rec_stepB}, and step $(b)$ follows Lemma \ref{lemma_w_tilde_difference}. Additionally, using \eqref{h''_l_rec_stepB} and Lemma \ref{lemma_w_tilde_difference}, for the fourth moment of $\mathbf{e}_2$ we have 
	\begin{align}\label{moment4_e2}
		\E \left [\left \| \mathbf{e}_2 \right \|^4\right ]  &  \leq 8\left(\frac{\kappa_G+3(D_{i,\nu}-1)\sigma^4_G}{D_{i,\nu}^3}+L^4 h''_{\nu}\right).
	\end{align}
	
	On the other hand, for $\E \left [ \mathbf{E}^{(\nu)}_1 \right ]$ we have
	\begin{align}\label{E1_mean}
		\E \left [ \mathbf{E}^{(\nu)}_1 \right ] &\stackrel{(a)}{=}\E\left[\mathbf{\tilde{B}}^{(\nu)}_{\tilde{\mathbf{w}}}\right]-\mathbf{{B}}^{(\nu)}_{\mathbf{w}}\nonumber\\
		& \stackrel{(b)}{=} \E \left [ \E \left[\prod_{l=0}^{\nu-1}\left(\mathbf{I} - \alpha \tnabla^2 f_i(\tilde{\mathbf{w}}_l,\D'_{i,l})\right){\big|}\left\{\tilde{\mathbf{w}}_j\right\}_{j=1}^{\nu-1}\right]\right] -\mathbf{{B}}^{(\nu)}_{\mathbf{w}}\nonumber\\
		& \stackrel{(c)}{=} \E \left[\prod_{l=0}^{\nu-1}\E\left[\mathbf{I} - \alpha \tnabla^2 f_i(\tilde{\mathbf{w}}_l,\D'_{i,l}){\big|}\tilde{\mathbf{w}}_l\right]\right] -\mathbf{{B}}^{(\nu)}_{\mathbf{w}}\nonumber\\
		& \stackrel{(d)}{=} \E \left[\prod_{l=0}^{\nu-1}\left(\mathbf{I} - \alpha \nabla^2 f_i(\tilde{\mathbf{w}}_l)\right)\right] -\mathbf{{B}}^{(\nu)}_{\mathbf{w}},
	\end{align}
	where step $(a)$ follows the fact that $\mathbf{{B}}^{(\nu)}_{\mathbf{w}}$ is deterministic, step $(b)$ is by the law of total expectation (tower rule) and \eqref{B_tilde}, 
	step $(c)$ follows the independence of $\D'_{i,l}$'s, and step $(d)$ is by applying \eqref{supervised_ML} to show that $\E\left[\tnabla^2 f_i(\tilde{\mathbf{w}}_l,\D'_{i,l})|\tilde{\mathbf{w}}_l\right]=\nabla^2 f_i(\tilde{\mathbf{w}}_l)$.

	To find the norm of the expression in \eqref{E1_mean}, let us define the matrix  $\mathbf{{B}}_{\tilde{\mathbf{w}}}^{(\nu)}$ as
	\begin{align}\label{B_wt}
		\mathbf{{B}}_{\tilde{\mathbf{w}}}^{(\nu)}:=& \left(\mathbf{I} - \alpha \nabla^2 f_i({\mathbf{w}})\right) \left(\mathbf{I} - \alpha \nabla^2 f_i(\tilde{\mathbf{w}}_1)\right) \cdots \left(\mathbf{I} - \alpha \nabla^2 f_i(\tilde{\mathbf{w}}_{\nu-1})\right),
	\end{align}
	which can be recursively expressed as
	\begin{align}\label{B_wt_rec}
		\mathbf{{B}}^{(\nu)}_{\tilde{\mathbf{w}}} &=  \mathbf{{B}}^{(\nu-1)}_{\tilde{\mathbf{w}}} \left(\mathbf{I} - \alpha \nabla^2 f_i(\tilde{\mathbf{w}}_{\nu-1})\right),\nonumber\\
		\mathbf{{B}}^{(1)}_{\tilde{\mathbf{w}}} &=
		\mathbf{I} - \alpha \nabla^2 f_i(\mathbf{w}). 
	\end{align}
	Then, using \eqref{E1_mean}, we have (recall that $\tilde{\mathbf{w}}_0:={\mathbf{w}}_0:={\mathbf{w}}$)
	\begin{align}\label{E1_mean_part2}
		\E \left [ \mathbf{E}^{(\nu)}_1 \right ] &= \E \left[\mathbf{{B}}^{(\nu)}_{\tilde{\mathbf{w}}}-\mathbf{{B}}^{(\nu)}_{\mathbf{w}}\right],
	\end{align}
	Now, we apply induction to find  the upper bound $t_{\nu}$ of $\E [ \| \mathbf{{B}}^{(\nu)}_{\tilde{\mathbf{w}}}-\mathbf{{B}}^{(\nu)}_{\mathbf{w}} \|]$. Note that, for the base case of $\nu=1$, we have
	\begin{align}\label{E1_mean_part2_base}
		\E\left[\left\| \mathbf{{B}}^{(1)}_{\tilde{\mathbf{w}}}-\mathbf{{B}}^{(1)}_{\mathbf{w}} \right\|\right] = \alpha \E \left[\left\|\nabla^2 f_i(\tilde{\mathbf{w}}_{0})-\nabla^2 f_i({\mathbf{w}}_{0})\right\|\right]=0,
	\end{align}
	which implies $t_1=0$. Next, assuming $\E [ \| \mathbf{{B}}^{(\nu-1)}_{\tilde{\mathbf{w}}}-\mathbf{{B}}^{(\nu-1)}_{\mathbf{w}} \|] \leq t_{\nu-1}$ (i.e., the induction hypothesis), we find the upper bound $t_{\nu}$ on $\E [ \| \mathbf{{B}}^{(\nu)}_{\tilde{\mathbf{w}}}-\mathbf{{B}}^{(\nu)}_{\mathbf{w}} \|]$.
	\begin{align}\label{E1_mean_ind}
		\E \left[\left\|\mathbf{{B}}^{(\nu)}_{\tilde{\mathbf{w}}}-\mathbf{{B}}^{(\nu)}_{\mathbf{w}}\right\|\right]
		&  \stackrel{(a)}{=} \E \left[\left\|\mathbf{{B}}^{(\nu-1)}_{\tilde{\mathbf{w}}} \left(\mathbf{I} - \alpha \nabla^2 f_i(\tilde{\mathbf{w}}_{\nu-1})\right)-\mathbf{{B}}^{(\nu-1)}_{{\mathbf{w}}} \left(\mathbf{I} - \alpha \nabla^2 f_i({\mathbf{w}}_{\nu-1})\right)\right\|\right]\nonumber\\
		& \stackrel{(b)}{\leq} 
		\E \left[\left\|\mathbf{{B}}^{(\nu-1)}_{\tilde{\mathbf{w}}} -\mathbf{{B}}^{(\nu-1)}_{{\mathbf{w}}}\right\|\right]
		+\alpha \E \left[\left\|\mathbf{{B}}^{(\nu-1)}_{\tilde{\mathbf{w}}} \nabla^2 f_i(\tilde{\mathbf{w}}_{\nu-1})-\mathbf{{B}}^{(\nu-1)}_{{\mathbf{w}}} \nabla^2 f_i({\mathbf{w}}_{\nu-1})\right\|\right]\nonumber\\
		& \stackrel{(c)}{\leq} (1+\alpha L) t_{\nu-1}
		+ \alpha\rho(1+\alpha L)^{\nu-1} h_{\nu-1}:=t_{\nu},
	\end{align}
	where step $(a)$ is by \eqref{B_w_rec} and \eqref{B_wt_rec}, step $(b)$ is by the convexity of norm, and step $(c)$ is by the induction hypothesis and
	\begin{align}\label{E1_mean_ind_part2}
		&\E \left[\left\|\mathbf{{B}}^{(\nu-1)}_{\tilde{\mathbf{w}}} \nabla^2 f_i(\tilde{\mathbf{w}}_{\nu-1})-\mathbf{{B}}^{(\nu-1)}_{{\mathbf{w}}} \nabla^2 f_i({\mathbf{w}}_{\nu-1})\right\|\right]\nonumber\\
		& = \E \left[\left\|\left(\mathbf{{B}}^{(\nu-1)}_{\tilde{\mathbf{w}}} \nabla^2 f_i(\tilde{\mathbf{w}}_{\nu-1})-\mathbf{{B}}^{(\nu-1)}_{{\mathbf{w}}} \nabla^2 f_i({\tilde{\mathbf{w}}}_{\nu-1})\right)+\big(\mathbf{{B}}^{(\nu-1)}_{{\mathbf{w}}} \nabla^2 f_i({\tilde{\mathbf{w}}}_{\nu-1})-\mathbf{{B}}^{(\nu-1)}_{{\mathbf{w}}} \nabla^2 f_i({\mathbf{w}}_{\nu-1})\big)\right\|\right]\nonumber\\& \leq
		\E \left[\left\|\left(\mathbf{{B}}^{(\nu-1)}_{\tilde{\mathbf{w}}} -\mathbf{{B}}^{(\nu-1)}_{{\mathbf{w}}}\right)\nabla^2 f_i(\tilde{\mathbf{w}}_{\nu-1})\right\|\right]
		+ \E \left[\left\|\mathbf{{B}}^{(\nu-1)}_{{\mathbf{w}}} \big(\nabla^2 f_i({\tilde{\mathbf{w}}}_{\nu-1})-\nabla^2 f_i({\mathbf{w}}_{\nu-1})\big)\right\|\right]
		\nonumber\\
		& \stackrel{(a)}{\leq}
		L \E \left[\left\|\mathbf{{B}}^{(\nu-1)}_{\tilde{\mathbf{w}}} -\mathbf{{B}}^{(\nu-1)}_{{\mathbf{w}}}\right\|\right]
		+ \rho(1+\alpha L)^{\nu-1} \E \left[\left\|{\tilde{\mathbf{w}}}_{\nu-1}-{\mathbf{w}}_{\nu-1}\right\|\right]
		\nonumber\\
		& \stackrel{(b)}{\leq}
		L t_{\nu-1}
		+ \rho(1+\alpha L)^{\nu-1} h_{\nu-1},
	\end{align}
	where step $(a)$ is by $\|\nabla^2 f_i(\mathbf{w})\|\leq L$ for any arbitrary model $\mathbf{w}$, Assumption \ref{asm_Hesian_Lip}, and applying \eqref{eq51_norm1_1} on top of $\mathbf{{B}}^{(\nu-1)}_{{\mathbf{w}}}$ being deterministic. Moreover, step $(b)$ follows the induction hypothesis and Lemma \ref{lemma_w_tilde_difference}.
	
	Finally, by solving the recursive expression in \eqref{E1_mean_ind} with the initial value $d_1=0$, we have
	\begin{align}\label{E1_mean_d1}
		\E \left[\left\|\mathbf{{B}}^{(\nu)}_{\tilde{\mathbf{w}}}-\mathbf{{B}}^{(\nu)}_{\mathbf{w}}\right\|\right]\leq t_{\nu}:=\alpha\rho(1+\alpha L)^{\nu-1}\sum_{j=1}^{\nu-1}h_j, ~~~~ \nu > 1.
	\end{align}
	This together with \eqref{E1_mean_part2} and the Jensen's inequality  gives
	\begin{align}\label{E1_mean_new}
		\left\|\E \left [ \mathbf{E}^{(\nu)}_1 \right ]\right\| & {\leq} \E \left[\left\|\mathbf{{B}}^{(\nu)}_{\tilde{\mathbf{w}}}-\mathbf{{B}}^{(\nu)}_{\mathbf{w}}\right\|\right]\nonumber\\
		& \leq t_{\nu}:=\alpha\rho(1+\alpha L)^{\nu-1}\sum_{j=1}^{\nu-1}h_j, ~~~~ \nu > 1.
	\end{align}

	Next, we apply a proof by induction to upper bound $\E \left [ \|\mathbf{E}^{(\nu)}_1\|^2 \right ]$ by some quantity $d_{\nu}$. Note that for the base case of $\nu=1$ we have
	\begin{align}\label{norm_E1_1}
		\E \left [ \|\mathbf{E}^{(1)}_1\|^2 \right ] & = \E \left [ \|\mathbf{\tilde{B}}^{(1)}_{\tilde{\mathbf{w}}}-\mathbf{{B}}^{(1)}_{\mathbf{w}}\|^2 \right ]\nonumber\\
		& = \alpha^2\E \left [ \| \nabla^2 f_i(\mathbf{w})-\tnabla^2 f_i(\mathbf{w},\D'_{i,0}) \|^2 \right ] \nonumber\\
		& \leq  \alpha^2 \frac{\sigma_H^2}{D'_{i,0}}:=d_1,
	\end{align}
	where the last inequality is by \cite[Eq. (30b)]{fallah2020personalized}. Now, as the induction hypothesis, we assume that step $\nu-1$ holds with some bound $d_{\nu -1}$, i.e., $\E \left [ \|\mathbf{E}^{(\nu-1)}_1\|^2 \right ]\leq d_{\nu -1}$,
	and we show that step $\nu$ holds with some factor $d_{\nu}$. Using \eqref{B_w_rec} and \eqref{B_tilde_rec} we have
	\begin{align}\label{norm_E1_nu}
		\E \left [ \|\mathbf{E}^{(\nu)}_1\|^2 \right ] & = \E \left [ \|\mathbf{\tilde{B}}^{(\nu)}_{\tilde{\mathbf{w}}}-\mathbf{{B}}^{(\nu)}_{\mathbf{w}}\|^2 \right ]\nonumber\\
		& = \E \left [ \left\| \mathbf{\tilde{B}}^{(\nu-1)}_{\tilde{\mathbf{w}}} \left(\mathbf{I} - \alpha \tnabla^2 f_i(\tilde{\mathbf{w}}_{\nu-1},\D'_{i,\nu-1})\right) - \mathbf{B}^{(\nu-1)}_{\mathbf{w}} \left(\mathbf{I} - \alpha \nabla^2 f_i(\mathbf{w}_{\nu-1})\right) \right\|^2 \right ] \nonumber\\
		& = \E \Big [ \big\| \mathbf{\tilde{B}}^{(\nu-1)}_{\tilde{\mathbf{w}}} \left(\mathbf{I} - \alpha \tnabla^2 f_i(\tilde{\mathbf{w}}_{\nu-1},\D'_{i,\nu-1})\right) -\mathbf{B}^{(\nu-1)}_{\mathbf{w}} \left(\mathbf{I} - \alpha \tnabla^2 f_i(\tilde{\mathbf{w}}_{\nu-1},\D'_{i,\nu-1})\right)
		\nonumber\\
		&~~~~~~~~~~+ \mathbf{B}^{(\nu-1)}_{\mathbf{w}} \left(\mathbf{I} - \alpha \tnabla^2 f_i(\tilde{\mathbf{w}}_{\nu-1},\D'_{i,\nu-1})\right) -\mathbf{B}^{(\nu-1)}_{\mathbf{w}} \left(\mathbf{I} 
		- \alpha \nabla^2 f_i(\mathbf{w}_{\nu-1})\right) \big\|^2 \Big ] \nonumber\\
		& = \E \Big [ \big\| \left(\mathbf{\tilde{B}}^{(\nu-1)}_{\tilde{\mathbf{w}}}-\mathbf{B}^{(\nu-1)}_{\mathbf{w}}\right) \left(\mathbf{I} - \alpha \tnabla^2 f_i(\tilde{\mathbf{w}}_{\nu-1},\D'_{i,\nu-1})\right)
		\nonumber\\
		&~~~~~~~~~~+ \alpha\mathbf{B}^{(\nu-1)}_{\mathbf{w}} \left(\nabla^2 f_i(\mathbf{w}_{\nu-1}) -  \tnabla^2 f_i(\tilde{\mathbf{w}}_{\nu-1},\D'_{i,\nu-1})\right) \big\|^2 \Big ] \nonumber\\
		& \leq \E \Big [ \Big(\big\| \left(\mathbf{\tilde{B}}^{(\nu-1)}_{\tilde{\mathbf{w}}}-\mathbf{B}^{(\nu-1)}_{\mathbf{w}}\right) \left(\mathbf{I} - \alpha \tnabla^2 f_i(\tilde{\mathbf{w}}_{\nu-1},\D'_{i,\nu-1})\right)\big\|
		\nonumber\\
		&~~~~~~~~~~+ \alpha\big\|\mathbf{B}^{(\nu-1)}_{\mathbf{w}} \left(\nabla^2 f_i(\mathbf{w}_{\nu-1}) -  \tnabla^2 f_i(\tilde{\mathbf{w}}_{\nu-1},\D'_{i,\nu-1})\right) \big\| \Big)^{\!\!2}~\Big ]\nonumber\\
		& \leq 2\E \Big [\big\| \left(\mathbf{\tilde{B}}^{(\nu-1)}_{\tilde{\mathbf{w}}}-\mathbf{B}^{(\nu-1)}_{\mathbf{w}}\right) \left(\mathbf{I} - \alpha \tnabla^2 f_i(\tilde{\mathbf{w}}_{\nu-1},\D'_{i,\nu-1})\right)\big\|^2\Big]
		\nonumber\\
		&~~~+ 2\alpha^2\E \Big [\big\|\mathbf{B}^{(\nu-1)}_{\mathbf{w}} \left(\nabla^2 f_i(\mathbf{w}_{\nu-1}) -  \tnabla^2 f_i(\tilde{\mathbf{w}}_{\nu-1},\D'_{i,\nu-1})\right) \big\|^2\Big ],
	\end{align}
	where the last inequality is by applying the Cauchy-Schwarz inequality, $(a + b)^2 \leq 2(a^2 + b^2)$ for $a,b\geq0$.
	
	Next, we bound each of the two terms in \eqref{norm_E1_nu}. For the first term we have
	\begin{align}\label{norm_E1_nu_term1}
		&\E \Big [\big\| \left(\mathbf{\tilde{B}}^{(\nu-1)}_{\tilde{\mathbf{w}}}-\mathbf{B}^{(\nu-1)}_{\mathbf{w}}\right) \left(\mathbf{I} - \alpha \tnabla^2 f_i(\tilde{\mathbf{w}}_{\nu-1},\D'_{i,\nu-1})\right)\big\|^2\Big]\nonumber\\
		&\leq\E \Big [\big\| \mathbf{\tilde{B}}^{(\nu-1)}_{\tilde{\mathbf{w}}}-\mathbf{B}^{(\nu-1)}_{\mathbf{w}}\big\|^2 .\big\|\mathbf{I} - \alpha \tnabla^2 f_i(\tilde{\mathbf{w}}_{\nu-1},\D'_{i,\nu-1})\big\|^2\Big]\nonumber\\
		&\stackrel{(a)}{=}\E \Big[\E \Big[\big\| \mathbf{\tilde{B}}^{(\nu-1)}_{\tilde{\mathbf{w}}}-\mathbf{B}^{(\nu-1)}_{\mathbf{w}}\big\|^2~\!{\big |}~\!\!\left\{\tilde{\mathbf{w}}_j\right\}_{j=1}^{\nu-2}\Big].\E\Big[\big\|\mathbf{I} - \alpha \tnabla^2 f_i(\tilde{\mathbf{w}}_{\nu-1},\D'_{i,\nu-1})\big\|^2~\!{\big |}~\!\!\left\{\tilde{\mathbf{w}}_j\right\}_{j=1}^{\nu-2}\Big]\Big]\nonumber\\
		&\stackrel{(b)}{\leq}
		\left(1+\alpha L + \alpha \frac{\sigma_H}{\sqrt{D'_{i,\nu-1}}}\right)^{\!2\!}\E \Big[\big\| \mathbf{\tilde{B}}^{(\nu-1)}_{\tilde{\mathbf{w}}}-\mathbf{B}^{(\nu-1)}_{\mathbf{w}}\big\|^2\Big]\nonumber\\
		&\leq
		\left(1+\alpha L + \alpha \frac{\sigma_H}{\sqrt{D'_{i,\nu-1}}}\right)^{\!\!2}d_{\nu-1},
	\end{align}
	where step $(a)$ is by the tower rule and then using the independence of $\D'_{i,\nu-1}$ from $\left\{\D'_{i,l}\right\}_{l=0}^{\nu-2}$, and thus the independence of the two norm terms conditioned on $\left\{\tilde{\mathbf{w}}_j\right\}_{j=1}^{\nu-2}$. Moreover, step $(b)$ is by the following results.
	\begin{align}\label{norm_E1_nu_term1_expct}
		&\E\Big[\big\|\mathbf{I} - \alpha \tnabla^2 f_i(\tilde{\mathbf{w}}_{\nu-1},\D'_{i,\nu-1})\big\|^2~\!{\big |}~\!\!\left\{\tilde{\mathbf{w}}_j\right\}_{j=1}^{\nu-2}\Big]\nonumber\\
		& = \E\Big[\big\|\mathbf{I} - \alpha \tnabla^2 f_i(\tilde{\mathbf{w}}_{\nu-1},\D'_{i,\nu-1}) + \alpha \nabla^2 f_i(\tilde{\mathbf{w}}_{\nu-1}) - \alpha \nabla^2 f_i(\tilde{\mathbf{w}}_{\nu-1})\big\|^2~\!{\big |}~\!\!\left\{\tilde{\mathbf{w}}_j\right\}_{j=1}^{\nu-2}\Big]\nonumber\\
		& \leq \E\Big[\left(\big\|\mathbf{I} - \alpha \nabla^2 f_i(\tilde{\mathbf{w}}_{\nu-1})\big\|+\alpha\big\|\tnabla^2 f_i(\tilde{\mathbf{w}}_{\nu-1},\D'_{i,\nu-1}) - \nabla^2 f_i(\tilde{\mathbf{w}}_{\nu-1})\big\|\right)^{\!2}~\!{\big |}~\!\!\left\{\tilde{\mathbf{w}}_j\right\}_{j=1}^{\nu-2}\Big]\nonumber\\
		&\stackrel{(a)}{\leq} (1+\alpha L)^2+\alpha^2\E\Big[\big\|\tnabla^2 f_i(\tilde{\mathbf{w}}_{\nu-1},\D'_{i,\nu-1}) - \nabla^2 f_i(\tilde{\mathbf{w}}_{\nu-1})\big\|^{2}~\!{\big |}~\!\!\left\{\tilde{\mathbf{w}}_j\right\}_{j=1}^{\nu-2}\Big]\nonumber\\
		&\hspace{1in}\!\!+2\alpha(1+\alpha L)\E\Big[\big\|\tnabla^2 f_i(\tilde{\mathbf{w}}_{\nu-1},\D'_{i,\nu-1}) - \nabla^2 f_i(\tilde{\mathbf{w}}_{\nu-1})\big\|~\!{\big |}~\!\!\left\{\tilde{\mathbf{w}}_j\right\}_{j=1}^{\nu-2}\Big]\nonumber\\
		&\stackrel{(b)}{\leq} (1+\alpha L)^2+\alpha^2\frac{\sigma_H^2}{D'_{i,\nu-1}}+2\alpha(1+\alpha L)\frac{\sigma_H}{\sqrt{D'_{i,\nu-1}}}\nonumber\\
		& =  \left(1+\alpha L + \alpha \frac{\sigma_H}{\sqrt{D'_{i,\nu-1}}}\right)^{\!\!2},
	\end{align}
	where step $(a)$ is by the fact $\|\mathbf{I} - \alpha \nabla^2 f_i(\mathbf{w})\|\leq 1+\alpha L$ at any arbitrary model $\mathbf{w}$, and step $(b)$ is by tower rule (to condition on $\tilde{\mathbf{w}}_{\nu-1}$) and then applying \cite[Eq. (30b)]{fallah2020personalized} together with the inequality $\E[X]\leq\sqrt{\E[X^2]}$ for any random variable $X$. 
	
	Additionally, for the second term in \eqref{norm_E1_nu} we have
	\begin{align}\label{norm_E1_nu_term2}
		&\E \Big [\big\|\mathbf{B}^{(\nu-1)}_{\mathbf{w}} \left(\nabla^2 f_i(\mathbf{w}_{\nu-1}) -  \tnabla^2 f_i(\tilde{\mathbf{w}}_{\nu-1},\D'_{i,\nu-1})\right) \big\|^2\Big ]\nonumber\\
		&\stackrel{(a)}{\leq} \big\|\mathbf{B}^{(\nu-1)}_{\mathbf{w}}\big\|^2.\E \Big [\big\|\tnabla^2 f_i(\tilde{\mathbf{w}}_{\nu-1},\D'_{i,\nu-1})- \nabla^2 f_i(\mathbf{w}_{\nu-1})  \big\|^2\Big ]
		\nonumber\\
		&\stackrel{(b)}{\leq} (1+\alpha L)^{2\nu-2}\left(\frac{\sigma^2_H}{D'_{i,\nu-1}} +
		\rho^2 h'_{\nu-1}\right),
	\end{align}
	where step $(a)$ is by the fact that $\mathbf{B}^{(\nu-1)}_{\mathbf{w}}$ is deterministic, and step $(b)$ is by \eqref{eq51_norm1_1} together with 
	the following result.
	\begin{align}
		&\E \Big [\big\|\tnabla^2 f_i(\tilde{\mathbf{w}}_{\nu-1},\D'_{i,\nu-1})- \nabla^2 f_i(\mathbf{w}_{\nu-1})  \big\|^2\Big ]\nonumber\\
		& = \E \left [\E \left [\left \| \left(\tnabla^2 f_i(\tilde{\mathbf{w}}_{\nu-1},\D'_{i,\nu-1})-\nabla^2 f_i(\tilde{\mathbf{w}}_{\nu-1})\right)+ \big(\nabla^2 f_i(\tilde{\mathbf{w}}_{\nu-1}) -\nabla^2 f_i(\mathbf{w}_{\nu-1})\big) \right \|^2 {\big |} ~\! \tilde{\mathbf{w}}_{l-1} \right ]  \right ]
		\nonumber\\
		& \stackrel{(a)}{=} \E \left [ \E \left [ \left \| \tnabla^2 f_i(\tilde{\mathbf{w}}_{\nu-1},\D'_{i,\nu-1})-\nabla^2 f_i(\tilde{\mathbf{w}}_{\nu-1})\right \|^2 {\big |} ~\! \tilde{\mathbf{w}}_{l-1} \right ] \right ] +
		\E \left [\left \| \nabla^2 f_i(\tilde{\mathbf{w}}_{\nu-1}) -\nabla^2 f_i(\mathbf{w}_{\nu-1})\right \|^2  \right ] 
		\nonumber\\
		& \stackrel{(b)}{\leq} \frac{\sigma^2_H}{D'_{i,\nu-1}} +
		\rho^2 \E \left [\left \| \tilde{\mathbf{w}}_{\nu-1}-\mathbf{w}_{\nu-1}\right \|^2  \right ]
		\nonumber\\
		& \stackrel{(c)}{\leq} \frac{\sigma^2_H}{D'_{i,\nu-1}} +
		\rho^2 h'_{\nu-1},
	\end{align}
	where step $(a)$ follows the fact that conditioned on $\tilde{\mathbf{w}}_{l-1}$ the term $\tnabla^2 f_i(\tilde{\mathbf{w}}_{\nu-1},\D'_{i,\nu-1})-\nabla^2 f_i(\tilde{\mathbf{w}}_{\nu-1})$ is zero mean, and the term $\nabla^2 f_i(\tilde{\mathbf{w}}_{\nu-1}) -\nabla^2 f_i(\mathbf{w}_{\nu-1})$ is deterministic. Also, step $(b)$ is obtained by \cite[Eq. (30b)]{fallah2020personalized} together with Assumption \ref{asm_Hesian_Lip}, and step $(c)$ follows Lemma \ref{lemma_w_tilde_difference}.

	By plugging \eqref{norm_E1_nu_term1} and \eqref{norm_E1_nu_term2} in \eqref{norm_E1_nu} we have
	\begin{align}\label{norm_E1_nu_final}
		\E \left [ \|\mathbf{E}^{(\nu)}_1\|^2 \right ] \leq d_{\nu},
	\end{align}
	where $d_{\nu}$ is the solution to the recursive relation
	\begin{align}\label{d_nu_rec} 
		d_{\nu}&=2d_{\nu-1}\left(1+\alpha L + \alpha \frac{\sigma_H}{\sqrt{D'_{i,\nu-1}}}\right)^{\!\!2}+ 
		2\alpha^2 (1+\alpha L)^{2\nu-2}\left(\frac{\sigma^2_H}{D'_{i,\nu-1}} +
		\rho^2 h'_{\nu-1}\right),~~~ \nu>1,
	\end{align}
	with the initial value $d_1=\alpha^2{\sigma_H^2}/{D'_{i,0}}$. 
	One can work to find the closed form for $d_{\nu}$ from \eqref{d_nu_rec}. However, for the sake of brevity, we leave \eqref{d_nu_rec} as is, and derive the results in terms of $d_{\nu}$ that is the solution to \eqref{d_nu_rec}. 

	Similarly, we apply induction to upper bound $\E \left [ \|\mathbf{E}^{(\nu)}_1\|^4 \right ]$ by some quantity $d'_{\nu}$. Note that for the base case of $\nu=1$ we have
	\begin{align}\label{norm4_E1_1}
		\E \left [ \|\mathbf{E}^{(1)}_1\|^4 \right ] & = \E \left [ \|\mathbf{\tilde{B}}^{(1)}_{\tilde{\mathbf{w}}}-\mathbf{{B}}^{(1)}_{\mathbf{w}}\|^4 \right ]\nonumber\\
		& = \alpha^4\E \left [ \| \nabla^2 f_i(\mathbf{w})-\tnabla^2 f_i(\mathbf{w},\D'_{i,0}) \|^4 \right ] \nonumber\\
		& \leq  \alpha^4\frac{\kappa_H+3(D'_{i,0}-1)\sigma^4_H}{D'^{~\!\!3}_{i,0}}:=d'_1,
	\end{align}
	where the last inequality is by Lemma \ref{lemma_intermediate}. Now, assuming $\E \left [ \|\mathbf{E}^{(\nu-1)}_1\|^4 \right ]\leq d'_{\nu -1}$,
	we derive the upper bound $d'_{\nu}$ on $\E \left [ \|\mathbf{E}^{(\nu)}_1\|^4 \right ]$. By following similar steps to \eqref{norm_E1_nu}, we have
	\begin{align}\label{norm4_E1_nu}
		\E \left [ \|\mathbf{E}^{(\nu)}_1\|^4 \right ] & = \E \left [ \|\mathbf{\tilde{B}}^{(\nu)}_{\tilde{\mathbf{w}}}-\mathbf{{B}}^{(\nu)}_{\mathbf{w}}\|^4 \right ]\nonumber\\
		& \leq \E \Big [ \Big(\big\| \left(\mathbf{\tilde{B}}^{(\nu-1)}_{\tilde{\mathbf{w}}}-\mathbf{B}^{(\nu-1)}_{\mathbf{w}}\right) \left(\mathbf{I} - \alpha \tnabla^2 f_i(\tilde{\mathbf{w}}_{\nu-1},\D'_{i,\nu-1})\right)\big\|
		\nonumber\\
		&~~~~~~~~~~+ \alpha\big\|\mathbf{B}^{(\nu-1)}_{\mathbf{w}} \left(\nabla^2 f_i(\mathbf{w}_{\nu-1}) -  \tnabla^2 f_i(\tilde{\mathbf{w}}_{\nu-1},\D'_{i,\nu-1})\right) \big\| \Big)^{\!\!4}~\Big ]\nonumber\\
		& \leq 8\E \Big [\big\| \left(\mathbf{\tilde{B}}^{(\nu-1)}_{\tilde{\mathbf{w}}}-\mathbf{B}^{(\nu-1)}_{\mathbf{w}}\right) \left(\mathbf{I} - \alpha \tnabla^2 f_i(\tilde{\mathbf{w}}_{\nu-1},\D'_{i,\nu-1})\right)\big\|^4\Big]
		\nonumber\\
		&~~~+ 8\alpha^4\E \Big [\big\|\mathbf{B}^{(\nu-1)}_{\mathbf{w}} \left(\nabla^2 f_i(\mathbf{w}_{\nu-1}) -  \tnabla^2 f_i(\tilde{\mathbf{w}}_{\nu-1},\D'_{i,\nu-1})\right) \big\|^4\Big ],
	\end{align}
	where the last inequality is by twice applying the Cauchy-Schwarz inequality, $(a + b)^2 \leq 2(a^2 + b^2)$ for $a,b\geq0$.
	
	Next, we bound each of the two terms in \eqref{norm4_E1_nu}. For the first term we have
	\begin{align}\label{norm4_E1_nu_term1}
		&\E \Big [\big\| \left(\mathbf{\tilde{B}}^{(\nu-1)}_{\tilde{\mathbf{w}}}-\mathbf{B}^{(\nu-1)}_{\mathbf{w}}\right) \left(\mathbf{I} - \alpha \tnabla^2 f_i(\tilde{\mathbf{w}}_{\nu-1},\D'_{i,\nu-1})\right)\big\|^4\Big]\nonumber\\
		&\leq \E \Big[\E \Big[\big\| \mathbf{\tilde{B}}^{(\nu-1)}_{\tilde{\mathbf{w}}}-\mathbf{B}^{(\nu-1)}_{\mathbf{w}}\big\|^4~\!{\big |}~\!\!\left\{\tilde{\mathbf{w}}_j\right\}_{j=1}^{\nu-2}\Big].\E\Big[\big\|\mathbf{I} - \alpha \tnabla^2 f_i(\tilde{\mathbf{w}}_{\nu-1},\D'_{i,\nu-1})\big\|^4~\!{\big |}~\!\!\left\{\tilde{\mathbf{w}}_j\right\}_{j=1}^{\nu-2}\Big]\Big]\nonumber\\
		&\stackrel{(a)}{\leq}
		\left(8(1+\alpha L)^4+8\alpha^4\frac{\kappa_H+3(D'_{i,\nu-1}-1)\sigma^4_H}{D'^{~\!\!3}_{i,\nu-1}}\right)\E \Big[\big\| \mathbf{\tilde{B}}^{(\nu-1)}_{\tilde{\mathbf{w}}}-\mathbf{B}^{(\nu-1)}_{\mathbf{w}}\big\|^4\Big]\nonumber\\
		&\leq
		8d'_{\nu-1}\left((1+\alpha L)^4+\alpha^4\frac{\kappa_H+3(D'_{i,\nu-1}-1)\sigma^4_H}{D'^{~\!\!3}_{i,\nu-1}}\right),
	\end{align}
	where step $(a)$ is by 
	the following result.
	\begin{align}\label{norm4_E1_nu_term1_expct}
		&\E\Big[\big\|\mathbf{I} - \alpha \tnabla^2 f_i(\tilde{\mathbf{w}}_{\nu-1},\D'_{i,\nu-1})\big\|^4~\!{\big |}~\!\!\left\{\tilde{\mathbf{w}}_j\right\}_{j=1}^{\nu-2}\Big]\nonumber\\
		& \leq \E\Big[\left(\big\|\mathbf{I} - \alpha \nabla^2 f_i(\tilde{\mathbf{w}}_{\nu-1})\big\|+\alpha\big\|\tnabla^2 f_i(\tilde{\mathbf{w}}_{\nu-1},\D'_{i,\nu-1}) - \nabla^2 f_i(\tilde{\mathbf{w}}_{\nu-1})\big\|\right)^{\!4}~\!{\big |}~\!\!\left\{\tilde{\mathbf{w}}_j\right\}_{j=1}^{\nu-2}\Big]\nonumber\\
		&\leq 8(1+\alpha L)^4+8\alpha^4\E\Big[\big\|\tnabla^2 f_i(\tilde{\mathbf{w}}_{\nu-1},\D'_{i,\nu-1}) - \nabla^2 f_i(\tilde{\mathbf{w}}_{\nu-1})\big\|^{4}~\!{\big |}~\!\!\left\{\tilde{\mathbf{w}}_j\right\}_{j=1}^{\nu-2}\Big]\nonumber\\
		&\leq 8(1+\alpha L)^4+8\alpha^4\frac{\kappa_H+3(D'_{i,\nu-1}-1)\sigma^4_H}{D'^{~\!\!3}_{i,\nu-1}},
	\end{align}
	where the last inequality is by Lemma \ref{lemma_intermediate}. 
	
	Additionally, for the second term in \eqref{norm4_E1_nu} we have
	\begin{align}\label{norm4_E1_nu_term2}
		&\E \Big [\big\|\mathbf{B}^{(\nu-1)}_{\mathbf{w}} \left(\nabla^2 f_i(\mathbf{w}_{\nu-1}) -  \tnabla^2 f_i(\tilde{\mathbf{w}}_{\nu-1},\D'_{i,\nu-1})\right) \big\|^4\Big ]\nonumber\\
		&\leq \big\|\mathbf{B}^{(\nu-1)}_{\mathbf{w}}\big\|^4.\E \Big [\big\|\tnabla^2 f_i(\tilde{\mathbf{w}}_{\nu-1},\D'_{i,\nu-1})- \nabla^2 f_i(\mathbf{w}_{\nu-1})  \big\|^4\Big ]
		\nonumber\\
		&\leq 8(1+\alpha L)^{4\nu-4}\left(\rho^4 h''_{\nu-1}+\frac{\kappa_H+3(D'_{i,\nu-1}-1)\sigma^4_H}{D'^{~\!\!3}_{i,\nu-1}}\right),
	\end{align}
	where 
	the last inequality is by \eqref{eq51_norm1_1} together with 
	the following result.
	\begin{align}
		&\E \Big [\big\|\tnabla^2 f_i(\tilde{\mathbf{w}}_{\nu-1},\D'_{i,\nu-1})- \nabla^2 f_i(\mathbf{w}_{\nu-1})  \big\|^4\Big ]\nonumber\\
		& \leq 8\E \left [ \E \left [ \left \| \tnabla^2 f_i(\tilde{\mathbf{w}}_{\nu-1},\D'_{i,\nu-1})-\nabla^2 f_i(\tilde{\mathbf{w}}_{\nu-1})\right \|^4 {\big |} ~\! \tilde{\mathbf{w}}_{l-1} \right ] \right ] +
		8\E \left [\left \| \nabla^2 f_i(\tilde{\mathbf{w}}_{\nu-1}) -\nabla^2 f_i(\mathbf{w}_{\nu-1})\right \|^4  \right ] 
		\nonumber\\
		& \stackrel{(a)}{\leq} 8\frac{\kappa_H+3(D'_{i,\nu-1}-1)\sigma^4_H}{D'^{~\!\!3}_{i,\nu-1}} +
		8\rho^4 \E \left [\left \| \tilde{\mathbf{w}}_{\nu-1}-\mathbf{w}_{\nu-1}\right \|^4  \right ]
		\nonumber\\
		& \stackrel{(b)}{\leq} 8\frac{\kappa_H+3(D'_{i,\nu-1}-1)\sigma^4_H}{D'^{~\!\!3}_{i,\nu-1}} +
		8\rho^4 h''_{\nu-1},
	\end{align}
	where 
	step $(a)$ is Lemma \ref{lemma_intermediate} together with Assumption \ref{asm_Hesian_Lip}, and step $(b)$ follows Lemma \ref{lemma_w_tilde_difference}.
	
	By plugging \eqref{norm4_E1_nu_term1} and \eqref{norm4_E1_nu_term2} in \eqref{norm4_E1_nu} we have
	\begin{align}\label{norm4_E1_nu_final}
		\E \left [ \|\mathbf{E}^{(\nu)}_1\|^4 \right ] \leq d'_{\nu},
	\end{align}
	where $d'_{\nu}$ is the solution to the recursive relation
	\begin{align}\label{d'_nu_rec} 
		d'_{\nu} = &  {~} 64d'_{\nu-1}\left((1+\alpha L)^4+\alpha^4\frac{\kappa_H+3(D'_{i,\nu-1}-1)\sigma^4_H}{D'^{~\!\!3}_{i,\nu-1}}\right)\nonumber\\
		&+64\alpha^4(1+\alpha L)^{4\nu-4}\left(\rho^4 h''_{\nu-1}+\frac{\kappa_H+3(D'_{i,\nu-1}-1)\sigma^4_H}{D'^{~\!\!3}_{i,\nu-1}}\right),
	\end{align}
	with the initial value $d'_1=\alpha^4 ({\kappa_H}+3(D'_{i,0}-1)\sigma^4_H)/{D'^{~\!\!3}_{i,0}}$. 

	Finally, we are ready to apply the statistics of $\mathbf{E}^{(\nu)}_1$ and $\mathbf{e}_2$ to derive the desired results. For the first result in the lemma we have (for $\nu>1$)
	\begin{align}\label{lemma4_3_result1}
		&\big\|\E\left[\tnabla F_i(\mathbf{w}) - \nabla F_i(\mathbf{w})\right]\big\| \nonumber\\
		&\stackrel{(a)}{=}\big\|\E\left[\mathbf{{B}}^{(\nu)}_{\mathbf{w}}\mathbf{e}_2+\mathbf{E}^{(\nu)}_1\left(\nabla f_i(\mathbf{w}_{\nu-1} - \alpha \nabla f_i(\mathbf{w}_{\nu-1}))+\mathbf{e}_2\right)\right]\big\|\nonumber\\
		&{\leq}
		\left\|\E\left[\mathbf{{B}}^{(\nu)}_{\mathbf{w}}\mathbf{e}_2\right]\right\|
		+\left\|\E\left[\mathbf{E}^{(\nu)}_1\nabla f_i(\mathbf{w}_{\nu-1} - \alpha \nabla f_i(\mathbf{w}_{\nu-1}))\right]\right\|
		+\left\|\E\left[\mathbf{E}^{(\nu)}_1\mathbf{e}_2\right]\right\|
		\nonumber\\
		&{=}
		\left\|\mathbf{{B}}^{(\nu)}_{\mathbf{w}}\E\left[\mathbf{e}_2\right]\right\|
		+\left\|\E\left[\mathbf{E}^{(\nu)}_1\right]\nabla f_i(\mathbf{w}_{\nu-1} - \alpha \nabla f_i(\mathbf{w}_{\nu-1}))\right\|
		+\left\|\E\left[\mathbf{E}^{(\nu)}_1\mathbf{e}_2\right]\right\|
		\nonumber\\
		&\stackrel{(b)}{\leq}\mu_{F}:=(1+\alpha L)^{\nu}\left(\frac{\sigma_G}{\sqrt{D_{i,\nu}}}+Lh_{\nu}\right)+B\alpha\rho(1+\alpha L)^{\nu-1}\sum_{j=1}^{\nu-1}h_j+\sqrt{d_{\nu}\left(\frac{\sigma^2_G}{D_{i,\nu}} +
			L^2 h'_{\nu}\right)},
	\end{align}
	where step $(a)$ is by comparing \eqref{grad_Fi_rec} and \eqref{F_tilde_rewrite2}, and step $(b)$ is by \eqref{eq51_norm1_1}, \eqref{moment1_e2}, \eqref{E1_mean_new}, Assumption \eqref{asm:boundedness}, and the following result.
	\begin{align}\label{eq103_last}
		\left\|\E\left[\mathbf{E}^{(\nu)}_1\mathbf{e}_2\right]\right\|&\leq \E\left[\left\|\mathbf{E}^{(\nu)}_1\mathbf{e}_2\right\|\right]\nonumber\\
		&\leq \E\left[\left\|\mathbf{E}^{(\nu)}_1\right\|\left\|\mathbf{e}_2\right\|\right]
		\nonumber\\
		&\stackrel{(a)}{\leq} \sqrt{\E\left[\left\|\mathbf{E}^{(\nu)}_1\right\|^2\right]\E\left[\left\|\mathbf{e}_2\right\|\right]^2}
		\nonumber\\
		&\stackrel{(b)}{\leq} \sqrt{d_{\nu}\left(\frac{\sigma^2_G}{D_{i,\nu}} +
			L^2 h'_{\nu}\right)},
	\end{align}
	where step $(a)$ is by using the Cauchy-Schwarz inequality  $\E[|XY|]\leq\sqrt{\E[X^2]\E[Y^2]}$ for arbitrary random variables $X$ and $Y$, and step $(b)$ is by \eqref{moment2_e2} and \eqref{norm_E1_nu_final}.

	
	To show the second result in the lemma, by comparing \eqref{grad_Fi_rec} and \eqref{F_tilde_rewrite2} we have
	\begin{align}
		\left \| \tnabla F_i(\mathbf{w}) - \nabla F_i(\mathbf{w}) \right\| & = 
		\left\|\mathbf{{B}}^{(\nu)}_{\mathbf{w}}\mathbf{e}_2+\mathbf{E}^{(\nu)}_1\nabla f_i(\mathbf{w}_{\nu-1} - \alpha \nabla f_i(\mathbf{w}_{\nu-1}))+\mathbf{E}^{(\nu)}_1\mathbf{e}_2\right\|
		\nonumber\\
		& \leq
		\left\|\mathbf{{B}}^{(\nu)}_{\mathbf{w}}\right\|\left\|\mathbf{e}_2\right\|+B\left\|\mathbf{E}^{(\nu)}_1\right\|+\left\|\mathbf{E}^{(\nu)}_1\right\|\left\|\mathbf{e}_2\right\|.
	\end{align}
	Next, by the Cauchy-Schwarz inequality $(a+b+c)^2 \leq 3(a^2+b^2+c^2)$ for $a,b,c \geq 0$, we have
	\begin{align}\label{eq77}
		&\left \| \tnabla F_i(\mathbf{w}) - \nabla F_i(\mathbf{w}) \right \|^2  \nonumber\\
		& \leq 3\left\|\mathbf{{B}}^{(\nu)}_{\mathbf{w}}\right\|^2\left\|\mathbf{e}_2\right\|^2+3B^2\left\|\mathbf{E}^{(\nu)}_1\right\|^2+3\left\|\mathbf{E}^{(\nu)}_1\right\|^2\left\|\mathbf{e}_2\right\|^2
		\nonumber\\
		& \leq 3(1+\alpha L)^{2\nu}\left\|\mathbf{e}_2\right\|^2+3B^2\left\|\mathbf{E}^{(\nu)}_1\right\|^2+3\left\|\mathbf{E}^{(\nu)}_1\right\|^2\left\|\mathbf{e}_2\right\|^2.
	\end{align} 
	By taking expectation from both sides of \eqref{eq77}
	we have
	\begin{align}\label{proof2_ineq4}
		&\E \left [ \left \| \tnabla F_i(\mathbf{w}) - \nabla F_i(\mathbf{w}) \right \|^2 \right ]\nonumber\\
		& \leq 3(1+\alpha L)^{2\nu}\E\left[\left\|\mathbf{e}_2\right\|^2\right]+3B^2\E\left[\left\|\mathbf{E}^{(\nu)}_1\right\|^2\right]+3\E\left[\left\|\mathbf{E}^{(\nu)}_1\right\|^2\left\|\mathbf{e}_2\right\|^2\right]
		\nonumber\\
		& \stackrel{(a)}{\leq} 3(1+\alpha L)^{2\nu}\E\left[\left\|\mathbf{e}_2\right\|^2\right]+3B^2\E\left[\left\|\mathbf{E}^{(\nu)}_1\right\|^2\right]+3\sqrt{\E\left[\left\|\mathbf{E}^{(\nu)}_1\right\|^4\right]\E\left[\left\|\mathbf{e}_2\right\|^4\right]}
		\nonumber\\
		& \stackrel{(b)}{\leq} 3(1+\alpha L)^{2\nu}\left(\frac{\sigma^2_G}{D_{i,\nu}} +
		L^2 h'_{\nu}\right) +3B^2d_{\nu}+6\sqrt{2d'_{\nu}\left(\frac{\kappa_G+3(D_{i,\nu}-1)\sigma^4_G}{D_{i,\nu}^3}+L^4 h''_{\nu}\right)},
	\end{align}
	where step $(a)$ is by the Cauchy-Schwarz inequality, and step $(b)$ follows \eqref{moment2_e2}, \eqref{moment4_e2}, \eqref{norm_E1_nu_final}, and \eqref{norm4_E1_nu_final}.
	\endproof

	\section{Proof of Lemma \ref{lemma_F_similar}}\label{App_C}
	Let us define the matrix $\mathbf{H}^{(\nu)}_{\mathbf{w}}$, for $\nu\geq 1$, as
	\begin{align}\label{H_w}
		\mathbf{H}^{(\nu)}_{\mathbf{w}} := & \left(\mathbf{I} - \alpha \nabla^2 f(\mathbf{w})\right) \left(\mathbf{I} - \alpha \nabla^2 f(\mathbf{w}_1)\right) \left(\mathbf{I} - \alpha \nabla^2 f(\mathbf{w}_2)\right) \cdots \left(\mathbf{I} - \alpha \nabla^2 f(\mathbf{w}_{\nu-1})\right), 
	\end{align}
	which can be recursively expressed as
	\begin{align}\label{H_w_rec}
		\mathbf{H}^{(\nu)}_{\mathbf{w}} &=  \mathbf{H}^{(\nu-1)}_{\mathbf{w}} \left(\mathbf{I} - \alpha \nabla^2 f(\mathbf{w}_{\nu-1})\right),\nonumber\\
		\mathbf{H}^{(1)}_{\mathbf{w}} &=
		\mathbf{I} - \alpha \nabla^2 f(\mathbf{w}). 
	\end{align}
	Using the recursive form \eqref{grad_Fi_rec}, we can then rewrite $\nabla F_i(\mathbf{w})$ as
	\begin{align}\label{eq83}
		\nabla F_i(\mathbf{w}) = \left(\mathbf{H}^{(\nu)}_{\mathbf{w}}+\mathbf{M}_i^{(\nu)}\right)\left(\nabla f(\mathbf{w}_{\nu-1} - \alpha \nabla f(\mathbf{w}_{\nu-1}))+\mathbf{r}_i\right),
	\end{align}
	where 
	\begin{align}
		\mathbf{M}_i^{(\nu)} & := \mathbf{B}^{(\nu)}_{\mathbf{w}} - \mathbf{H}^{(\nu)}_{\mathbf{w}},\label{Mi}\\
		\mathbf{r}_i & := \nabla f_i(\mathbf{w}_{\nu-1} - \alpha \nabla f_i(\mathbf{w}_{\nu-1})) - \nabla f(\mathbf{w}_{\nu-1} - \alpha \nabla f(\mathbf{w}_{\nu-1})).\label{ri}
	\end{align}
	
	First, note that 
	\begin{align}
		\| \mathbf{r}_i\| & =  \|\nabla f_i(\mathbf{w}_{\nu-1} - \alpha \nabla f_i(\mathbf{w}_{\nu-1})) - \nabla f_i(\mathbf{w}_{\nu-1} - \alpha \nabla f(\mathbf{w}_{\nu-1}))\nonumber\\
		&~~~~~ + \nabla f_i(\mathbf{w}_{\nu-1} - \alpha \nabla f(\mathbf{w}_{\nu-1})) - \nabla f(\mathbf{w}_{\nu-1} - \alpha \nabla f(\mathbf{w}_{\nu-1}))\|
		\nonumber\\
		&\leq \left \| \nabla f_i(\mathbf{w}_{\nu-1} - \alpha \nabla f_i(\mathbf{w}_{\nu-1})) - \nabla f_i(\mathbf{w}_{\nu-1} - \alpha \nabla f(\mathbf{w}_{\nu-1})) \right \| \nonumber\\
		&~~~~~ + 	\left \| \nabla f_i(\mathbf{w}_{\nu-1} - \alpha \nabla f(\mathbf{w}_{\nu-1})) - \nabla f(\mathbf{w}_{\nu-1} - \alpha \nabla f(\mathbf{w}_{\nu-1})) \right \| \nonumber \\
		& \leq \alpha L \|  \nabla \!f_i(\mathbf{w}_{\nu-1}) \!-\!  \nabla\! f(\mathbf{w}_{\nu-1}) \| \!+\!  \left \| \nabla\! f_i(\mathbf{w}_{\nu-1} \!-\! \alpha \nabla \!f(\mathbf{w}_{\nu-1}))\! -\! \nabla f(\mathbf{w}_{\nu-1} \!-\! \alpha \nabla\! f(\mathbf{w}_{\nu-1})) \right \|.
	\end{align}
	Therefore, using the inequality $(a+b)^2 \leq 2(a^2+b^2)$, we have
	\begin{align}\label{norm_ri}
		\frac{1}{N} \sum_{i=1}^N \|\mathbf{r}_i\|^2 \leq & ~\! \frac{2}{N} \sum_{i=1}^N \big ( (\alpha L)^2 \|  \nabla f_i(\mathbf{w}_{\nu-1}) -  \nabla f(\mathbf{w}_{\nu-1}) \|^2 \nonumber\\
		& ~~~~~~~~~~ + \left \| \nabla f_i(\mathbf{w}_{\nu-1} - \alpha \nabla f(\mathbf{w}_{\nu-1})) - \nabla f (\mathbf{w}_{\nu-1} - \alpha \nabla f (\mathbf{w}_{\nu-1})) \right \|^2 \big ) \nonumber \\
		\leq & ~\! 2\gamma_G^2 \left (1+ \alpha^2 L^2 \right ).
	\end{align}
	
	Next, we apply a proof by induction to derive the bound $g_{\nu}$ to $\frac{1}{N} \sum_{i=1}^N \|\mathbf{M}_i^{(\nu)}\|^2$. Note that for the base case of $\nu=1$ we have
	\begin{align}\label{norm_Mi_nu1}
		\frac{1}{N} \sum_{i=1}^N \|\mathbf{M}_i^{(1)}\|^2 & =~\! \frac{1}{N} \sum_{i=1}^N \|\mathbf{B}^{(1)}_{\mathbf{w}} - \mathbf{H}^{(1)}_{\mathbf{w}}\|^2\nonumber\\
		& =~\! \frac{1}{N} \sum_{i=1}^N \alpha^2\|\nabla^2 f_i(\mathbf{w})- \nabla^2 f(\mathbf{w}) \|^2\nonumber\\
		& \leq~\!
		\alpha^2 \gamma_H^2.	
	\end{align}
	Now, assuming that step $\nu-1$ holds with some bound $g_{\nu -1}$, i.e., $\frac{1}{N} \sum_{i=1}^N \|\mathbf{M}_i^{(\nu-1)}\|^2\leq g_{\nu -1}$,
	the goal is to derive the bound $g_{\nu}$ at step $\nu$.
	Using \eqref{B_w_rec} and \eqref{H_w_rec} we have
	\begin{align}\label{norm_Mi_nu}
		\left\|\mathbf{M}_i^{(\nu)}\right\|^2 & = ~\! \left\|\mathbf{B}^{(\nu-1)}_{\mathbf{w}} \left(\mathbf{I} - \alpha \nabla^2 f_i(\mathbf{w}_{\nu-1})\right) - \mathbf{H}^{(\nu-1)}_{\mathbf{w}} \left(\mathbf{I} - \alpha \nabla^2 f(\mathbf{w}_{\nu-1})\right)\right\|^2 \nonumber\\
		& = ~\! \big\|\mathbf{B}^{(\nu-1)}_{\mathbf{w}} \left(\mathbf{I} - \alpha \nabla^2 f_i(\mathbf{w}_{\nu-1})\right) - \mathbf{B}^{(\nu-1)}_{\mathbf{w}} \left(\mathbf{I} - \alpha \nabla^2 f(\mathbf{w}_{\nu-1})\right) \nonumber\\
		& ~~~~~~ + \mathbf{B}^{(\nu-1)}_{\mathbf{w}} \left(\mathbf{I} - \alpha \nabla^2 f(\mathbf{w}_{\nu-1})\right) - \mathbf{H}^{(\nu-1)}_{\mathbf{w}} \left(\mathbf{I} - \alpha \nabla^2 f(\mathbf{w}_{\nu-1})\right)\big\|^2
		\nonumber\\
		& \stackrel{(a)}{\leq} ~\! 2\alpha^2\big\|\mathbf{B}^{(\nu-1)}_{\mathbf{w}} \left( \nabla^2 f_i(\mathbf{w}_{\nu-1}) - \nabla^2 f(\mathbf{w}_{\nu-1})\right)\big\|^2 \nonumber\\
		& ~~~~ + 2\big\|\left(\mathbf{B}^{(\nu-1)}_{\mathbf{w}}-\mathbf{H}^{(\nu-1)}_{\mathbf{w}}\right) \left(\mathbf{I} - \alpha \nabla^2 f(\mathbf{w}_{\nu-1})\right)\big\|^2
		\nonumber\\
		& \stackrel{(b)}{\leq} ~\! 2\alpha^2(1+\alpha L)^{2\nu-2}\big\|\nabla^2 f_i(\mathbf{w}_{\nu-1}) - \nabla^2 f(\mathbf{w}_{\nu-1})\big\|^2 + 2(1+\alpha L)^{2}\left\|\mathbf{M}_i^{(\nu-1)}\right\|^2,	
	\end{align}
	where step $(a)$ is by $(a+b)^2 \leq 2(a^2+b^2)$, and step $(b)$ is by \eqref{eq51_norm1_1} and \eqref{Mi}. By taking the average from both sides of \eqref{norm_Mi_nu}, we have
	\begin{align}\label{norm_Mi_nu_final}
		\frac{1}{N} \sum_{i=1}^N \left\|\mathbf{M}_i^{(\nu)}\right\|^2 \leq ~\! &  2\alpha^2(1+\alpha L)^{2\nu-2}.\frac{1}{N} \sum_{i=1}^N\big\|\nabla^2 f_i(\mathbf{w}_{\nu-1}) - \nabla^2 f(\mathbf{w}_{\nu-1})\big\|^2 \nonumber\\
		&+ 2(1+\alpha L)^{2}.\frac{1}{N} \sum_{i=1}^N\left\|\mathbf{M}_i^{(\nu-1)}\right\|^2\nonumber\\
		\leq ~\! &  2\alpha^2(1+\alpha L)^{2\nu-2}\gamma_H^2+ 2(1+\alpha L)^{2}g_{\nu-1}.
	\end{align}
	Therefore, by solving the recursive relation
	\begin{align}
		g_{\nu}=2\alpha^2(1+\alpha L)^{2\nu-2}\gamma_H^2+ 2(1+\alpha L)^{2}g_{\nu-1},
	\end{align}
	with the initial value $g_1=\alpha^2 \gamma_H^2$ (from \eqref{norm_Mi_nu1}), we find the following upper bound on $\frac{1}{N} \sum_{i=1}^N \|\mathbf{M}_i^{(\nu)}\|^2$
	\begin{align}\label{g_nu}
		\frac{1}{N} \sum_{i=1}^N \left\|\mathbf{M}_i^{(\nu)}\right\|^2 \leq ~\! 
		g_{\nu}  := \alpha^2 \gamma_H^2 (1+\alpha L)^{2\nu-2}\left(2^{\nu-1}+\sum_{l=1}^{\nu-1}2^l\right).
	\end{align}

	Finally, to prove the lemma, we note that the goal is to bound the variance of $\nabla F_i(\mathbf{w})$ when $i$ is drawn from a uniform distribution. Since subtracting a constant from a random variable does not change its variance, the variance of $\nabla F_i(\mathbf{w})$ is equal to the variance of $\nabla F_i(\mathbf{w}) - \mathbf{H}^{(\nu)}_{\mathbf{w}}\nabla f(\mathbf{w}_{\nu-1} - \alpha \nabla f(\mathbf{w}_{\nu-1}))$, which itself can be upper bounded by its second moment. Therefore, 
	\begin{align}
		\frac{1}{N} \sum_{i=1}^N \| \nabla F_i(\mathbf{w})  &- \nabla F(\mathbf{w})\|^2 
		\nonumber\\ 
		&\leq \frac{1}{N} \sum_{i=1}^N \left\|\mathbf{M}_i^{(\nu)}\nabla f(\mathbf{w}_{\nu-1} - \alpha \nabla f(\mathbf{w}_{\nu-1}))+\mathbf{H}^{(\nu)}_{\mathbf{w}}\mathbf{r}_i+\mathbf{M}_i^{(\nu)}\mathbf{r}_i\right\|^2 
		\nonumber\\
		&\leq \frac{1}{N} \sum_{i=1}^N\left( \left\|\mathbf{M}_i^{(\nu)}\nabla f(\mathbf{w}_{\nu-1} - \alpha \nabla f(\mathbf{w}_{\nu-1}))\right\|+\left\|\mathbf{H}^{(\nu)}_{\mathbf{w}}\mathbf{r}_i\right\|+\left\|\mathbf{M}_i^{(\nu)}\mathbf{r}_i\right\|\right)^2 
		\nonumber\\
		& \stackrel{(a)}{\leq} 3B^2\frac{1}{N}\sum_{i=1}^N \left\|\mathbf{M}_i^{(\nu)}\right\|^2+3(1+\alpha L)^{2\nu}\frac{1}{N}\sum_{i=1}^N\left\|\mathbf{r}_i\right\|^2+3\frac{1}{N}\sum_{i=1}^N\left\|\mathbf{M}_i^{(\nu)}\right\|^2 \left\|\mathbf{r}_i\right\|^2
		\nonumber\\
		& \stackrel{(b)}{\leq} 3B^2g_{\nu}+6\gamma_G^2(1+\alpha L)^{2\nu} \left (1+ \alpha^2 L^2\right )+3\frac{1}{N}\sum_{i=1}^N\left\|\mathbf{M}_i^{(\nu)}\right\|^2 \left\|\mathbf{r}_i\right\|^2
		\nonumber\\
		& \stackrel{(c)}{\leq} 15B^2g_{\nu}+6\gamma_G^2(1+\alpha L)^{2\nu} \left (1+ \alpha^2 L^2 \right ),\label{eq93}
	\end{align}
	where step $(a)$ is by $\left \| \nabla f (\mathbf{w} - \alpha \nabla f (\mathbf{w})) \right \| \leq B$, $\|\mathbf{H}^{(\nu)}_{\mathbf{w}}\|\leq (1+\alpha L)^{\nu}$ (similar to \eqref{eq51_norm1_1}), as well as the Cauchy-Schwarz inequality $(a+b+c)^2 \leq 3(a^2+b^2+c^2)$ for $a,b,c \geq 0$. Moreover, step $(b)$ is by \eqref{norm_ri} and \eqref{g_nu}, and step $(c)$ follows
	\begin{align}
		\frac{1}{N}\sum_{i=1}^N\left\|\mathbf{M}_i^{(\nu)}\right\|^2 \left\|\mathbf{r}_i\right\|^2 & \leq \max_{i} \left\|\mathbf{r}_i\right\|^2 \left (\frac{1}{N}\sum_{i=1}^N\left\|\mathbf{M}_i^{(\nu)}\right\|^2\right) \nonumber\\
		& \leq 4B^2g_{\nu}.
	\end{align}
	Eq. \eqref{eq93} thus gives the desired result.
	\endproof
	
\section{Proof of Theorem \ref{theorem_main}}\label{App_thrm}
The proof of Theorem \ref{theorem_main} follows the same procedure as the proof provided in \cite[Appendix G]{fallah2020personalized}. This is because the local update rule at every communication round and for every agent is according to \eqref{local_update_2} for both our scheme and the Per-FedAvg algorithm \cite{fallah2020personalized}. The main difference, however, is the definition of meta loss functions $F_i$'s. Therefore, it is easy to verify that a similar proof directly extends here though with different values for the statistics of meta loss functions, i.e., $L_F$, $\mu_F$, $\sigma_F$, and $\gamma_F$, which are characterized specifically for our algorithm in Lemmas \ref{lemma_F_smooth}, \ref{lemma_err_moment}, and \ref{lemma_F_similar}.\endproof

	\section{Proof of Lemma \ref{lemma_err_moment_FO}}\label{App_FO}
	Note that $\tnabla F_i^{\rm FO}$ is estimating 
	\begin{align}\label{g_i}
		\mathbf{g}_i^{\rm FO}(\mathbf{w}) := \nabla f_i(\mathbf{w}_{\nu-1} - \alpha \nabla f_i(\mathbf{w}_{\nu-1})).
	\end{align}
	For the first result we have
	\begin{align} \label{mu_FO_1}
		&	\left \| \E \left [ \tnabla F_i^{\rm FO}(\mathbf{w}) \!-\! \nabla F_i(\mathbf{w}) \right ] \right \| \leq 	\left \| \E \left [ \tnabla F_i^{\rm FO}(\mathbf{w}) - \mathbf{g}_i^{\rm FO}(\mathbf{w}) \right ] \right \| \!+\! 	\left \| \E \left [ \mathbf{g}_i^{\rm FO}(\mathbf{w}) \!-\! \nabla F_i(\mathbf{w}) \right ] \right \|.
	\end{align}
	To bound the first term in \eqref{mu_FO_1}, note that
	we can rewrite $\tnabla F^{\rm FO}_i(\mathbf{w})$ 
	as
	\begin{align}\label{grad_tilde_FO_rewrite1}
		\tnabla F^{\rm FO}_i(\mathbf{w})=\mathbf{g}_i^{\rm FO}(\mathbf{w})+\mathbf{e}_2,
	\end{align}
	where $\mathbf{e}_2$ is defined in \eqref{e2} as $\mathbf{e}_2:=\tnabla f_i(\tilde{\mathbf{w}}_{\nu},\D_{i,\nu}) - \nabla f_i(\mathbf{w}_{\nu})$ with its moments bounded in \eqref{moment1_e2} and \eqref{moment2_e2}. 
	Using \eqref{grad_tilde_FO_rewrite1}, we have for the first term in \eqref{mu_FO_1}
	\begin{align}\label{lemma4_3_result1_FO}
		\big\|\E\left[\tnabla F^{\rm FO}_i(\mathbf{w}) - \mathbf{g}^{\rm FO}_i(\mathbf{w})\right]\big\|&=\big\|\E\left[\mathbf{e}_2\right]\big\| \nonumber\\
		&{\leq}\frac{\sigma_G}{\sqrt{D_{i,\nu}}}+Lh_{\nu},
	\end{align}
	where the last inequality follows \eqref{moment1_e2}. Additionally, using the recursive form for $\nabla F_i(\mathbf{w})$ in \eqref{grad_Fi_rec}, we have for the second term in \eqref{mu_FO_1}
	\begin{align}\label{mu_FO_2}
		\left \| \E \left [ \mathbf{g}_i^{\rm FO}(\mathbf{w}) \!-\! \nabla F_i(\mathbf{w}) \right ] \right \| &= \left \| \mathbf{B}^{(\nu)}_{\mathbf{w}} \nabla f_i(\mathbf{w}_{\nu}) -  \nabla f_i(\mathbf{w}_{\nu}) \right \|  \nonumber\\
		& \leq \left \| \mathbf{B}^{(\nu)}_{\mathbf{w}} \right\| \left \|\nabla f_i(\mathbf{w}_{\nu})\right\| +\left \|  \nabla f_i(\mathbf{w}_{\nu}) \right \|  \nonumber\\
		& \leq B\left((1+\alpha L)^{\nu} + 1\right),
	\end{align}
	where the last inequality is by Assumption \ref{asm_grad} and \eqref{eq51_norm1_1}. Substituting \eqref{lemma4_3_result1_FO} and \eqref{mu_FO_2} in \eqref{mu_FO_1} gives  the first result in the lemma.
	
	Similarly, for the second result in the lemma we have
	\begin{align}\label{sigma_FO_1}
		\E \left [ \left \| \tnabla F_i^{\rm FO}(\mathbf{w}) - \nabla F_i(\mathbf{w})\right \|^2 \right ] \leq 2\E \left [ \left \| \tnabla F_i^{\rm FO}(\mathbf{w}) - \mathbf{g}^{\rm FO}_i(\mathbf{w})\right \|^2 \right ] + 2\E \left [ \left \| \mathbf{g}^{\rm FO}_i(\mathbf{w}) - \nabla F_i(\mathbf{w})\right \|^2 \right ].
	\end{align}
	By applying \eqref{moment2_e2}, we have for the first term in \eqref{sigma_FO_1}
	\begin{align}\label{lemma4_3_result2_FO}
		\E\left[\left\|\tnabla F^{\rm FO}_i(\mathbf{w}) - \mathbf{g}^{\rm FO}_i(\mathbf{w})\right\|^2\right]&=\E\left[\left\|\mathbf{e}_2\right\|^2\right] \nonumber\\&\leq\frac{\sigma^2_G}{D_{i,\nu}} +
		L^2 h'_{\nu}.
	\end{align}
	Plugging \eqref{mu_FO_2} and \eqref{lemma4_3_result2_FO} into \eqref{sigma_FO_1} complete the proof of the second result in the lemma.
	\endproof

	\section{Proof of Lemma \ref{lemma_err_moment_HF}}\label{App_HF}
	Note that $\tnabla F^{\rm HF}_i(\mathbf{w})$ is indeed estimating $\mathbf{g}_i^{\rm HF}(\mathbf{w}):= \mathbf{d}_{\mathbf{w},i}^{(\nu)}$, where the vector $\mathbf{d}_{\mathbf{w},i}^{(l)}$, for $l=1,\cdots,\nu$, is recursively defined as
	\begin{align}\label{d_HF}
		\mathbf{d}_{\mathbf{w},i}^{(l)} =  \mathbf{d}_{\mathbf{w},i}^{(l-1)} - \alpha \frac{\nabla f_i(\mathbf{w}_{\nu-l}+\delta \mathbf{d}_{\mathbf{w},i}^{(l-1)})-\nabla f_i(\mathbf{w}_{\nu-l}-\delta \mathbf{d}_{\mathbf{w},i}^{(l-1)})}{2\delta},
	\end{align}
	with the initial value of $\mathbf{d}_{\mathbf{w},i}^{(0)}:=\nabla f_i(\mathbf{w}_{\nu-1} - \alpha \nabla f_i(\mathbf{w}_{\nu-1}))$.

	We first derive the second result in the lemma. Note that
	\begin{align}\label{sigma_HF_1}
		\E \left [ \left \| \tnabla F_i^{\rm HF}(\mathbf{w}) - \nabla F_i(\mathbf{w})\right \|^2 \right ] \leq 2\E \left [ \left \| \tnabla F_i^{\rm HF}(\mathbf{w}) - \mathbf{g}^{\rm HF}_i(\mathbf{w})\right \|^2 \right ] + 2\E \left [ \left \| \mathbf{g}^{\rm HF}_i(\mathbf{w}) - \nabla F_i(\mathbf{w})\right \|^2 \right ].
	\end{align}
	Next, we bound each term in \eqref{sigma_HF_1} separately. 
	To bound the first term in \eqref{sigma_HF_1},	
	using \eqref{grad_tilde_Fi_HF} and \eqref{d_HF}, we have $\E \left [  \| \tnabla F^{\rm HF}_i(\mathbf{w}) - \mathbf{g}^{\rm HF}_i(\mathbf{w}) \|^2 \right ]=\E \left [ \| \tilde{\mathbf{d}}_{\mathbf{w},i}^{(\nu)} - \mathbf{d}_{\mathbf{w},i}^{(\nu)} \|^2 \right ]$. We apply induction to upper bound $\E \left [ \| \tilde{\mathbf{d}}_{\mathbf{w},i}^{(l)} - \mathbf{d}_{\mathbf{w},i}^{(l)} \|^2 \right ]$, for $l\in\{0,1,\cdots,\nu\}$, by some positive quantity $p'_l$. For the base case of $l=0$ we have
	\begin{align}\label{p'_bound_base}
		\E \!\left [ \left\| \tilde{\mathbf{d}}_{\mathbf{w},i}^{(0)} - \mathbf{d}_{\mathbf{w},i}^{(0)} \right\|^2 \right ] & \!= \E\! \left [ \left\| \tnabla f_i(\tilde{\mathbf{w}}_{\nu-1} - \!\alpha \tnabla f_i(\tilde{\mathbf{w}}_{\nu-1},\D_{i,\nu-1}),\D_{i,\nu}) \!-\! \nabla f_i(\mathbf{w}_{\nu-1} -\! \alpha \nabla f_i(\mathbf{w}_{\nu-1})) \right\|^2 \right ]\nonumber\\
		& \leq \frac{\sigma^2_G}{D_{i,\nu}} +
		L^2 h'_{\nu}:=p'_0,
	\end{align}
	where the last inequality follows \eqref{moment2_e2} with $h'_{\nu}$  defined in Lemma \eqref{lemma_w_tilde_difference}.
	Now, as the induction hypothesis, assuming $\E\left[\|\tilde{\mathbf{d}}_{\mathbf{w},i}^{(l-1)}-\mathbf{d}_{\mathbf{w},i}^{(l-1)}\|^2\right]\leq p'_{l-1}$, we have for $l\geq 1$
	\begin{align}
		& \hspace{-1cm}\E\left[	\left\|\tilde{\mathbf{d}}_{\mathbf{w},i}^{(l)}-\mathbf{d}_{\mathbf{w},i}^{(l)}\right\|^2\right] \nonumber\\
		& = \E\big[\big\|\left(\tilde{\mathbf{d}}_{\mathbf{w},i}^{(l-1)}-\mathbf{d}_{\mathbf{w},i}^{(l-1)}\right)  - \frac{\alpha}{2\delta} \left(\tnabla f_i(\tilde{\mathbf{w}}_{\nu-l}+\delta \tilde{\mathbf{d}}_{\mathbf{w},i}^{(l-1)},\D'_{i,\nu-l}) - \nabla f_i(\mathbf{w}_{\nu-l}+\delta \mathbf{d}_{\mathbf{w},i}^{(l-1)})\right) \nonumber\\
		& \hspace{1.4cm} + \frac{\alpha}{2\delta} \left(\tnabla f_i(\tilde{\mathbf{w}}_{\nu-l}-\delta \tilde{\mathbf{d}}_{\mathbf{w},i}^{(l-1)},\D'_{i,\nu-l}) - \nabla f_i(\mathbf{w}_{\nu-l}-\delta \mathbf{d}_{\mathbf{w},i}^{(l-1)})\right)
		\big\|^2\big]\nonumber\\
		& \stackrel{(a)}{\leq} 3p'_{l-1} +  \frac{3\alpha^2}{4\delta^2}\E\big[\big\| \tnabla f_i(\tilde{\mathbf{w}}_{\nu-l}+\delta \tilde{\mathbf{d}}_{\mathbf{w},i}^{(l-1)},\D'_{i,\nu-l}) - \nabla f_i(\mathbf{w}_{\nu-l}+\delta \mathbf{d}_{\mathbf{w},i}^{(l-1)})\big\|^2\big] \nonumber\\
		& \hspace{0.6cm} +  \frac{3\alpha^2}{4\delta^2}\E\big[\big\| \tnabla f_i(\tilde{\mathbf{w}}_{\nu-l}-\delta \tilde{\mathbf{d}}_{\mathbf{w},i}^{(l-1)},\D'_{i,\nu-l}) - \nabla f_i(\mathbf{w}_{\nu-l}-\delta \mathbf{d}_{\mathbf{w},i}^{(l-1)})\big\|^2\big]
		\nonumber\\
		& \stackrel{(b)}{\leq} p'_l:=3p'_{l-1}\left(1+ \alpha^2 L^2\right) +  \frac{3\alpha^2}{2\delta^2}\left( \frac{\sigma^2_G}{\D'_{i,\nu-l}} +
		2L^2 h'_{\nu-l}\right),\label{p'_rec}
	\end{align}
	where step $(a)$ is by the Cauchy-Schwarz inequality $(a+b+c)^2 \leq 3(a^2+b^2+c^2)$ for $a,b,c \geq 0$ and the induction hypothesis. Moreover, step $(b)$ is by the following result.
	\begin{align}
		&\hspace{-2cm}\E\big[\big\| \tnabla f_i(\tilde{\mathbf{w}}_{\nu-l}\pm\delta \tilde{\mathbf{d}}_{\mathbf{w},i}^{(l-1)},\D'_{i,\nu-l}) - \nabla f_i(\mathbf{w}_{\nu-l}\pm\delta \mathbf{d}_{\mathbf{w},i}^{(l-1)})\big\|^2\big]\nonumber\\
		& \stackrel{(a)}{\leq} \frac{\sigma^2_G}{\D'_{i,\nu-l}} +
		L^2 \E \left [\left \| \left(\tilde{\mathbf{w}}_{\nu-l}-\mathbf{w}_{\nu-l}\right)\pm\delta\left( \tilde{\mathbf{d}}_{\mathbf{w},i}^{(l-1)}- \mathbf{d}_{\mathbf{w},i}^{(l-1)}\right) \right \|^2  \right ]
		\nonumber\\
		& {\leq} \frac{\sigma^2_G}{\D'_{i,\nu-l}} +
		2L^2 \E \left [\left \| \tilde{\mathbf{w}}_{\nu-l}-\mathbf{w}_{\nu-l} \right \|^2  \right ] + 2L^2\delta^2 \E \left [\left \| \tilde{\mathbf{d}}_{\mathbf{w},i}^{(l-1)}- \mathbf{d}_{\mathbf{w},i}^{(l-1)} \right \|^2  \right ]
		\nonumber\\
		& \stackrel{(b)}{\leq} \frac{\sigma^2_G}{\D'_{i,\nu-l}} +
		2L^2 h'_{\nu-l} + 2L^2\delta^2 p'_{l-1},
	\end{align}
	where step $(a)$ follows the same procedure as \eqref{h'_l_rec_stepB} (i.e., defining $\tilde{\mathbf{w}}_{\rm tmp}:=\tilde{\mathbf{w}}_{\nu-l}\pm\delta \tilde{\mathbf{d}}_{\mathbf{w},i}^{(l-1)}$ and $\mathbf{w}_{\rm tmp}:=\mathbf{w}_{\nu-l}\pm\delta \mathbf{d}_{\mathbf{w},i}^{(l-1)}$ and then conditioning on $\tilde{\mathbf{w}}_{\rm tmp}$), and step $(b)$ follows Lemma \ref{lemma_w_tilde_difference} and the induction hypothesis. 
	
	Next, to bound the second term in \eqref{sigma_HF_1}, let us define
	\begin{align}\label{s_HF}
		\mathbf{s}_{\mathbf{w},i}^{(l)}:=\left[\prod_{j=\nu-l}^{\nu-1}\left(\mathbf{I} - \alpha \nabla^2 f_i(\mathbf{w}_j)\right)\right]\nabla f_i(\mathbf{w}_{\nu}), ~~~ l=1,2,\cdots,\nu,
	\end{align}
	which can also be recursively defined as $\mathbf{s}_{\mathbf{w},i}^{(l)}:=\left(\mathbf{I} - \alpha \nabla^2 f_i(\mathbf{w}_{\nu-l})\right)\mathbf{s}_{\mathbf{w},i}^{(l-1)}$ with the initial value of $\mathbf{s}_{\mathbf{w},i}^{(0)}:=\nabla f_i(\mathbf{w}_{\nu})$. Clearly, $\mathbf{s}_{\mathbf{w},i}^{(\nu)}=\nabla F_i(\mathbf{w})$, and $\mathbf{d}_{\mathbf{w},i}^{(l)}$ in \eqref{d_HF} is the HF approximation of $\mathbf{s}_{\mathbf{w},i}^{(l)}$. Therefore,
	\begin{align}\label{Error_HF_result2_term2}
		\E \left [ \left \| \mathbf{g}^{\rm HF}_i(\mathbf{w}) - \nabla F_i(\mathbf{w})\right \|^2 \right ] = \left \| \mathbf{d}_{\mathbf{w},i}^{(\nu)} - \mathbf{s}_{\mathbf{w},i}^{(\nu)}\right \|^2.
	\end{align}

	We  apply induction to upper bound $\left \| \mathbf{d}_{\mathbf{w},i}^{(l)} - \mathbf{s}_{\mathbf{w},i}^{(l)}\right \|$ by some quantity $q_l$, for $l=\{1,\cdots,\nu\}$. Note that for the base case we have
	\begin{align}\label{q_1}
		\left \| \mathbf{d}_{\mathbf{w},i}^{(1)} - \mathbf{s}_{\mathbf{w},i}^{(1)}\right \|&\stackrel{(a)}{=} \alpha\left \| \nabla^2 f_i(\mathbf{w}_{\nu-1})\nabla f_i(\mathbf{w}_{\nu})  -  \frac{\nabla f_i(\mathbf{w}_{\nu-l}+\delta \nabla f_i(\mathbf{w}_{\nu}))-\nabla f_i(\mathbf{w}_{\nu-l}-\delta \nabla f_i(\mathbf{w}_{\nu}))}{2\delta}\right \|\nonumber\\
		&\stackrel{(b)}{\leq} \alpha \rho\delta B^2:=q_1,
	\end{align}
	where step $(a)$ is by \eqref{d_HF} and \eqref{s_HF}, and step $(b)$ is by noting that the second term inside the norm is the HF approximation of $\nabla^2 f_i(\mathbf{w}_{\nu-1})\nabla f_i(\mathbf{w}_{\nu})$ and then using the fact that the error of the HF approximation of the product of $\nabla^2 g (\mathbf{w}) $ by any vector $\mathbf{v}$ is at most $\rho \delta \|\mathbf{v}\|^2$, where $\rho $ is the parameter for Lipschitz continuity of the Hessian of $g$.
	Now, assuming $\left \| \mathbf{d}_{\mathbf{w},i}^{(l-1)} - \mathbf{s}_{\mathbf{w},i}^{(l-1)}\right \|\leq q_{l-1}$, we have
	\begin{align}\label{q_l}
		\left \| \mathbf{d}_{\mathbf{w},i}^{(l)} - \mathbf{s}_{\mathbf{w},i}^{(l)}\right \| =& \Big \| \left(\mathbf{d}_{\mathbf{w},i}^{(l-1)} - \alpha \frac{\nabla f_i(\mathbf{w}_{\nu-l}+\delta \mathbf{d}_{\mathbf{w},i}^{(l-1)})-\nabla f_i(\mathbf{w}_{\nu-l}-\delta \mathbf{d}_{\mathbf{w},i}^{(l-1)})}{2\delta}\right) \nonumber\\
		& \hspace{0.3cm}- \left(\mathbf{s}_{\mathbf{w},i}^{(l-1)}  - \alpha \nabla^2 f_i(\mathbf{w}_{\nu-l})\mathbf{s}_{\mathbf{w},i}^{(l-1)} \right) + \alpha \nabla^2 f_i(\mathbf{w}_{\nu-l})\mathbf{d}_{\mathbf{w},i}^{(l-1)} - \alpha \nabla^2f_i(\mathbf{w}_{\nu-l})\mathbf{d}_{\mathbf{w},i}^{(l-1)} \Big \|\nonumber\\
		{\leq} &\left \| \mathbf{d}_{\mathbf{w},i}^{(l-1)} - \mathbf{s}_{\mathbf{w},i}^{(l-1)} \right \| +\alpha\left \| \nabla^2 f_i(\mathbf{w}_{\nu-l}) \right \| \left \| \mathbf{d}_{\mathbf{w},i}^{(l-1)} - \mathbf{s}_{\mathbf{w},i}^{(l-1)} \right \|
		\nonumber\\
		& +\alpha\left \| \nabla^2 f_i(\mathbf{w}_{\nu-l})\mathbf{d}_{\mathbf{w},i}^{(l-1)} -  \frac{\nabla f_i(\mathbf{w}_{\nu-l}+\delta \mathbf{d}_{\mathbf{w},i}^{(l-1)})-\nabla f_i(\mathbf{w}_{\nu-l}-\delta \mathbf{d}_{\mathbf{w},i}^{(l-1)})}{2\delta} \right \|\nonumber\\
		{\leq} & q_l:=(1+\alpha L)q_{\l-1} + \alpha \rho \delta \|\mathbf{d}_{\mathbf{w},i}^{(l-1)}\|^2,
	\end{align}
	where the last inequality follows the induction hypothesis and  the fact that the last norm is the squared error of the HF approximation of $\nabla^2 f_i(\mathbf{w}_{\nu-l})\mathbf{d}_{\mathbf{w},i}^{(l-1)}$. To fully quantify $q_l$ in \eqref{q_l}, it only remains to bound $\|\mathbf{d}_{\mathbf{w},i}^{(l)}\|^2$; we apply induction to bound $\|\mathbf{d}_{\mathbf{w},i}^{(l)}\|$ by some quantity $a_l$. For the base case, $\|\mathbf{d}_{\mathbf{w},i}^{(0)}\|=\|\nabla f_i(\mathbf{w}_{\nu})\|\leq B:=a_0$. Additionally, by assuming $\|\mathbf{d}_{\mathbf{w},i}^{(l-1)}\|\leq a_{l-1}$, we have
	\begin{align}\label{a_l}
		\left\|\mathbf{d}_{\mathbf{w},i}^{(l)}\right\|&\leq
		\left\|\mathbf{d}_{\mathbf{w},i}^{(l-1)}\right\| + \frac{\alpha}{2\delta}\left\|\nabla f_i(\mathbf{w}_{\nu-l}+\delta \mathbf{d}_{\mathbf{w},i}^{(l-1)})-\nabla f_i(\mathbf{w}_{\nu-l}-\delta \mathbf{d}_{\mathbf{w},i}^{(l-1)})\right\|
		\nonumber\\
		&\leq
		a_{l-1} + \frac{\alpha L}{2\delta}\left\|2\delta \mathbf{d}_{\mathbf{w},i}^{(l-1)}\right\|
		\nonumber\\
		&\leq
		(1+\alpha L)a_{l-1}:=a_l.
	\end{align}
	By solving the recursive relation $a_l=(1+\alpha L)a_{l-1}$ with the initial value of $a_0=B$, we find $a_l=B(1+\alpha L)^l$. 
	Replacing this in \eqref{q_l} results in
	\begin{align}\label{q_l_final}
		q_l=(1+\alpha L)q_{\l-1} + \alpha \rho \delta B^2  (1+\alpha L)^{2l-2},
	\end{align}
	which can be further characterized in the closed form as
	\begin{align}\label{q_l_closed}
		q_l=\alpha \rho \delta B^2\sum_{i=0}^{l-1}(1+\alpha L)^{l+i-1}.
	\end{align}
	This completes the proof of the second result in the lemma by noting that the right hand side of \eqref{Error_HF_result2_term2} is upper bounded by $q^2_{\nu}$.

	To derive the first result, note that
	\begin{align}\label{mu_HF_1}
		\left \| \E \left [ \tnabla F_i^{\rm HF}(\mathbf{w}) - \nabla F_i(\mathbf{w})\right ]\right \| \leq \left\|\E \left [ \tnabla F_i^{\rm HF}(\mathbf{w}) - \mathbf{g}^{\rm HF}_i(\mathbf{w})\right ]\right \| + \left\|\E \left [\mathbf{g}^{\rm HF}_i(\mathbf{w}) - \nabla F_i(\mathbf{w}) \right ]\right \|.
	\end{align}
	To bound the first term in \eqref{mu_HF_1},	
	using \eqref{grad_tilde_Fi_HF} and \eqref{d_HF}, we have $\|\E [ \tnabla F_i^{\rm HF}(\mathbf{w}) - \mathbf{g}^{\rm HF}_i(\mathbf{w})] \| = \|\E [\tilde{\mathbf{d}}_{\mathbf{w},i}^{(\nu)} - \mathbf{d}_{\mathbf{w},i}^{(\nu)}]\|$. We apply induction to upper bound $\|\E [\tilde{\mathbf{d}}_{\mathbf{w},i}^{(l)} - \mathbf{d}_{\mathbf{w},i}^{(l)}]\|$, for $l\in\{0,1,\cdots,\nu\}$, by some positive quantity $p_l$. For the base case of $l=0$ we have
	\begin{align}\label{p_bound_base}
		\left\| \E\! \left [ \tilde{\mathbf{d}}_{\mathbf{w},i}^{(0)} - \mathbf{d}_{\mathbf{w},i}^{(0)} \right ] \right\|& = \left\| \E \!\left [ \tnabla f_i(\tilde{\mathbf{w}}_{\nu-1} - \alpha \tnabla f_i(\tilde{\mathbf{w}}_{\nu-1},\D_{i,\nu-1}),\D_{i,\nu}) - \nabla f_i(\mathbf{w}_{\nu-1} - \alpha \nabla f_i(\mathbf{w}_{\nu-1})) \right ] \right\|\nonumber\\
		& \leq \frac{\sigma_G}{\sqrt{D_{i,\nu}}} +
		L h_{\nu} := p_0,
	\end{align}
	where the last inequality follows \eqref{moment1_e2} with $h_{\nu}$  defined in Lemma \eqref{lemma_w_tilde_difference}.
	Now, assuming $\|\E [\tilde{\mathbf{d}}_{\mathbf{w},i}^{(l-1)} - \mathbf{d}_{\mathbf{w},i}^{(l-1)}]\|\leq p_{l-1}$ as the induction hypothesis, we have for $l\geq 1$
	\begin{align}
		& \hspace{-1cm}	\left\|\E\left[\tilde{\mathbf{d}}_{\mathbf{w},i}^{(l)}-\mathbf{d}_{\mathbf{w},i}^{(l)}\right]\right\| \nonumber\\
		& = \big\|\E\big[\left(\tilde{\mathbf{d}}_{\mathbf{w},i}^{(l-1)}-\mathbf{d}_{\mathbf{w},i}^{(l-1)}\right)  - \frac{\alpha}{2\delta} \left(\tnabla f_i(\tilde{\mathbf{w}}_{\nu-l}+\delta \tilde{\mathbf{d}}_{\mathbf{w},i}^{(l-1)},\D'_{i,\nu-l}) - \nabla f_i(\mathbf{w}_{\nu-l}+\delta \mathbf{d}_{\mathbf{w},i}^{(l-1)})\right) \nonumber\\
		& \hspace{1.4cm} + \frac{\alpha}{2\delta} \left(\tnabla f_i(\tilde{\mathbf{w}}_{\nu-l}-\delta \tilde{\mathbf{d}}_{\mathbf{w},i}^{(l-1)},\D'_{i,\nu-l}) - \nabla f_i(\mathbf{w}_{\nu-l}-\delta \mathbf{d}_{\mathbf{w},i}^{(l-1)})\right)\big]
		\big\|\nonumber\\
		& \stackrel{(a)}{\leq} p_{l-1} +  \frac{\alpha}{2\delta}\big\| \E\big[\tnabla f_i(\tilde{\mathbf{w}}_{\nu-l}+\delta \tilde{\mathbf{d}}_{\mathbf{w},i}^{(l-1)},\D'_{i,\nu-l}) - \nabla f_i(\mathbf{w}_{\nu-l}+\delta \mathbf{d}_{\mathbf{w},i}^{(l-1)})\big]\big\| \nonumber\\
		& \hspace{0.6cm} +  \frac{\alpha}{2\delta}\big\|\E\big[ \tnabla f_i(\tilde{\mathbf{w}}_{\nu-l}-\delta \tilde{\mathbf{d}}_{\mathbf{w},i}^{(l-1)},\D'_{i,\nu-l}) - \nabla f_i(\mathbf{w}_{\nu-l}-\delta \mathbf{d}_{\mathbf{w},i}^{(l-1)})\big]\big\|
		\nonumber\\
		& \stackrel{(b)}{\leq} p_l:=\left(1+ \alpha L\right) p_{l-1}+  \frac{\alpha}{\delta}\left(\frac{\sigma_G}{\sqrt{\D'_{i,\nu-l}}} +
		L h_{\nu-l} \right),\label{p_rec}
	\end{align}
	where step $(a)$ is by the induction hypothesis, and step $(b)$ is by the following result.
	\begin{align}
		&\hspace{-2cm}\big\| \E\big[\tnabla f_i(\tilde{\mathbf{w}}_{\nu-l}\pm\delta \tilde{\mathbf{d}}_{\mathbf{w},i}^{(l-1)},\D'_{i,\nu-l}) - \nabla f_i(\mathbf{w}_{\nu-l}\pm\delta \mathbf{d}_{\mathbf{w},i}^{(l-1)})\big]\big\|\nonumber\\
		& \stackrel{(a)}{\leq} \frac{\sigma_G}{\sqrt{\D'_{i,\nu-l}}} +
		L \E \left [\left \| \left(\tilde{\mathbf{w}}_{\nu-l}-\mathbf{w}_{\nu-l}\right)\pm\delta\left( \tilde{\mathbf{d}}_{\mathbf{w},i}^{(l-1)}- \mathbf{d}_{\mathbf{w},i}^{(l-1)}\right) \right \|  \right ]
		\nonumber\\
		& \leq \frac{\sigma_G}{\sqrt{\D'_{i,\nu-l}}} +
		L \E \left [\left \| \tilde{\mathbf{w}}_{\nu-l}-\mathbf{w}_{\nu-l} \right \|  \right ] + L\delta \E \left [\left \| \tilde{\mathbf{d}}_{\mathbf{w},i}^{(l-1)}- \mathbf{d}_{\mathbf{w},i}^{(l-1)} \right \|  \right ]
		\nonumber\\
		& \stackrel{(b)}{\leq} \frac{\sigma_G}{\sqrt{\D'_{i,\nu-l}}} +
		L h_{\nu-l} + L\delta p_{l-1},
	\end{align}
	where step $(a)$ follows the same procedure as \eqref{h_l_rec_stepB},
	and step $(b)$ follows Lemma \ref{lemma_w_tilde_difference} and the induction hypothesis. 
	The proof is then complete by noting that the right hand side of \eqref{mu_HF_1} is upper bounded by $p_{\nu}+q_{\nu}$.
	\endproof
	
\section{Additional Experiments}\label{Appendix_experiments}	

	\begin{figure}[t]
	\centering
	\includegraphics[trim=0.3cm 1.2cm 0 0,width=6in]{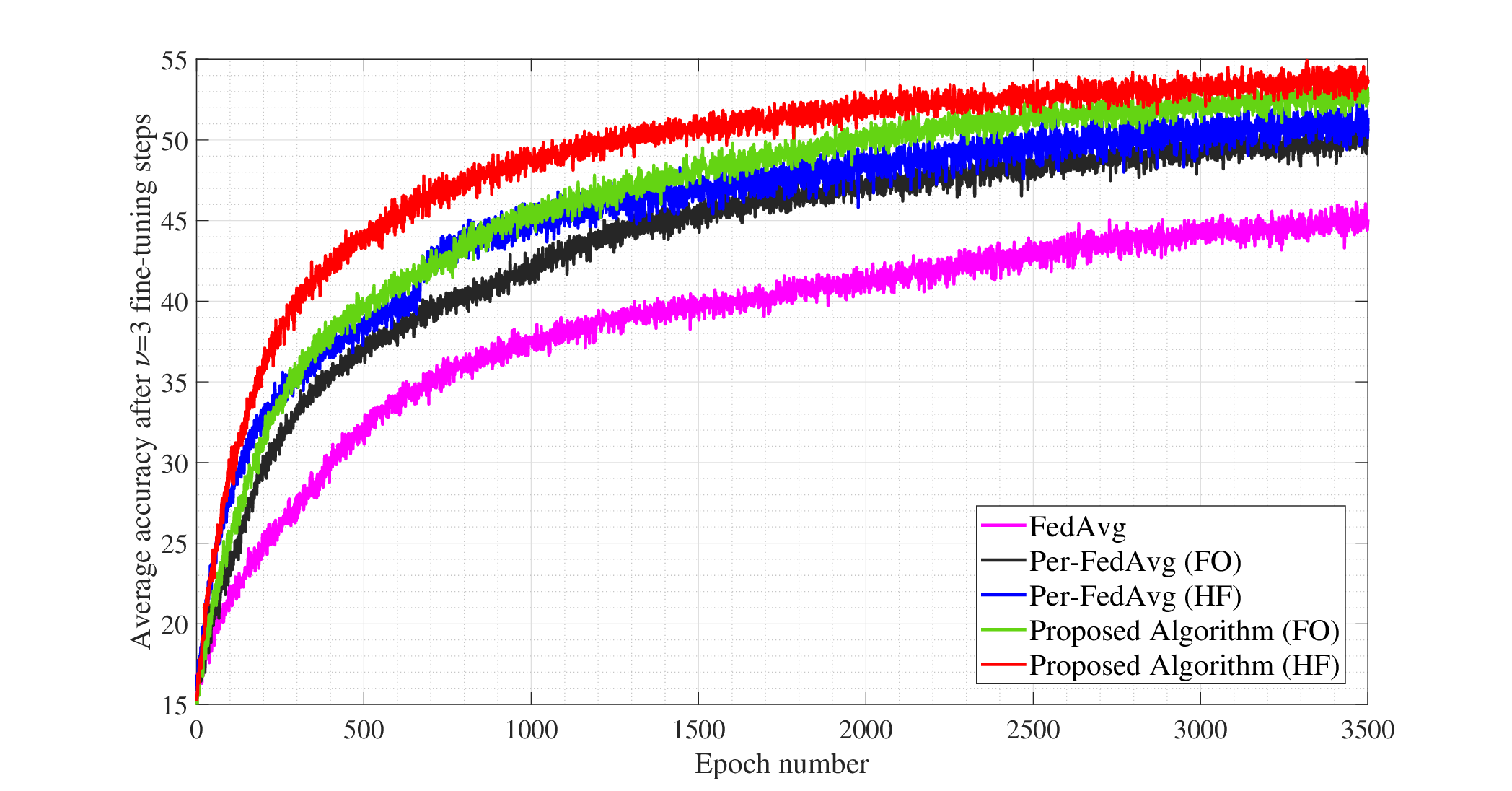}
	\caption{Accuracy of various schemes over the CIFAR-100 dataset with $\alpha_d=0.01$.}
	\label{fig2}
\end{figure}

Figure \ref{fig2} shows the average accuracy percentage of various algorithms over the CIFAR-100 dataset. Similar to Figure \ref{fig1}, our scheme achieves a better accuracy compared to the conventional meta FL and FedAvg algorithms. We note that the accuracies (for all schemes) are degraded compared to Figure \ref{fig1}, implying a more challenging classification task over  CIFAR-100  compared to the CIFAR-10 dataset.

\begin{figure}[t]
	\centering
	\includegraphics[trim=0.3cm 1.2cm 0 0,width=6in]{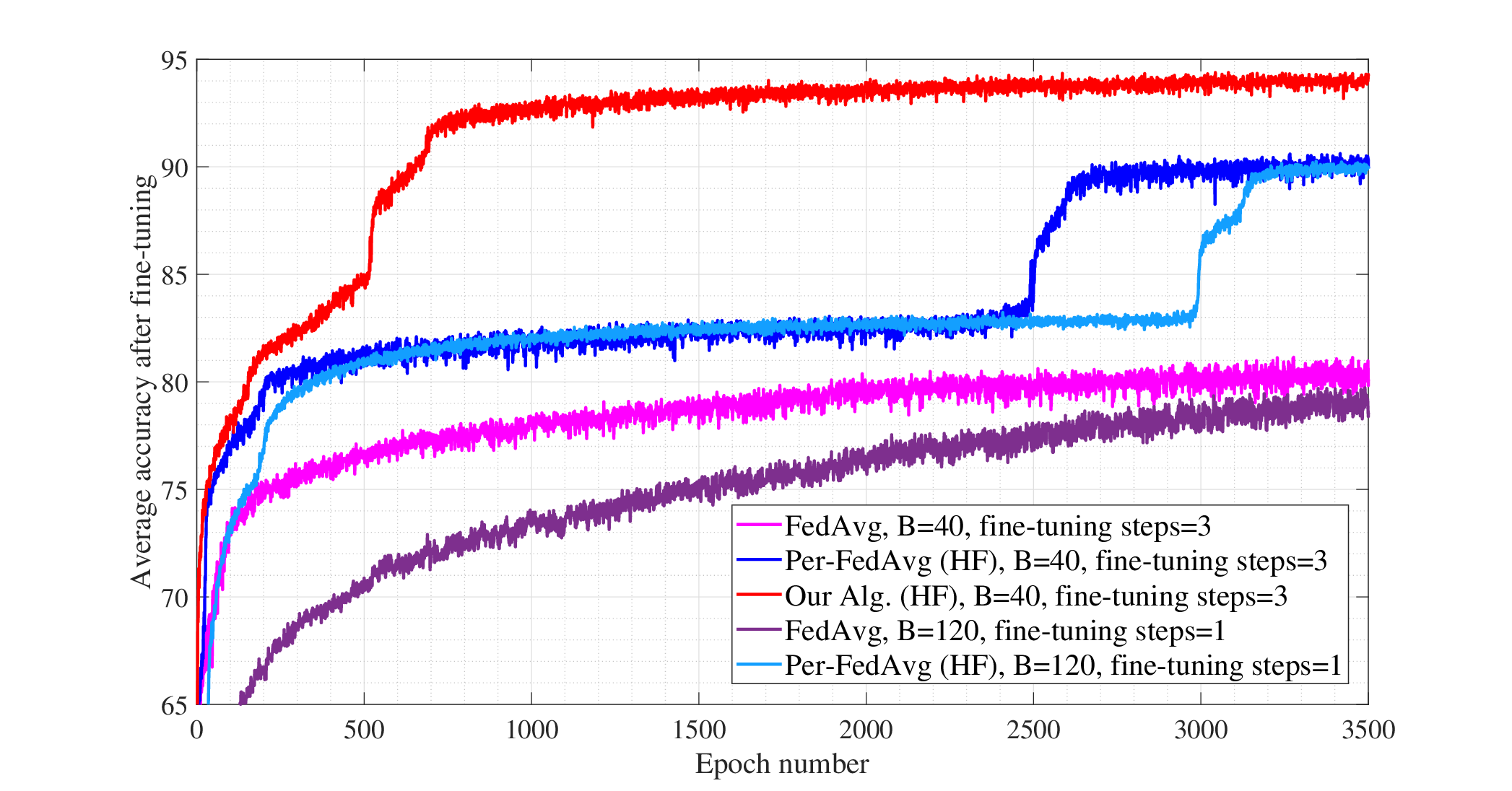}
	\caption{Accuracy of various schemes over the CIFAR-10 dataset with $\alpha_d=0.01$.}
	\label{fig3}
\end{figure}

Given that our algorithms attempt to optimize the performance of the model obtained after $\nu$ fine-tuning steps, one may wonder whether a similar performance could be achieved by using a $\nu$-times larger batch size. In the next experiments, we explore the effect of larger training and fine-tuning batch sizes for various schemes.  Specifically, in Figure \ref{fig3}, we explore the effect of $\nu=3$ times larger batch sizes for the conventional FedAvg and Per-FedAVg algorithms.\footnote{We are considering the larger batch sizes for both training and fine-tuning. To be more precise, we could  consider the larger batch size only for the fine-tuning as we are interested in the performance of such conventional algorithms (trained under the same batch size of $B=40$) after fine-tuning with a larger batch size.} As seen, the FedAvg performance is degraded, meaning that (under the setting of this paper) the FedAVg algorithm favors a larger number of fine-tuning steps compared to a larger batch size. This provides further motivation for the setting of this paper, i.e., optimizing the performance after multiple fine-tuning steps, as simply using a larger batch size may not be able to deliver the same performance as using a larger number of fine-tuning steps. The performance of the Per-FedAvg algorithm, however, is not much affected by using a larger batch size (and a single fine-tuning step), as the Per-FedAvg algorithm is already designed to optimize the performance after a single fine-tuning step. 

\begin{figure}[t]
	\centering
	\includegraphics[trim=0.3cm 1.2cm 0 0,width=6in]{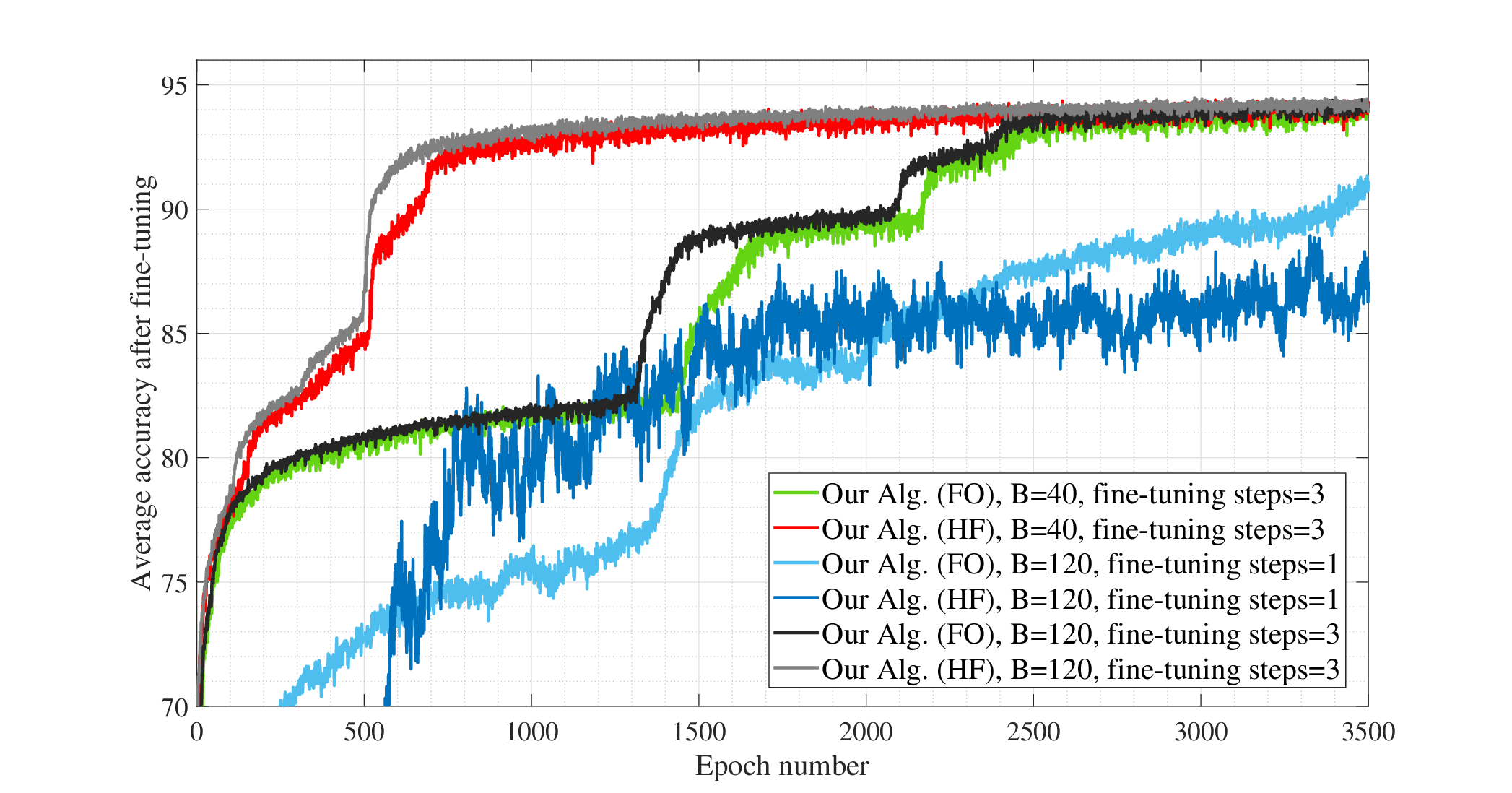}
	\caption{Accuracy of our algorithms over the CIFAR-10 dataset with $\alpha_d=0.01$. Various batch sizes and numbers of fine-tuning steps are considered.}
	\label{fig4}
\end{figure}

In Figure \ref{fig4}, we explore the effect of larger batch size on the performance of our algorithms (under both the FO and HF approximations). We consider  a $3$-times larger batch size of $B=120$ (for both training and fine-tuning -- similar to the setting of the conventional methods in Figure \ref{fig3}). We then evaluate the accuracy for the models obtained after $1$ and $3$ fine-tuning steps. When a single fine-tuning step is used, as expected, the performance is significantly degraded as our algorithms attempt to optimize the performance after $\nu=3$ fine-tuning steps. On the other hand, when $3$ fine-tuning steps are used, increasing the batch size to $120$ does not  improve the performance much compared to the base case of $B=40$. This implies that what matters most for our algorithms is the number of fine-tuning steps and not the batch size.   

\end{document}